\documentclass[12pt]{article}
\usepackage{multibib}
\newcites{Supp}{REFERENCES}
\usepackage{multirow}
\usepackage{amsmath}
\usepackage{amssymb}
\usepackage{amsthm}
\usepackage{mathtools}
\usepackage{graphicx}
\usepackage{xcolor}
\usepackage{natbib}
\usepackage{url}
\usepackage{algorithm}
\usepackage{subcaption}
\usepackage{booktabs}
\usepackage[noend]{algpseudocode}

\usepackage[font=small,labelfont=bf]{caption}
\usepackage[toc,page]{appendix}
\usepackage{hyperref}
\usepackage{multirow}
\usepackage[normalem]{ulem}
\usepackage{titlesec}

\titlespacing*{\section}{0pt}{*2}{*1} % Adjusts spacing
\titleformat{\section}[block]{\large\bfseries}{\thesection}{1em}{}

\newtheorem{proposition}{Proposition}

\newtheorem{theorem}{Theorem}
\newtheorem{lemma}{Lemma}
\newtheorem{remark}{Remark}
\newtheorem{corollary}{Corollary}

\newcommand{\given}{\middle\vert }
\newcommand{\bs}[1]{\boldsymbol{#1}}

\newcommand{\R}{{\mathbb R}}

\renewcommand{\bf}[1]{\mathbf{#1}}

\DeclareMathOperator*{\argmax}{arg\,max}
\DeclareMathOperator*{\argmin}{arg\,min}

\newcommand{\appropto}{\mathrel{\vcenter{
  \offinterlineskip\halign{\hfil$##$\cr
    \propto\cr\noalign{\kern2pt}\sim\cr\noalign{\kern-2pt}}}}}

\newcommand{\showrevisions}{0}
\newif\ifshowrevisions
\showrevisionstrue 
\newcommand{\revised}[1]{%
    \if1\showrevisions%
        \textcolor{blue}{#1}%
    \else%
        #1%
    \fi%
}

\begin{document}

\def\spacingset#1{\renewcommand{\baselinestretch}%
{#1}\small\normalsize} \spacingset{1}

\title{
NeuroPMD: Neural Fields for Density Estimation on Product Manifolds}
\author{
\parbox{\textwidth}{
    \centering
    William Consagra\\
  Department of Statistics, University of South Carolina \\
  Zhiling Gu \\
  Department of Biostatistics, Yale University \\
  Zhengwu Zhang \\
  Department of Statistics and Operations Research, University of North Carolina at Chapel Hill
}
}
\maketitle
\date{}

\bigskip
\begin{abstract}
We propose a novel deep neural network methodology for density estimation on product Riemannian manifold domains. In our approach, the network directly parameterizes the unknown density function and is trained using a penalized maximum likelihood framework, with a penalty term formed using manifold differential operators. The network architecture and estimation algorithm are carefully designed to handle the challenges of high-dimensional product manifold domains, effectively mitigating the \revised{computational} curse of dimensionality that limits traditional kernel and basis expansion estimators, as well as overcoming the convergence issues encountered by non-specialized neural network methods. Extensive simulations and a real-world application to brain structural connectivity data highlight the clear advantages of our method over the competing alternatives. 
\end{abstract}

\noindent%
\vfill
\thispagestyle{empty}

\newpage
\spacingset{1.5} % DON'T change the spacing!
\pagenumbering{arabic} 

\section{Introduction}\label{sec:intro}
This work considers the problem of density (and intensity) estimation on the $D$-product manifold, denoted $\Omega := \bigtimes_{d=1}^D\mathcal{M}_{d}$, where each marginal domain $\mathcal{M}_{d}$ is a closed Riemannian manifold of dimension $p_{d}\ge 1$. 
\revised{Our particular focus is on the big data (large sample size
$n$) and high-dimensional product domain $(D\ge 2$ with $\sum_{d=1}^Dp_{d}\ge 4$) setting.}
\par 
Product manifold-based point set data arise in a wide range of fields, including neuroscience \citep{moyer2017}, genomics/proteomics \citep{zoubouloglou2023}, and climate science \citep{Begu2024}. Our work is specifically motivated by an application in the emerging field of \textit{structural connectomics}, a subfield of neuroscience that studies the physical pattern of neural connections formed by white matter fibers across the cerebral cortex \citep{chung2021}.
These connections can be inferred using modern high-resolution neuroimaging techniques, which map the endpoints of the fibers directly onto the cortical surface \citep{STONGE2018524} (see Figure \ref{fig:structural_connectivity_data} panel A). 
The resulting endpoint connectivity data can be considered a point set on $\Omega = \bigtimes_{d=1}^2\mathcal{M}_{d}$, where $\mathcal{M}_{d}\subset\mathbb{R}^3$ denotes the cortical surfaces. These connections are visualized as paired red and blue points in Figure \ref{fig:structural_connectivity_data} panel B. The goal is to estimate the density function of the connectivity on $\Omega$, which is commonly referred to as the continuous structural connectivity in the literature \citep{moyer2017,consagra2023continuous}. However, this high-dimensional, big data setting, with $D = 2,p_{1}=p_{2}=2$, and $n > 10^6$, poses significant computational and statistical challenges for classical density estimators. Consequently, many 
existing connectivity analysis methods coarsen the connectivity data using a low-resolution brain atlas, masking high-resolution connectivity information that could offer critical insights into key neuroscience questions. 
\begin{figure}[!ht]
    \centering
    \includegraphics[width=0.9\textwidth]{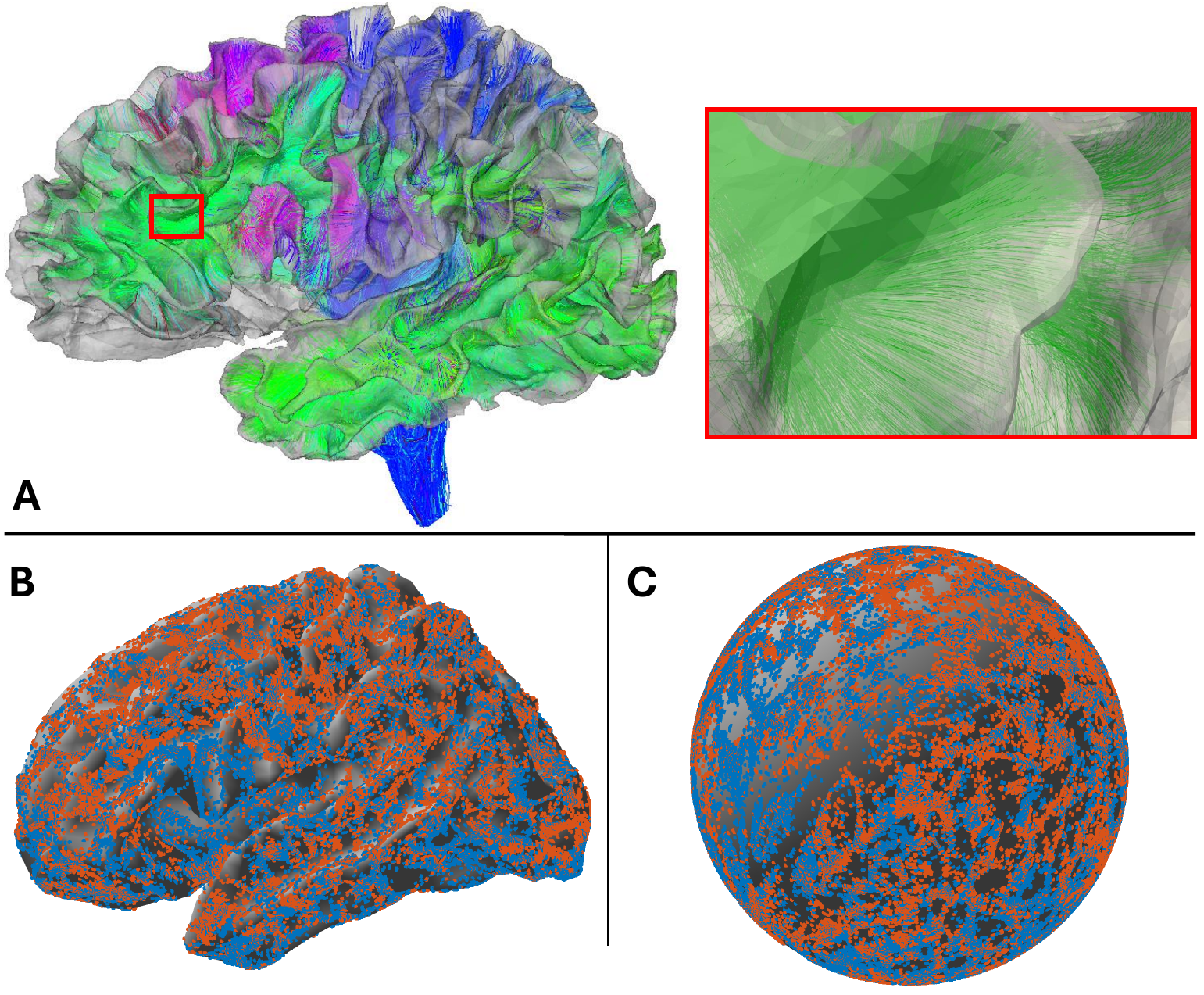}
    \caption{Product manifold point process data from a neuroscience connectomics application. A) Cortical surface mesh of the left hemisphere with representative white matter fiber connections. The red zoomed-in region highlights fibers terminating on the surface. B) Observed endpoints on the surface. Each red point corresponds with a single blue point, which together form the surface coordinates of a connection. C) Observed endpoints represented under spherical parameterization.}
    \label{fig:structural_connectivity_data}
\end{figure}
\subsection{Related Work}\label{ssec:background}
A substantial body of work exists for nonparametric kernel density estimation (KDE) on Riemannian manifolds \citep{pelletier2005,kim2013,bates2014,berry2017,cleanthous2020,ward2023}. These methods can naturally extend to the product space $\Omega$ by forming a $D$-product of the marginal kernels. However, KDE becomes computational\revised{ly} problematic for large $n$ and \revised{$\sum_{d=1}^D p_d$}, as naive point-wise estimation requires at least $O(n\sum_{d=1}^D p_d)$ operations. Computation is further complicated on the manifold domain, as many methods require repeated calculation of geodesic distances or exponential/log map, which can incur additional computational overhead. While approximation techniques for accelerating computation exist for Euclidean spaces \citep{Bentley1975,Karppa2022}, adapting these methods to high-dimensional manifold domains is complicated and cumbersome.  Additionally, bandwidth selection for KDE-based methods is notoriously challenging for high-dimensional domains, often resulting in locally excessive and/or insufficient smoothing due to the varying smoothness of the underlying function across the domain \citep{Wang2019}.
\par
An alternative line of research models the density (or intensity) function using parametric approaches based on linear basis expansion \revised{\citep{goodd1971nonparametric,ferraccioli2021,arnone2022,Begu2024}}. These methods alleviate the computational burden of the non-parametric approaches in large $n$ settings, as computational costs generally scale with the number of basis functions rather than the number of observations. \revised{Regularization is typically imposed via differential operators, with common choices including Laplacian-based penalties in Euclidean domains \citep{Simpson2016,ferraccioli2021} or, on manifolds, its extension, the Laplace-Beltrami operator \citep{arnone2022}.} Additionally, many basis systems serve as universal approximators, capable of representing most functions of interest by selecting a sufficiently large rank. However, these methods are typically limited to relatively low-dimensional domains due to the computational curse of dimensionality, as the \revised{number of parameters (basis coefficients)} often grows exponentially with $D$ to maintain approximation accuracy.
%\citep{HACKBUSCH2007697,Oseledets2009}. 
Data-driven approaches have been proposed to construct more parameter-efficient functional basis for representation %\citep{Chen2017,wang2020,wang2022fda,consagra2023fda,shi2024}. 
\citep{Chen2017,wang2020,wang2022fda,consagra2023fda}.
However, these methods often rely on the availability of replicated functions and assume a direct observation model (signal plus noise framework), neither of which hold in our setting. 
\par 

Recent developments in the machine learning community have introduced flexible density/intensity estimators based on neural networks \citep{NIPS2017_f45a1078,NIPS2017_6463c884,NEURIPS2023_9d30c2de}, though these approaches primarily focus on the Euclidean setting (temporal and spatiotemporal domains). A notable exception is the work by \cite{tsuchida2024exact}, who extended the squared neural density estimation framework of \citep{tsuchida2023} to some simple product manifolds. However, their method relies on custom-designed activation functions that facilitate a closed-form expression of the normalization integral under the Poisson process likelihood, which limits flexibility. Moreover, their network estimator is shallow, consisting of only two layers, and deep extensions are non-trivial due to the aforementioned closed-form integrability requirement. Although deep networks 
%(with depth $ > 2$) 
are not strictly necessary for universal approximation in the infinite-width limit \citep{hornik1989}, empirical %\citep{krizhevsky2012,He_2016_CVPR} 
\citep{krizhevsky2012} 
and theoretical 
%\citep{telgarsky2015representation,eldan2016,Shijun_Zhabg_2022} 
\citep{telgarsky2015representation,eldan2016} 
findings suggest that deep architectures can be significantly more parameter-efficient than wider ones. Thus, the restriction to shallow architectures is a major limitation, as deep networks typically require far fewer parameters to approximate functions of equivalent complexity.

\subsection{Our Contribution}
In this work, we develop a general deep neural network-based density estimator for product Riemannian manifolds. By leveraging the parameter efficiency of deep, fully connected neural networks, our approach avoids the computational curse of dimensionality that plagues classical basis function estimators. Once trained, point-wise inference requires only a forward pass through the network, which can be efficiently executed on a GPU, offering significant improvements in computational scalability compared to KDE.
\par 
While the parameter efficiency and inference speed of deep neural network based function representations are highly desirable in our setting, their estimation is complicated by the so-called \textit{spectral bias}, the tendency of the network to learn overly smooth function approximations and exhibit slow convergence to high frequency components \citep{rahaman2019}. While this phenomenon has been linked to strong generalization properties in certain tasks \citep{Cao2019TowardsUT}, it is typically undesirable when using neural networks to parameterize continuous functions, as it can make learning important high-frequency functional features difficult. To combat this issue, we propose a first-layer basis expansion inspired by Fourier features \citep{Rahimi2007,tancik2020}. Specifically, our approach constructs these basis functions, termed \textit{encoding functions}, by sampling random subsets from the full tensor product set of marginal Laplace-Beltrami eigenfunctions. This design enables the efficient modeling of rich, high-frequency functional features while simultaneously avoiding exponential parameter growth that would result from using a complete tensor product basis, i.e., the computational curse of dimensionality. We demonstrate this property empirically and, for the special case of the $D$-dimensional torus, theoretically, showing that our approach
generates a representation space where elements are linear combinations of a large set of sinusoidal basis functions whose number grows rapidly with network depth, underscoring \revised{how layer composition expands our model's representational capacity.}
%the expressive power provided by layer composition in deep networks.
\par 
Our estimation framework employs a penalized maximum likelihood objective optimized via a custom stochastic gradient ascent based algorithm.  Explicit computation of high-dimensional integrals in the objective function is avoided by forming unbiased estimates of their gradients via Monte-Carlo approximations, thereby avoiding any re-introduction of the curse of dimensionality during network training. To control the smoothness of the estimator, we incorporate a penalty term derived from the Laplace-Beltrami operator, offering computational forms in both intrinsic and extrinsic coordinates for convenience. \revised{We evaluate our method on both challenging simulated and real-world datasets, demonstrating clear improvements over classical approaches, including KDE and linear basis expansion, as well as neural network methods, which we show suffer from spectral bias, an issue our customized architecture and training procedure effectively overcomes.}
\par 
The rest of this paper is organized as follows: Section \ref{sec:models_and_formulation} provides relevant background and formalizes our estimation problem. Section~\ref{ssec:deep_estimator} details the proposed network architecture and provides some theoretical analysis of its \revised{representational capacity}. Section~\ref{sec:estim_algo_comp} proposes an estimation algorithm for learning network parameters along with supporting implementation details including hyperparameter selection. Section \ref{sec:experiments} presents simulation studies and a real data application in brain structural network modeling. Concluding remarks are provided in Section \ref{SEC:conclusion}. Proofs for all propositions and theorems can be found in Supplemental Section S1.

\section{Background and Models}
\label{sec:models_and_formulation}
\subsection{Geometric Preliminaries}
Let $(\mathcal{M}_{d},g_{d})$ be a closed $p_{d}$-dimensional Riemannian manifold with $\mathcal{M}_{d}\subset\mathbb{R}^{m_{d}}$ and induced metric $g_{d}$. We consider the product space $\Omega = \bigtimes_{d=1}^D(\mathcal{M}_d, g_d)$, formed by the Cartesian product of the manifolds $\mathcal{M}_{d}$. This space has a total dimension of $\sum_{d=1}^Dp_{d}$ and inherits a way to measure distance and volumes from the individual metrics $g_d$ of the constituent spaces $\mathcal{M}_d$ via the metric $\sum_{d=1}^Dg_{d}$. Denote $d\omega$ as the product volume form on $\Omega$, analogous to the Lebesgue measure, constructed from the measures on the marginal manifolds $\mathcal{M}_d$ using the metrics $g_d$.
\par 
Analogous to the Laplacian on $\mathbb{R}^{m_{d}}$, a smooth Riemannian manifold $\mathcal{M}_{d}$ is equipped with a linear operator $\Delta_{\mathcal{M}_{d}}:C^{\infty}(\mathcal{M}_{d}) \mapsto C^{\infty}(\mathcal{M}_{d})$, known as the \textit{Laplace-Beltrami operator} (LBO), where $C^{\infty}(\mathcal{M}_{d})$ is the space of smooth functions on $\mathcal{M}_{d}$. To define the LBO, we must consider the differential geometry of $\mathcal{M}_{d}$. For any point $x_{d}\in\mathcal{M}_{d}$, there exists a homeomorphic local parameterization $l_{d}:\mathcal{O}\subset\mathbb{R}^{p_{d}}\mapsto \mathcal{V}\bigcap \mathcal{M}_{d}
\subset\mathbb{R}^{m_{d}}$, where $\mathcal{O}$ is an open set of $\mathbb{R}^{p_{d}}$, and $\mathcal{V}$ is an open set of $\mathbb{R}^{m_{d}}$. 
%Here $p_{d}$ is the dimension of the manifold, and $m_{d}$ is the dimension of the ambient space ($m_{d}\geq p_{d}$).  
We denote the local coordinates corresponding to $x_{d}$ by $\gamma^{(d)} = l_d^{-1}(x_d) \in \mathcal{O}$. Then the tangent space $T_{x_{d}}(\mathcal{M}_{d})$ at $x_{d}$ has a basis given by the vectors  
$\{\frac{\partial l_{d}}{\partial \gamma_{i}^{(d)}}(\gamma^{(d)})\}_{i=1}^{p_{d}}$, which are elements in $\mathbb{R}^{m_{d}}$. The induced metric $g_{d}$ at $x_{d}$ can be represented by a $p_{d}\times p_{d}$ matrix with elements  defined by
$[G^{(d)}(\gamma^{(d)})]_{ij} = \langle \frac{\partial l_{d}}{\partial \gamma_{i}^{(d)}}(\gamma^{(d)}), \frac{\partial l_{d}}{\partial \gamma_{j}^{(d)}}(\gamma^{(d)})\rangle$, such that for any $u_1,u_2\in T_{x_{d}}(\mathcal{M}_{d})$, $g_{d}(u_1,u_2) = u_1^{\intercal}G^{(d)}(\gamma^{(d)})u_2$. The LBO acting on a smooth function  $v_{d}:\mathcal{M}_{d}\mapsto\mathbb{R}$ is then defined as 
\begin{equation}\label{eqn:LBO_standard_def_marg}
    \Delta_{\mathcal{M}_{d}}[v_{d}](x_{d}) = \frac{1}{\sqrt{\det G^{(d)}(\gamma^{(d)})}} \sum_{i,j=1}^{p_{d}} \frac{\partial}{\partial \gamma_i^{(d)}} \left( \sqrt{\det G^{(d)}(\gamma^{(d)})} \left[G^{(d)}(\gamma^{(d)})\right]^{-1}_{ij} \frac{\partial v_{d}}{\partial \gamma_j^{(d)}} (l_d(\gamma^{(d)})) \right).
\end{equation}
Extending this definition to the manifold \(\Omega = \bigtimes_{d=1}^D \mathcal{M}_d\), we utilize the fact that the LBO on a product manifold is the sum of the LBOs of its component manifolds: $\Delta_{\Omega} = \sum_{d=1}^D\Delta_{\mathcal{M}_{d}}$ \citep{Canzani2013}.  
Thus, for a smooth function \(v: \Omega \to \mathbb{R}\), the LBO is given by

\begin{equation}\label{eqn:LBO_standard_def_joint}
\begin{aligned}
    \Delta_{\Omega}[v](x_{1}, \dots, x_{D}) & = \sum_{d=1}^D \frac{1}{\sqrt{\det G^{(d)}(\gamma^{(d)})}} \\ 
    & \times \sum_{i,j=1}^{p_{d}} \frac{\partial}{\partial \gamma_i^{(d)}} \left( \sqrt{\det G^{(d)}(\gamma^{(d)})} \left[G^{(d)}(\gamma^{(d)})\right]^{-1}_{ij}
  \frac{\partial v}{\partial \gamma_j^{(d)}} (x_{1}, \dots, l_d(\gamma^{(d)}), \dots, x_{D}) \right).
\end{aligned}
\end{equation}
%Note that in this expression, $v$ is considered as a function of all variables $(x_1, \dots, x_D)$, and when differentiating with respect to \(\gamma_j^{(d)}\), all other variables \(x_i\) for \(i \neq d\) are treated as constants. 
\subsection{Statistical Model and Problem Formulation}
Let $o = \{{\bs x}_{1}, \revised{{\bs x}_{2}}, \ldots,{\bs x}_{n}\}\subset \Omega$, where ${\bs x}_i := (x_{1i},x_{2i}, \ldots, x_{Di}) \in \Omega$, be an iid set of observations from density function $f\in \mathcal{H}$, with $\mathcal{H}:= \{f:\Omega\mapsto\mathbb{R}^{+}:\int_{\Omega}fd\omega = 1\}$. In this work, we focus on likelihood-based estimation of $f$ from $o$. To handle the positivity constraint on $f$, we adopt a standard log transformation approach and target the log-density, denoted here as $v = \log f$, with the corresponding transformed function space $\tilde{\mathcal{H}} =\{v:\Omega\mapsto\mathbb{R}:\int\exp \left(v\right)d\omega=1\}$. As $v$ is infinite dimensional, its estimation from finite samples $o$ requires regularization to avoid undesirable behavior, e.g., pathological estimates converging to a mixture of Dirac functions centered on each observation. A commonly used regularizer for function approximation is the \textit{roughness penalty},
%\citep{ramsay2005}, 
which controls smoothness by penalizing the $L^2$-norm of the Laplacian operator applied to the candidate function. Since the LBO generalizes the Laplacian to smooth manifold domains, it naturally serves as a basis for defining a roughness penalty in our context. Formally, assuming $v:\Omega\mapsto\mathbb{R}$ is smooth, we consider penalties of the form:
\begin{equation}\label{eqn:manifold_roughess}
    R_{\tau}(v) =  \tau\int_{\Omega}\left[\Delta_{\Omega} v\right]^2d\omega.
\end{equation}
\par 
Putting this all together, we aim to find the function $v$ that maximizes the penalized log-likelihood: 
\begin{equation}\label{eqn:map_pp}
\begin{aligned}
        \hat{v} = 
        & \argmax_{v\in\tilde{\mathcal{H}}} \log p(o|v) - R_{\tau}(v) =  \argmax_{v\in\tilde{\mathcal{H}}} \sum_{i=1}^n v({\bs x}_i) - R_{\tau}(v) \\ 
        = & \argmax_{v\in C^{\infty}(\Omega)} \sum_{i=1}^n v({\bs x}_i) - n\int_{\Omega}\exp(v)d\omega  - R_{\tau}(v), \\ 
\end{aligned}
\end{equation}
where $\tau>0$ is the regularization strength.
The equality in the second line of \eqref{eqn:map_pp} is a consequence of
Theorem 3.1 in \citep{Silverman1982}, which is deployed here in order to remove the integrability constraint in $\tilde{\mathcal{H}}$ from the  optimization problem. Notably, \eqref{eqn:map_pp} can be alternatively motivated from a Bayesian perspective as a MAP estimator under a function space smoothness prior proportional to $\exp\left(-R_{\tau}(v)\right)$ \citep{GASKINS1971}. However, as this work focuses solely on point estimation, the Bayesian framing is incidental.

\begin{remark}
    The optimization problem in \eqref{eqn:map_pp} provides an equivalent formulation for intensity function estimation. Specifically, assuming $o \sim O$ and ${O}$ is an inhomogeneous Poisson process, the penalized maximum likelihood intensity function estimate is simply $n\widehat{f}$. For more details on this relationship, we refer the reader to Appendix B of \cite{Begu2024}.
\end{remark}

\section{Deep Neural Field for Product Manifold Density Modeling} 
\label{ssec:deep_estimator}
In this section, we propose our model for the log-density and study its \revised{representational properties}. To avoid the severe computational issues of nonparametric estimators for functions over multidimensional domains, we restrict our attention to parametric models for $v$. That is, we seek solutions $\widehat{v}_{\bs\theta}$ to \eqref{eqn:map_pp}, parameterized by some finite dimensional vector $\bs\theta$. Traditional parametric approaches often lack flexibility (single index models, additive models) or scale poorly with domain dimensionality (tensor product basis). Alternatively, deep neural network-based parameterizations, commonly referred to as neural fields (NF) or implicit neural representations (INRs), have gained significant recent attention for their ability to provide high-fidelity, parameter-efficient representations of continuous functions \citep{Xie2022}. 
%These properties have lead to their widespread application in computer vision \citep{sitzmann2020,tancik2020,Xie2022}, shape representation %\citep{takikawa2021nglod,yifan2022geometryconsistent}, 
%\citep{takikawa2021nglod}, 
%and physics-based problems %\cite{Karniadakis2021,RAISSI2019686}.  
%\citep{RAISSI2019686}. 
\par 
In this work, we develop a NF-based approach for modeling the log-density $v$. The core architecture is a fully connected multi-layer perceptron (MLP), defined by:
\begin{equation}\label{eqn:nf}
    \begin{aligned}
        &\mathbf h^{(0)} = \bs\eta(\boldsymbol{x}) \\
        &\mathbf h^{(l)} = \alpha^{(l)}(\mathbf W^{(l)}\mathbf h^{(l-1)} + \mathbf b^{(l)}),\quad l=1,\ldots,L-1 \\
        &v_{\bs \theta}(\bs x) := \mathbf W^{(L)}\mathbf h^{(L-1)}, 
    \end{aligned}
\end{equation}
where $\bs x:=(x_1, \ldots, x_D)\in\Omega$, $\bs\theta := (\text{vec}(\mathbf W^{(1)}), \ldots, \text{vec}(\mathbf W^{(L)}),\text{vec}(\mathbf b^{(1)}), \ldots, \text{vec}(\mathbf b^{(L)}))\in \Theta$ consists of the weights $\mathbf W^{(l)}\in\mathbb{R}^{H_{l}\times H_{l-1}}$ and biases $\mathbf b^{(l)}\in\mathbb{R}^{H_{l}}$ of the network, $\alpha^{(l)}$ are the activation functions, and $\bs\eta:\Omega\mapsto\mathbb{R}^{H_0}$ is \revised{a} multivariate function on 
%multidimensional 
domain $\Omega$. 
%An illustration of the proposed MLP is provided in supplemental Figure~\ref{fig:mlp}.
\par 
We complete the definition of our functional model \eqref{eqn:nf}
by formally specifying the encoding function $\bs\eta$ in Section~\ref{sssec:RLBO_encoding} and the activation functions $\alpha^{(l)}$, along with the weight initialization scheme, in  Section~\ref{sssec:activation_funcs}. \revised{In Section~\ref{ssec:approx_theory}, we analyze the effect of these custom design choices by first discussing the approximation properties for the shallow model ($L=1$) and then, for the special case of hypertoroidal domains, characterizing the representation capacity of deep architectures ($L\ge2$).}
%In Section~\ref{ssec:approx_theory}, we study the effect of these custom design choices for the special case of hypertoroidal domains, demonstrating a strong trade-off between expressive power and parameter efficiency of \eqref{eqn:nf}, while avoiding the computational curse of dimensionality.

\subsection{Random Laplace-Beltrami Eigenfunction Encoding}\label{sssec:RLBO_encoding}
We now outline the construction of the encoding function $\bs\eta$. To avoid computational issues associated with direct operations on the full product space, our design uses randomized basis functions formed from products of the eigenfunctions of the LBO on the marginal manifolds. While this provides a simple computationally efficient way to design a basis for a potentially high-dimensional product manifold, the encoding functions resulting from this approach are always separable, which may be undesirable. Therefore, we introduce an optional step that rotates the eigenbasis within each eigenspace, enabling the formation of non-separable eigenfunctions for a more flexible encoding, while maintaining the frequency of the original basis.
\par 
For a smooth Riemannian manifold $\mathcal{M}$, the eigenfunctions $\phi_k:\mathcal{M}\mapsto\mathbb{R}$ of the LBO are defined as the solutions to the operator equation 
\begin{equation}\label{eqn:lbo_eigenfunctions}
    \Delta_{\mathcal{M}} \phi_k  = -\lambda_{k}\phi_k, \quad k=1,2,\ldots.
\end{equation}
These eigenfunctions 
form an orthonormal basis for the space $L^2(\mathcal{M})$, with the eigenvalues $\lambda_k$ forming a non-decreasing sequence of non-negative scalars, analogous to the frequencies in Fourier analysis. Consequently, finite rank-$K$ sets $\{\phi_{k}\}_{k=1}^K$ are often used as basis systems for function representation. For instance, popular basis systems such as the Fourier series and the spherical harmonics are special cases when $\mathcal{M}:=\mathbb{S}^1$ and $\mathcal{M}:=\mathbb{S}^2$, respectively. For general manifolds, these eigenfunctions typically cannot be analytically derived, requiring the use of numerical methods to approximate the solutions to \eqref{eqn:lbo_eigenfunctions} \citep{reuter2009}. 
\par 

For a product manifold $\Omega = \bigtimes_{d=1}^D(\mathcal{M}_d, g_d)$ equipped with the product metric $\sum_{d=1}^Dg_{d}$, the LBO decomposes as $\Delta_{\Omega} = \sum_{d=1}^D\Delta_{\mathcal{M}_{d}}$. Denote the multi-index $\boldsymbol{i}=(i_1,\ldots,i_{D})$. The functions defined by the tensor product of the marginal eigenfunctions, $\psi_{\boldsymbol{i}}:=\prod_{d=1}^D\phi_{i_{d}}$, form a complete basis of eigenfunctions for $L^2(\Omega)$, with corresponding eigenvalues $\sum_{d=1}^D\lambda_{i_{d}}$, where $\lambda_{i_{d}}$ is the marginal eigenvalue for $\phi_{i_d}$ \citep{Canzani2013}. Drawing from classical function approximation theory, a seemingly natural way to define $\boldsymbol{\eta}$ in \eqref{eqn:nf} is to use the basis functions formed by a finite tensor product of marginal eigenfunctions. Specifically, denote $\lambda_{d,\text{max}}\in\mathbb{R}^{+}$ as the truncation of the spectrum along the $d$'th domain and define the multi-index set $\mathcal{I}_{\lambda} = \{\boldsymbol{i}: \sum_{d=1}^D\lambda_{i_{d}}=\lambda, \lambda_{i_{d}}\le\lambda_{d,\text{max}},\forall d \}$. Define the function
$\boldsymbol{\psi}_{\lambda} := \left(\psi_{\boldsymbol{i}}: \psi := \prod_{d=1}^D \phi_{i_{d}}, \boldsymbol{i}\in \mathcal{I}_{\lambda}\right):\Omega\mapsto \mathbb{R}^{s_{\lambda}}$, where $s_{\lambda}$ is the multiplicity of $\mathcal{I}_{\lambda}$. With slight abuse of notation, the collection of tensor product basis functions is given by
$\bigcup_{\lambda \le \lambda_{\text{max}}}\boldsymbol{\psi}_{\lambda}$, where $\lambda_{\text{max}}=\sum_{d=1}^D\lambda_{d,\text{max}}$. 
However, as the number of these basis functions grows exponentially with $D$, 
this choice causes the column space of the first-layer weight matrix $\boldsymbol{W}^{(1)}$ in \eqref{eqn:nf} to grow at the same exponential rate, resulting in significant computational challenges during network training. To avoid this manifestation of the computational curse of dimensionality, a different approach is necessary.

\par 

Inspired by random Fourier feature encodings for functional data in $\mathbb{R}^{D}$, we propose to randomly sample without replacement $K$ LBO eigenfunctions $\psi_{1}, \ldots, \psi_{K}$ from the tensor product set $\bigcup_{\lambda \le \lambda_{\text{max}}}\boldsymbol{\psi}_{\lambda}$ and take $\bs{\eta}(x_1, \ldots, x_{D}) := (\psi_{1}(x_1, \ldots, x_D), \ldots, \psi_{K}(x_1, \ldots, x_{D}))^{\intercal}$. Since $K$ is an architecture hyperparameter and can be chosen independently of $D$, this approach avoids the curse of dimensionality associated with using the full tensor product basis system.
%We establish the strong expressive power of this choice of $\bs \eta$ theoretically in Section~\ref{ssec:approx_theory} and validate it empirically through simulations in Section~\ref{sec:experiments}.
\par

When using the raw tensor products, all basis functions in the encoding $\bs\eta$ are \textit{separable}, that is, they can be written as products over the marginal domains. To gain more flexibility, we introduce an optional basis rotation step. Specifically,  let $\boldsymbol{A}_{\lambda}\in \mathbb{R}^{s_{\lambda}\times s_{\lambda}}$ be an orthogonal matrix, then it is easy to show that $\tilde{\boldsymbol{\psi}}_{\lambda} = \boldsymbol{A}_{\lambda}\boldsymbol{\psi}_{\lambda}$ is a basis  for the same eigenspace: $\text{span}\left(\boldsymbol{\psi}_{\lambda}\right)$. When $\boldsymbol{A}_{\lambda}\neq \boldsymbol{I}$, the basis functions $\tilde{\boldsymbol{\psi}}_{\lambda}$ are \textit{non-separable}, meaning they cannot be written as a product of functions over the marginal domains $\mathcal{M}_{d}$. Although in theory functions $\boldsymbol{\psi}_{\lambda}$ and $\tilde{\boldsymbol{\psi}}_{\lambda}$ span the same space, we have empirically observed that the choice of basis can have a significant impact on the algorithmic convergence (the spectral bias), and hence can be tuned to accelerate the training of \eqref{eqn:nf}. For further discussion and empirical analysis of this effect, see supplemental Section S3.3 and S5.3.1. Putting this together, we define the first layer encoding function as:
\begin{equation}\label{eqn:random_LBO_encoding}
\begin{aligned}
    &\psi_{1}, \ldots, \psi_{K}\overset{u.w.o.r.}{\sim}\bigcup_{\lambda \le \lambda_{\text{max}}}\boldsymbol{A}_{\lambda}\boldsymbol{\psi}_{\lambda} \\
        &\bs{\eta}(x_1, \ldots, x_{D}) := (\psi_{1}(x_1, \ldots, x_D), \ldots, \psi_{K}(x_1, \ldots, x_{D}))^{\intercal},
\end{aligned}
\end{equation}
where \textit{u.w.o.r.} indicates uniform sampling w/o replacement. 
\par 
In summary, the first layer encoding function $\bs\eta$ consists of \revised{$K$} randomly sampled \revised{tensor product} LBO eigenfunctions \revised{(optionally rotated), thereby controlling the input width of \eqref{eqn:nf} with hyperparameter $K$ and avoiding exponential parameter growth in $D$.} 
%The random sampling of LBO eigenfunctions is essential in overcoming the computational curse of dimensionality as $D$ and $|\mathcal M_d|$ increase, while preserving strong expressive power.

\subsection{Hidden Layer Activations and Initialization}\label{sssec:activation_funcs}
To complete the definition of our deep intensity estimator \eqref{eqn:nf}, we define the activation functions $\alpha^{(l)}$ as sinusoidal functions for all hidden layers, i.e., $\alpha^{(l)}(\cdot) := \sin(\cdot)$, $\forall l$. Despite historical reticence in using non-monotonic, periodic activations due to observed pathological effects on the loss surface \citep{parascandolo2017taming} and perceived difficulty in training \citep{osti_5470451}, \cite{sitzmann2020} demonstrate that with a carefully designed weight initialization scheme--where the output distribution at initialization is independent of the network depth and the input to each hidden activation is a standard normal--deep networks with hidden sinusoidal activations can be trained stably. Empirical results indicate that sinusoids with proper initialization can substantially improve convergence speed compared to traditional activations \citep{Faroughi2024}, a finding also supported by our experiments.

\subsection{\revised{Some Properties} of the Neural Field Architecture}\label{ssec:approx_theory}
\revised{The ultra-high-dimensionality of the parameter space $\Theta$ and strong non-linearities make theoretic analysis of deep neural network estimators notoriously challenging. For the proposed architecture \eqref{eqn:nf}, this difficulty is compounded by our use of non-standard, non-monotonic sinusoidal activation functions and the random encoding $\boldsymbol{\eta}$. Indeed, understanding the approximation properties of neural networks with random feature encoding is still in its infancy, with existing results largely restricted to shallow architectures and requiring strong assumptions on the activation functions \citep{hsu2021approximation,li2023powerful}. In this context, and using the fact that a shallow $L=1$ layer NF of the form \eqref{eqn:nf} is equivalent to a linear combination of randomly sampled tensor-product LBOs, we can adapt results on the approximation properties of random Fourier Features \citep{bach2017equivalence} to show that such networks are universal approximators for sufficiently large $K$, and quantify how $K$ must scale with target accuracy. These results are presented in Supplemental Section S3.1.}
\par 
\revised{While approximation properties for $L=1$ serve as a useful baseline guarantee, it does not characterize the effect of depth ($L\ge 2$), which is a key driver of the recent empirical success of neural networks. To provide some insight into the role of depth, we consider the special case of $\mathcal{M}_{d}=\mathbb{S}^1$, i.e. $\Omega$ is a $D$-torus. In this setting, the (separable) tensor product LBO eigenfunctions have the analytic form:
$
\psi_{\boldsymbol{i}_k}(\bs x) = \pi^{-D/2}\prod_{d=1}^D\cos(i_{k,d}x_d -z_{k,d}\pi/2), z_{k,d}\in \{0,1\}. 
$
The following theorem provides
a compact interpretable form for the types of functions that can be represented by the
architecture~\eqref{eqn:nf} for the case of $\mathbb{T}^{D}$.}
\revised{
\begin{theorem}
\label{thm:representation_space_Td}
Assume all weights and biases are bounded, 
$0< |\mathbf b^{(l)}|<C_b<\infty, 0<|\mathbf W^{(l)}| <C_w<\infty$, for all $1\leq l\leq L-1$.
Let $ \Omega = \mathbb{T}^D = \bigtimes_{d=1}^D\mathbb{S}^1$, $v_{\bs{\theta}}: \mathbb{T}^D  \rightarrow \mathbb{R}$ be an NF of the form \eqref{eqn:nf},  with $\boldsymbol{\eta}:=(\psi_{\boldsymbol{i}_{1}},...,\psi_{\boldsymbol{i}_{K}})^{\intercal}:\mathbb{T}^{D}\mapsto\mathbb{R}^{K}$, where elements are samples $\boldsymbol{i}_{k}\overset{iid}{\sim}\text{Unif}(\underset{\lambda \le \lambda_{\text{max}}}{\bigcup}\mathcal{I}_{\lambda})$. Denote the frequency set $\mathcal W_{k}=\{({i_{k,1}}, \pm{i_{k,2}}, \ldots, \pm{i_{k,D}})\}$, for $k=1,\ldots,K$.
Then, 
$$
v_{\bs{\theta}}(\bs x) =\sum_{\bs{w}^{\prime} \in \mathcal{H}^{(L)}} \beta_{\bs{w}^{\prime}} \cos \left(\left\langle\bs{w}^{\prime}, \bs x\right\rangle+{b}_{\bs{w}^{\prime}}\right) + \nu,
\quad 
\mathcal{H}^{(L)}
= \left\{\tilde{\bs w}=\sum_{k =1}^K   c_k \bs w_{k}\given c_k \in \mathcal A^{(L)}, \bs w_{k}\in\mathcal W_{k}\right\},
$$
where $\mathcal A^{(L)} = \{c_k = \Pi_{l=1}^L c_k^{(l)} \mid c_k^{(l)}\in \mathcal A\}$, with $\mathcal A = \{c_k^{(l)}\in \{n_k, n_k-2, \ldots , -(n_k-2)\} \mid \sum_{k =1}^K  n_k \in \{1,\ldots, M_{\nu,L}\}, k=1,\ldots, K\}$, $\nu>0$ is a small user defined tolerance, $M_{\nu,L}\in \mathbb{N}$ that grows with increasing $L$ and decreasing $\nu$, and the $\beta_{\bs{w}^{\prime}}$ are non-linear 
%, potentially non-unique,
functions of $\boldsymbol{\theta}$.
\end{theorem}
It is important to note that Theorem \ref{thm:representation_space_Td} is a representation result that characterizes the functional form of our network, rather than an approximation theorem that bounds the accuracy in a target function class. However, the theorem highlights the role of depth in rapidly expanding the initial encoder $\boldsymbol{\eta}$, as the random frequency sets $\{\mathcal W_{k}\}$ are combined by integer scalings from set $\mathcal{A}^{(L)}$, whose cardinality grows very quickly with respect to the number of hidden layers $L$. 
%(roughly up to the order $\binom{L+M_{\nu,L}-1}{L}$).
Thus, for the special case of $\mathbb{T}^{D}$, the network \eqref{eqn:nf} is approximately equivalent to an expansion over a large set of sinusoidal basis functions, with frequencies in $\mathcal{H}^{(L)}$ that are scaled versions of those used to form $\boldsymbol{\eta}$, and whose expansion coefficients are determined by the network parameters. We note that similar representation results appear for NF architectures on Euclidean domains \citep{fathony2021,yuce:2021}. Supplement Section S1.3 provides further details on $\nu$, $M_{\nu,L}$ and the approximate order of the cardinality of $\mathcal{A}^{(L)}$.}

\section{Estimation and Computational Details}\label{sec:estim_algo_comp}
\subsection{Scalable Stochastic Gradient Algorithm}\label{ssec:algo}
Under the deep network model outlined in Section~\ref{ssec:deep_estimator}, the parametric form of the optimization problem
\eqref{eqn:map_pp} is given by:
%$\mathcal{L}(o,\theta):=\log p(o|v_{\theta}) - \tau R(v_{\theta})$
\begin{equation}\label{eqn:MAP_pp_reparam}
\begin{aligned}
    \widehat{\boldsymbol{\theta}} := \argmax_{\boldsymbol{\theta}\in\Theta}\mathcal{L}(o,\boldsymbol{\theta}) &=  \argmax_{\boldsymbol{\theta}\in\Theta}\Big(\frac{1}{n}\sum_{i=1}^n v_{\boldsymbol{\theta}}({\bs x}_{i}) - \int_{\Omega}\exp\left(v_{\boldsymbol{\theta}}\right)d\omega  - R_{\tau}(v_{\boldsymbol{\theta}})\Big).
    \end{aligned}
\end{equation}
In general, optimization problems for deep network parameters are both high dimensional and highly non-convex and, therefore, are not solved globally. Instead, local solutions are found using stochastic gradient-based procedures \citep{KingBa15}, with gradients computed via backpropagation. 
\par 
A potential complexity arises in our case because the normalization and penalty terms in~\eqref{eqn:MAP_pp_reparam} involve integration over $\Omega$, which in general is analytically intractable. Numerical approximation using quadrature rules, an effective approach in lower-dimensional settings \citep{ferraccioli2021}, must be avoided as it reintroduces the curse of dimensionality, with the total number of quadrature points being exponential in $D$. Instead, recall that stochastic gradient-based optimization procedures require only an unbiased estimate of the objective function's gradient \citep{Bottou2018}. Leveraging this fact, the following proposition establishes that the necessary gradients can be computed while avoiding the computational curse of dimensionality.
\begin{proposition}\label{prop:unbiased_grads}
Let $\text{Vol}(\Omega):=\int_{\Omega}d\omega$ and define
\begin{equation}\label{eqn:unbiased_gradient_estimates}
\begin{aligned}
    &\boldsymbol{a}^{b}(\boldsymbol{\theta}):=\frac{1}{b}\sum_{i=1}^b\frac{\partial}{\partial\boldsymbol{\theta}}v_{\boldsymbol{\theta}}(x_{1i}, \ldots,x_{Di}), \quad \boldsymbol{b}^{q_{1}}(\boldsymbol{\theta}):=\frac{\text{Vol}(\Omega)}{q_{1}}\sum_{j=1}^{q_{1}}\frac{\partial}{\partial\boldsymbol{\theta}}\exp(v_{\boldsymbol{\theta}}(x_{1j}, \ldots, x_{Dj})) \\
    &\boldsymbol{c}_{\tau}^{q_{2}}(\boldsymbol{\theta}):=\tau\frac{\text{Vol}(\Omega)}{q_{2}}\sum_{l=1}^{q_{2}} \frac{\partial}{\partial\boldsymbol{\theta}}\left[\Delta_{\Omega} v_{\boldsymbol{\theta}}(x_{1l},\ldots,x_{Dl})\right]^2,
\end{aligned}
\end{equation}
where $b\le n$ is the size of a uniformly sampled data batch from $o$, and $q_{1}, q_{2}$ are the Monte Carlo sample sizes with points sampled uniformly over $\Omega$. Then
$$
\mathbb{E}[\boldsymbol{a}^{b}(\boldsymbol{\theta}) - \boldsymbol{b}^{q_{1}}(\boldsymbol{\theta}) - \boldsymbol{c}_{\tau}^{q_{2}}(\boldsymbol{\theta})] = \frac{\partial}{\partial\boldsymbol{\theta}}\mathcal{L}(o,\boldsymbol{\theta}).
$$
\end{proposition}
As a result of proposition~\ref{prop:unbiased_grads}, if the terms in Equation~\ref{eqn:unbiased_gradient_estimates} can be quickly computed, the solution to \eqref{eqn:MAP_pp_reparam} can be approximated using a batch stochastic gradient ascent procedure. The gradients required for $\boldsymbol{a}^{b}(\boldsymbol{\theta})$ and $\boldsymbol{b}^{q_{1}}(\boldsymbol{\theta})$ can be formed efficiently using backpropagation. However, forming $\boldsymbol{c}_{\tau}^{q_{2}}(\boldsymbol{\theta})$ is complicated by the presence of the manifold differential operator $\Delta_{\Omega}$. To handle this term, we consider two cases. 
\par 
For certain product manifolds, definition \eqref{eqn:LBO_standard_def_joint} simplifies considerably. For instance, in the case of the \(D\)-dimensional torus \(\mathbb{T}^{D} = \bigtimes_{d=1}^D \mathbb{S}^1\), the metric tensor components and their determinants become constants, and the LBO reduces to a sum of standard Laplacians on each circle component (see Supplemental Section S2.1). In such cases, the network \eqref{eqn:nf} can be defined directly on the intrinsic coordinates $(\gamma^{(1)}, ...,\gamma^{(D)})$, and $\Delta_{\Omega}[v]$ computed using automatic differentiation \citep{Baydin2018}.
\par 
Alternatively, in applications involving more complex manifolds with many local charts, forming \eqref{eqn:LBO_standard_def_joint} may become cumbersome and challenging. 
%\citep{Dziuk_Elliott_2013}. 
In such cases, it may be more convenient to work directly with the extrinsic (Euclidean) coordinates \citep{Harlim2023}. 
%\citep{Fuselier2013}. 
The following proposition provides an alternative computational form for point-wise evaluation of the LBO of a smooth function $v$ that avoids the explicit use of local coordinates.
\begin{proposition}\label{prop:LBO_via_ambient_projection}
    Let $v:\Omega\mapsto\mathbb{R}$ be a smooth function on the product manifold $\Omega$. Let $\{t_1^{(d)},\ldots,t_{p_{d}}^{(d)}\}$ be a set of orthogonal basis vectors in $\mathbb{R}^{m_{d}}$ that span $T_{x_{d}}(\mathcal{M}_d)$ (tangent space at $x_{d}\in\mathcal{M}_d$). Define the matrices $\boldsymbol{P}^{(d)}=\sum_{i=1}^{p_{d}}t^{(d)}_{i}{t^{(d)}_{i}}^{\intercal}\in\mathbb{R}^{m_{d}\times m_{d}}$ and the block diagonal matrix $\boldsymbol{P} = \text{BlockDiag}\left(\boldsymbol{P}^{(1)}, ..., \boldsymbol{P}^{(D)}\right)$.
    %\begin{equation}\label{eqn:manifold_grad_proj_mat}
        %P =
        %\begin{pmatrix}
        %P^{(1)} & 0           & \cdots & 0           \\
        %0           & P^{(2)} & \cdots & 0           \\
        %\vdots      & \vdots       & \ddots & \vdots      \\
        %0           & 0            & \cdots & P^{(D)}
        %\end{pmatrix}.
        %\end{equation}
    Then the Laplace-Beltrami operator of $v$ is given by
    \begin{equation}\label{eqn:projected_diff_operators}
    \begin{aligned}
                \Delta_{\Omega}[v](\bs x) = \sum_{i=1}^{M} \boldsymbol{P}_{i,\cdot}\boldsymbol{H}_{\mathbb{R}^{M}}[v](x_{1}, \ldots, x_{D})\boldsymbol{P}_{i,\cdot}^{\intercal},\\
        \end{aligned}
    \end{equation}
    where $M=\sum_{d=1}^Dm_{d}$, $\boldsymbol{H}_{\mathbb{R}^{M}}$ is the standard Hessian in $\mathbb{R}^{M}$.% and $\boldsymbol{P}_{i,\cdot}$ is the $i$-th row of $P$.
\end{proposition}
Equation \eqref{eqn:projected_diff_operators} implies that if we can efficiently calculate i) the second order derivatives of $v$ in the ambient Euclidean space and ii) the basis vectors to form $\boldsymbol{P}^{(d)}$,  we can construct the differential operators of interest efficiently. Regarding i), the required Euclidean Hessian can be calculated via automatic differentiation. The ease of ii) depends on the manifold in question. For some manifolds of interest, e.g. $\mathbb{S}^{p}$, closed form solutions exist and hence can be calculated rapidly. In the general case, it is common for complicated manifolds to be approximated using triangulations (or their higher dimensional analogs). In such situations, the tangent basis can be approximated directly using the triangulation vertices. 
\par 
Algorithm~\ref{alg:sgd} provides pseudocode for our stochastic gradient ascent based estimation procedure. In general, proving the convergence of the iterates in Algorithm~\ref{alg:sgd} to a stationary point of \eqref{eqn:MAP_pp_reparam} requires some additional assumptions on $\mathcal{L}(o,\boldsymbol{\theta})$ that are difficult to verify, e.g. globally Lipschitz gradients \citep{ghadimi2013}. Empirically, we found convergence to be robust for simple learning rate sequences, provided that $\tau$ was large enough to avoid the case of unbounded or nearly unbounded likelihoods. For additional details on the algorithm and practical implementation guidance, please refer to Supplemental Section S3.2.

\begin{algorithm}[t]
  \caption{Mini-batch Stochastic Gradient Ascent Estimator}
  \label{alg:sgd}
   \small 
  \begin{algorithmic}[1]
    \State \textbf{Input} non-negative sequence of learning rates $\{w_{t}\}$, architecture hyperparameters for \eqref{eqn:nf}, batch sizes $b$, MC sample sizes $q_{1}$, $q_{2}$, regularization strength $\tau$
    \State Set $\bs{\eta} := (\psi_{1}, \ldots, \psi_{K})^\intercal$ by sampling $\{\psi_{\bs i}\}$ via \eqref{eqn:random_LBO_encoding}
    \State Initialize $\boldsymbol{\theta}^{\text{i}}$ using scheme from \cite{sitzmann2020}
    \For{$t = 1,\ldots,T$}
    \For{$l = 1,\cdots \lfloor\frac{n}{b}\rfloor$}
    \State Sample mini-batch $\{{\bs x}_i\}_{i=1}^b$ uniformly from $o$
    \State Sample $\{{\bs x}_{j}\}_{j=1}^{q_{1}}$, $\{{\bs x}_l\}_{l=1}^{q_{2}}$ $\overset{iid}{\sim}\text{Unif}(\Omega)$ \label{alg:manifold_sampling}
    \State Calculate $\frac{\partial}{\partial\boldsymbol{\theta}}\widehat{\mathcal{L}}(o,\boldsymbol{\theta}^{\text{i}})$ using  Proposition \ref{prop:unbiased_grads}. 
    \State Update $\boldsymbol{\theta}^{\text{c}} = \boldsymbol{\theta}^{\text{i}} - w_{t} \frac{\partial}{\partial\boldsymbol{\theta}}\widehat{\mathcal{L}}(o,\boldsymbol{\theta}^{\text{i}})$
    \State Set $\boldsymbol{\theta}^{\text{i}} = \boldsymbol{\theta}^{\text{c}}$
    \EndFor
    \EndFor
    \State \textbf{Return}: $\widehat{\boldsymbol{\theta}} := \boldsymbol{\theta}^{\text{c}}$
  \end{algorithmic}
\end{algorithm}

\subsection{Hyperparameter Selection}\label{ssec:hyper_param_selection}
As is typical for deep neural network based estimators, Algorithm~\ref{alg:sgd} involves many hyperparameters that require tuning. %These include standard parameters such as network depth, hidden layer width, learning rates, data batch size, and number of epochs, as well as those specific to our setting, e.g. MC batch sizes $q_1$ and $q_2$, encoding basis rank $K$, and roughness penalty strength $\tau$, all of which can interact in non-trivial ways to affect the quality of the final estimate. 
%\par 
Evaluating the quality of a given hyperparameterization necessitates a data-driven criterion. Towards this end, we first partition the observed points into training $o_{T}$ and validation $o_{V}$ sets. We estimate $\widehat{\boldsymbol{\theta}}$ using Algorithm~\ref{alg:sgd} with $o_{T}$ for each candidate hyperparameterization, and then calculate criterion
\begin{equation}\label{eqn:l2_selection_criteria}
    C(\widehat{\boldsymbol{\theta}}) = \left\|f_{\widehat{\boldsymbol{\theta}}}\right\|_{L^2(\Omega)}^2 - \frac{2}{|o_{V}|}\sum_{{\bs x}_i \in o_{V}}f_{\widehat{\boldsymbol{\theta}}}({\bs x}_i),     
\end{equation}
%where $f_{\widehat{\boldsymbol{\theta}}} = \exp\left(v_{\widehat{\boldsymbol{\theta}}}\right)\Big/\int_{\Omega}\exp\left(v_{\widehat{\boldsymbol{\theta}}}\right)$ is the normalized density function, 
with the integral estimated using Monte-Carlo integration. The criteria \eqref{eqn:l2_selection_criteria} corresponds to a shifted approximation of the integrated squared error (ISE) of the density estimator \citep{Hall1987}.
\par 
In theory, criterion \eqref{eqn:l2_selection_criteria} can be utilized within various algorithms for hyperparameter selection. The design of such algorithms remains a challenging and active area of research in the machine learning community, due to the high dimensionality of the hyperparameter spaces and high cost of computing estimators \citep{Yu2020HyperParameterOA}. After experimenting with modern approaches to jointly optimize over many hyperparameters \citep{snoek2012}, we found that a simple two stage approach worked best. In this approach, all hyperparameters except for $\tau$ are first fixed to values deemed reasonable based on the literature and exploratory experimentation. The parameter $\tau$ was then selected over a grid $\tau_{1} < \tau_{2} < \cdots <\tau_{S}$ via: 
$
    \tau = \argmin_{\tau_{s}\in\{\tau_1, \ldots, \tau_{S}\}} C(\widehat{\boldsymbol{\theta}}_{\tau_{s}}), 
$
where $\widehat{\boldsymbol{\theta}}_{\tau_{s}}$ is estimated from Algorithm~\ref{alg:sgd} with $o_{T}$ and $\tau_{s}$. This approach proved effective, as Algorithm~\ref{alg:sgd} was found to be particularly sensitive to the choice of $\tau$.
\par 
We conclude this section with a few practical guidelines for setting the remaining hyperparameters. For network size, we found that operating in or near the over-parameterized regime is effective, with the total number of network parameters ranging from at least half to twice $n$. A large data batch size is recommended, typically between $n/2$ and $n/4$, which aligns with empirical findings reported for training Euclidean NFs on Gaussian data \citep{dupont2022}. The Monte Carlo sample size $q_1$ can typically be set relatively high, however, a large $q_2$ may significantly slow training, particularly when automatic differentiation is employed to approximate the roughness penalty (see Supplemental Section S3.2 for alternative derivative calculation strategies). We recommend setting the encoding basis rank $K$ equal to the width of the first layer and increasing the maximum frequency of the marginal basis as the dimensionality increases. \revised{Finally, early stopping can be performed by monitoring convergence using the criterion in \eqref{eqn:l2_selection_criteria} on a small held-out validation set, or by tracking the variance of the gradients. See Supplemental Section S5.5 for a detailed discussion and evaluation of both criteria}.

\section{Empirical Evaluation}\label{sec:experiments}

\subsection{Synthetic Data Analysis}\label{ssec:simulation_studies}
\subsubsection{Simulation Setup}
%\subsubsection{Synthetic Data Generation}
We evaluate the recovery of the target density $f$ on the (hyper)-torus $\mathbb{T}^D$ using simulated data. We focus on the hyper-toroidal case because it allows for easier definition and sampling of non-separable anisotropic density functions, and because there are standard competing alternatives available. We consider $\Omega = \mathbb{T}^2$ (\revised{2-d}) and $\mathbb{T}^4$ (\revised{4-d}) cases to study the performance of the method on both (relatively) low and high-dimensional \revised{product} domains. The true density functions are defined using mixtures of anisotropic wrapped normal distributions \citep{mardia2009directional}. 
For $\mathbb{T}^2$, the density is defined as an equally weighted mixture of three anisotropic components, as shown in the top left panel of Figure~\ref{fig:sim_T2}. For $\mathbb{T}^4$, the density is an equally weighted mixture of 5 anisotropic components. $n=10,000$ and $n=100,000$ observations are simulated for $\mathbb{T}^2$ and $\mathbb{T}^4$, respectively. The MATLAB library \textit{libDirectional} was used for defining and sampling the mixture model \citep{libdirectional}. 20 replications were formed for each dataset. For more details on the synthetic data generation, see supplemental Section S5.
%\subsubsection{Implementation Details}
\par 
The eigenfunctions of the LBO on $\mathbb{S}^{1}$ using intrinsic (polar) coordinates are the Fourier functions with discrete frequencies. For both $\mathbb{T}^{2}$ and $\mathbb{T}^{4}$, we set the marginal maximum frequencies $\lambda_{d,max}=10$, $d=1,\ldots,D$ and take $K=128$. To form the first layer encoding $\boldsymbol{\eta}$, for $\mathbb{T}^{2}$, we sample the separable tensor product eigenfunctions. For $\mathbb{T}^{4}$, we employ the augmented approach outlined in Section~\ref{sssec:RLBO_encoding} and use non-separable eigenfunctions formed via predefined rotation of the separable eigenfunctions (see supplemental Section S5.3.1 for more details). The width of all hidden layers is set to $128$. For $\mathbb{T}^{2}$, we use a depth $L=3$, whereas for $\mathbb{T}^{4}$, we increase the depth to $L=4$ to reflect the increased dimensionality of the problem. To calculate the gradients in \eqref{eqn:unbiased_gradient_estimates}, we let batch size $b=n/2$, and set $q_1=q_2=1,024$ to sample quadrature points using quasi-Monte Carlo over $[-\pi,\pi]^{D}$ \citep{OWEN1998466}. The intrinsic form of the LBO is used for the roughness penalty computation. For the $\mathbb{T}^2$ case, we train for $10,000$ epochs with a fixed learning rate of $10^{-5}$. In the  $\mathbb{T}^4$ case, we use a cyclic learning rate under the triangular policy \citep{Smith2017}, as described in Supplemental Section S3.2, where the learning rate cycles between a maximum and minimum of $10^{-3}$ and $10^{-5}$, respectively, every $5,000$ epochs, for $10,000$ total epochs. For selecting $\tau$, we use a held-out validation set of size $0.05n$ and consider a logarithmically spaced grid of candidate solutions from $10^{-5}$ to $10^1$.

%\subsubsection{Competing Methods}

\par 
We compare our method, from here on referred to as Neural Product Manifold Density (NeuroPMD), to \revised{several competing alternatives}. I) a product kernel density estimation (KDE) with marginal kernels taken to be von Mises densities. We assume a common bandwidth for all marginal domains to avoid complex multidimensional bandwidth selection issues. The bandwidth is selected via 5-fold cross validation using the approximate ISE criteria \eqref{eqn:l2_selection_criteria}. II) Basis expansion over the tensor product eigenfunctions (TPB), i.e. model $v(x_1, \ldots,x_{D}) = \sum_{i_{1},\cdots,i_{D}}c_{i_{1},\ldots,i_{D}}\phi_{i_{1}}(x_1)\cdots\phi_{i_{D}}(x_{D})$. The expansion coefficients $\{c_{i_{1},\ldots,i_{D}}\}$ are estimated using stochastic gradient ascent on the penalized log-likelihood \eqref{eqn:map_pp}, which has been used for basis function density estimators in the literature \citep{ferraccioli2021}, and can be guaranteed to converge under some standard conditions (see Supplemental  Section S4.1 for more details). The hyperparameter $\tau$ for the TBP estimator was selected using the scheme outlined in Section~\ref{ssec:hyper_param_selection}, applied independently to all replications for $\mathbb{T}^2$. For $\mathbb{T}^{4}$, due to the computational challenges of estimating the tensor-product basis in this high-dimensional domain, the full hyperparameter optimization was performed for a single replication to select $\tau$, which was then fixed for the remaining replications to manage the computational cost. For $\mathbb{T}^2$, the marginal maximum frequencies are also set as $\lambda_{d,max}=10$ $\forall d$. However, due to the curse of dimensionality, applying this strategy to $\mathbb{T}^{4}$ results in a basis expansion with $(2(10)+1)^4 = 194,481$ unknown parameters, the estimation of which proved to be prohibitively costly for our experiments. Therefore, we set the marginal maximum frequencies to $\lambda_{d,max}=7$ $\forall d$, resulting in $(2(7)+1)^4=50,625$ coefficients. This choice was made to approximately match the parameter count of the $\mathbb{T}^{4}$ NeuroPMD architecture, which has $3(128^2 + 128) + 128 + 1=49,665$ parameters. \revised{III) The product domain squared neural families (pSNF) from \citep{tsuchida2023}. For both the $\mathbb{T}^2$ and $\mathbb{T}^{4}$ cases, the number of network parameters in pSNF was chosen to approximately match the parameter count in the corresponding NeuroPMD model. Fixed learning rates of $5\cdot 10^{-5}$ and $1\cdot 10^{-5}$ were used for the $\mathbb{T}^{2}$ and $\mathbb{T}^{4}$ case, respectively, as these were found to be the largest learning rates that ensured stable convergence in these settings. We also experimented with a triangular learning rate policy, but found it to result in instability. Supplemental Section S4.2 provides more details on the pSNF model structure, hyperparameter settings, and initialization.} We also compared our method to a ``vanilla'' MLP with RelU activations and the identity encoding function: $\boldsymbol{\eta}(\boldsymbol{x})=\boldsymbol{x}$. However, this baseline approach did not achieve competitive performance compared to any of the alternatives; detailed results are provided in Supplemental Section S5.2.
%\subsubsection{Evaluation Metrics}
\par 
To evaluate the estimation performance of all methods, we use two metrics: the normalized $L^2(\Omega)$ error (nISE) and the Fisher-Rao metric (FR), both defined with respect to the true density function, as follows:
$$
\text{nISE}(\hat{f}) = \left\| f - \hat{f}\right\|_{L^{2}(\Omega)}^2\Big/\left\| f \right\|_{L^{2}(\Omega)}^2,\quad \text{FR}(\hat{f}) =  \cos^{-1}\left(\left\langle  \sqrt{f}, \sqrt{\hat{f}}\right\rangle_{L^{2}(\Omega)}^2\right).
$$
The integrals required for these metrics are approximated differently depending on the dimensionality of the problem. For the relatively low-dimensional case of $\mathbb{T}^2$, we use the tensor product of dense marginal grids, while for the higher-dimensional case of $\mathbb{T}^4$, we rely on Monte Carlo integration. %It is worth noting that the FR is the $L^2$ Riemannian distance on $\mathbb{S}^{\infty}$, the unit sphere in the space of $L^2(\Omega)$ functions. The square-root representation of a density function, also referred to as the half density, lies in the positive orthant of $\mathbb{S}^\infty$, making the FR a natural metric for comparing density functions. 

\subsubsection{Results}
\begin{figure}[!ht]
    \centering
    \includegraphics[width=\textwidth]{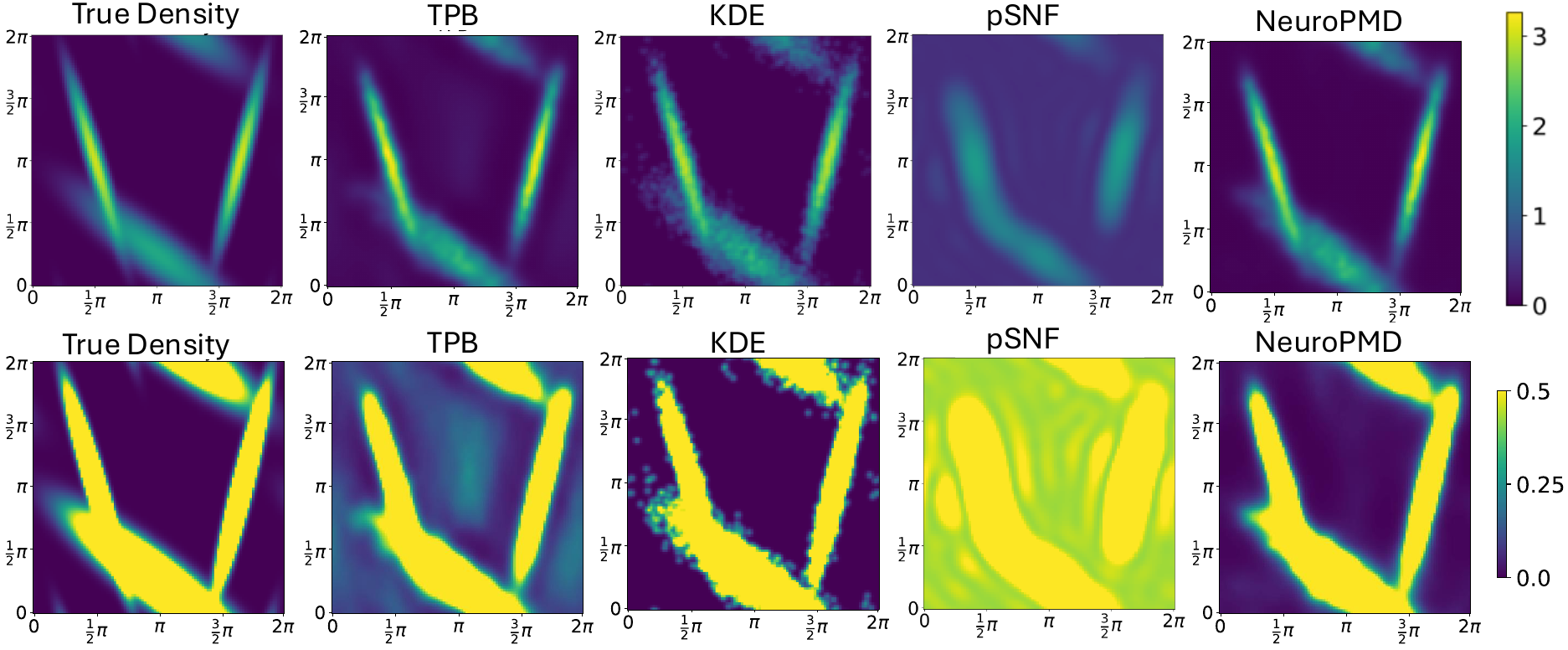}
    \caption{True (log) density function on $\mathbb{T}^2$ (left column) and estimates of (log) density from each method from a randomly selected experimental replication. Both rows present the same function visualized with different colorbars for enhanced comparison.}
    \label{fig:sim_T2}
\end{figure}

\begin{table}[ht!]
\centering
\begin{tabular}{llcc}
\toprule
Method & Metric & $\mathbb{T}^2$ & $\mathbb{T}^4$ \\
\midrule
\multirow{2}{*}{NeuroPMD} 
  & FR & $0.136 \pm (2.08\times 10^{-3})$ & $0.242 \pm (8.16 \times 10^{-3})$ \\
  & nISE & $0.0148 \pm (8.17 \times 10^{-4})$ & $0.0608 \pm (6.93 \times 10^{-3})$ \\
\midrule
\multirow{2}{*}{KDE} 
  & FR & $0.156 \pm (7.65 \times 10^{-4})$ & $1.059 \pm (1.48 \times 10^{-4})$ \\
  & nISE & $0.0235 \pm (3.96 \times 10^{-4})$ & $0.938 \pm (1.06 \times 10^{-5})$ \\
\midrule
\multirow{2}{*}{TPB} 
  & FR & $0.203 \pm (1.43\times 10^{-3})$ & $0.430 \pm (3.81 \times 10^{-4})$ \\
  & nISE & $0.0165 \pm (5.56 \times 10^{-4})$ & $0.0764\pm (3.56 \times 10^{-4})$ \\
\midrule
\multirow{2}{*}{\revised{pSNF}} 
  & \revised{FR} & \revised{$0.484 \pm (9.08\times 10^{-4})$} & \revised{$0.764\pm (1.63 \times 10^{-4})$} \\
  & \revised{nISE} & \revised{$0.352 \pm (5.37 \times 10^{-4})$} & \revised{$0.749\pm (7.84 \times 10^{-5})$} \\
\bottomrule
\end{tabular}
\caption{ \textbf{Monte Carlo average simulation results for density estimation on $\mathbb{T}^2$ and $\mathbb{T}^4$}. The table compares the performance of NeuroPMD, KDE, TPB, \revised{and pSNF} using the normalized $L^2(\Omega)$ error (nISE) and Fisher-Rao (FR) metric. Average errors and standard errors (in parentheses) demonstrate that NeuroPMD achieves the best overall performance across both metrics and domains.}
\label{tab:sim_1_2_MC_results}
\end{table}
\noindent{\textbf{$\mathbb{T}^2$ Results}: Figure~\ref{fig:sim_T2} displays the true (log) density function along with the estimates from each method for a randomly selected experimental replication. Both rows show the same functions but with different color scales for better visualization.
We observe that the KDE method performs well in capturing the highly peaked mode but struggles with more diffuse modes. This is likely due to the global bandwidth parametrization, which cannot adapt to the spatially varying anisotropy of the underlying density. This can be observed in the relative roughness in the KDE estimate of the most diffuse mode (mixture component at the bottom of the image). TPB accurately recovers the overall structure of the mixture components and captures the smoothness of the more diffuse component better than KDE, as shown in the top row. However, as displayed in the bottom row, TPB tends to overestimate in low-density regions. This tendency is reflected in its comparatively poor performance on the FR metric in Table~\ref{tab:sim_1_2_MC_results}, where the inclusion of the square root amplifies errors in these low-density areas. The issues arises due to the global, finite-rank (band-limited) representation of the Fourier functions. While increasing the marginal frequencies could eventually alleviate this problem, the resulting exponential scaling of the parameter space makes this approach impractical. Compared to the competitors, NeuroPMD offers an effective balance between capturing varying smoothness properties of the function while maintaining low bias in regions of low density. This can be observed quantitatively in Table~\ref{tab:sim_1_2_MC_results}, where NeuroPMD displays the lowest average error in both metrics across replications. Comparing TPB to NeuroPMD, the better performance of the latter can be attributed to its compositional framework. As shown in Theorem~\ref{thm:representation_space_Td}, the NeuroPMD estimator on $\mathbb{T}^{D}$ uses a basis of sinusoidal functions whose weights, frequencies, and phase shifts are adaptively learned from the data. In contrast, TPB is restricted to linear weights on a tensor-product expansion over a fixed set of frequencies, limiting its flexibility and adaptability.}
\par
\noindent{\revised{Among all methods, pSNF exhibits the weakest performance across both evaluation metrics, due largely to its architectural constraints. As noted earlier, pSNF requires the use of specific encoding and activation functions and is restricted to a single hidden layer design. While this formulation enables exact integration, the lack of flexibility results in a network that produces overly smooth estimates, as shown in Figure~\ref{fig:sim_T2}, and converges very slowly to high-frequency components, as illustrated in Figure S5 of the Supplemental Materials. This is a classic manifestation of the spectral bias and underscores the importance of the specific design choices in NeuroPMD made to combat this issue. See Supplemental Section S5.3.2 for a detailed comparison of spectral bias in pSNF and NeuroPMD.}
\par\bigskip
\noindent{\textbf{$\mathbb{T}^4$ Results}: Table~\ref{tab:sim_1_2_MC_results} provides a quantitative comparison between the methods in the high-dimensional $\mathbb{T}^{4}$ case. The results show that our NeuroPMD method consistently outperforms \revised{all of the competing} methods. A notable decline in the performance of KDE is observed when transitioning from the $\mathbb{T}^2$ to the $\mathbb{T}^4$ case. This degradation is largely attributable to the CV-based bandwidth selection, which consistently chooses a very small von-Mises concentration (i.e., a large bandwidth), resulting in significantly over-smoothed estimates.  Through post-hoc experimentation, we were able to identify some larger concentrations that did lead to lower FR and nISE. Still, none of these configurations achieved the performance levels of NeuroPMD. Identifying the optimal bandwidth for multidimensional density estimation is notoriously challenging. 
%Moreover, even if an optimal global bandwidth is determined, the resulting fits are likely to exhibit local regions of over- and/or under-smoothing, due to the spatially varying smoothness of the true density \citep{Wang2019}.
Moreover, even with an optimal global bandwidth, the fits are likely to exhibit local over- and/or under-smoothing due to spatial variations in the true density's smoothness and anisotropy. This further underscores the importance of the data-adaptive nature of NeuroPMD: its representation space in $\mathbb{T}^D$ being a large set of data-adaptive Fourier functions (Theorem~\ref{thm:representation_space_Td}), facilitating  adaptation to the localized features of the density. While TPB significantly outperforms KDE, it remains inferior to NeuroPMD, particularly in terms of the FR metric.  As observed in the $\mathbb{T}^2$ case, FR is highly sensitive to bias in low-density regions. The superior performance of NeuroPMD over TPB is particularly notable in this case since TPB actually has more parameters than NeuroPMD, highlighting the parameter efficiency of the compositional structure in \eqref{eqn:nf}. \revised{Although pSNF outperforms KDE in this higher dimensional case, it remains uncompetitive with either NeuroPMD or TPB.}}

\subsection{Real Data Analysis: Brain Connectivity}\label{ssec:rda_cc}
As outlined in Section~\ref{sec:intro}, a primary motivating application for this work is modeling the spatial distribution of white matter neural fiber connection endpoints on the brain's cortical surface. These connections, collectively referred to as the \textit{structural connectivity}, can be estimated through a combination of diffusion magnetic resonance imaging (dMRI) \citep{baliyan2016diffusion} and tractography algorithms \citep{STONGE2018524}. As illustrated in Figure~\ref{fig:structural_connectivity_data} panel A, the inferred tracts map physical connections between different brain regions, effectively reconstructing the complex structural network of the brain. Accurately estimating the density function governing the spatial pattern of connection endpoints is of substantial scientific interest, both for single-subject analysis \citep{moyer2017} and as a data representation for population-level studies \citep{mansour2022,consagra2023continuous}.
\subsubsection{Data Description, Implementation Details and Evaluation} 
In this work, we consider a randomly selected subject from the Adolescent Brain Cognitive Development (ABCD) study \citep{CASEY201843}. The structural connectivity endpoint data was inferred from the subject's diffusion and structural (T1) MRI using the SBCI pipeline \citep{cole2021}. Briefly, SBCI constructs cortical surfaces using FreeSurfer \citep{FISCHL2012774}, where each cortical hemisphere surface is represented using a dense triangular mesh consisting of 163,842 vertices. White matter fiber tracts connecting cortical surface locations are estimated using surface-enhanced tractography \citep{STONGE2018524}. In total, we observe $n=242,972$ left-hemisphere to left-hemisphere cortical connections. Each cortical hemisphere surface is approximately homeomorphic to the 2-sphere $\mathbb{S}^2$ \citep{FISCHL1999195}. Consequently, connections within a single hemisphere can be parameterized as points on the product space of 2-spheres, such that $(x_1, x_2) \in \mathbb{S}^2 \times \mathbb{S}^2$. Thus, the connectivity can be modeled as a point set on the product space $\Omega= \mathbb{S}^2 \times \mathbb{S}^2$, as shown in Figure~\ref{fig:structural_connectivity_data} panel C, and our goal is to recover the latent density function from the observed intra-hemisphere connectivity.

%\subsubsection{Implementation Details}
The eigenfunctions of the LBO on $\mathbb{S}^2$ are known as the \textit{spherical harmonics}, further details on these functions can be found in Supplemental Section S2.2. We set both maximum marginal \textit{degrees} to be 10 (resulting in $121$ harmonic basis for each marginal $\mathbb{S}^2$), $K=256$, $L=6$ hidden layers, and the width of each hidden layer to be 256. We sample the separable tensor product harmonics to form the first layer encoding. We set the batch size $b=n/2$ and $q_1=q_2=10,000$. $T=10,000$ training iterations are used under the cyclic triangular learning rate policy with minimum and maximum learning rates $10^{-5}$ and $10^{-3}$, respectively. We use the extrinsic form of the Laplacian-based roughness penalty \eqref{eqn:projected_diff_operators} and employ a centered difference scheme to numerically approximate the Hessian in the ambient space $\mathbb{R}^6$. Orthogonal basis vectors of the tangent space of $T_{x_{d}}\left(\mathbb{S}^2\right)$ can be calculated analytically (see Section S2.2), and hence the block matrix in proposition~\ref{prop:LBO_via_ambient_projection} can be computed rapidly. The regularization parameter $\tau$ is selected using the approach discussed in Section~\ref{ssec:hyper_param_selection}.
\par 
%\subsubsection{Competing Methods}
We compare our method with TPB with the same marginal ranks ($121^2$ basis functions). While specifying a product kernel on $\Omega$ is straightforward, the KDE method proves highly impractical in this setting. The difficulty arises not only from the large data size $n$, but more critically from the extremely high-resolution surface meshes ($>160,000$ vertices for each marginal $\mathbb{S}^2$) on which we wish to infer the density function. Therefore, we exclude the KDE method from our comparisons. 
\par 
Due to the absence of a ground truth density function, we evaluate the estimation performance qualitatively.  Specifically, we focus on an anatomically defined region of interest (ROI) on the brain surface and analyze the associated connectivity patterns using the concept of \textit{marginal connectivity}:
$$
\tilde{f}_E(x_2) = \int_{x_1 \in E}\tilde{f}(x_1,x_2)dx_1,
$$
where $E$ is a predefined brain ROI, and the integral is approximated numerically using the high-resolution spherical mesh. This evaluation approach enables us to examine the estimated connectivity function through marginal profiles, which are functions on $\mathbb{S}^2$ and are more amenable to visualization and interpretation. For this study, we focus on the medial orbitofrontal cortex (MOFC) ROI, displayed in red in the top row of Figure~\ref{fig:rda_results}. The MOFC is an important brain region involved in high-order cognitive processing, and its connectivity is impacted in psychiatric diseases, including depression \citep{Rolls2020}. 

\begin{figure}[!ht]
    \centering
    \includegraphics
    [width=0.8\textwidth]{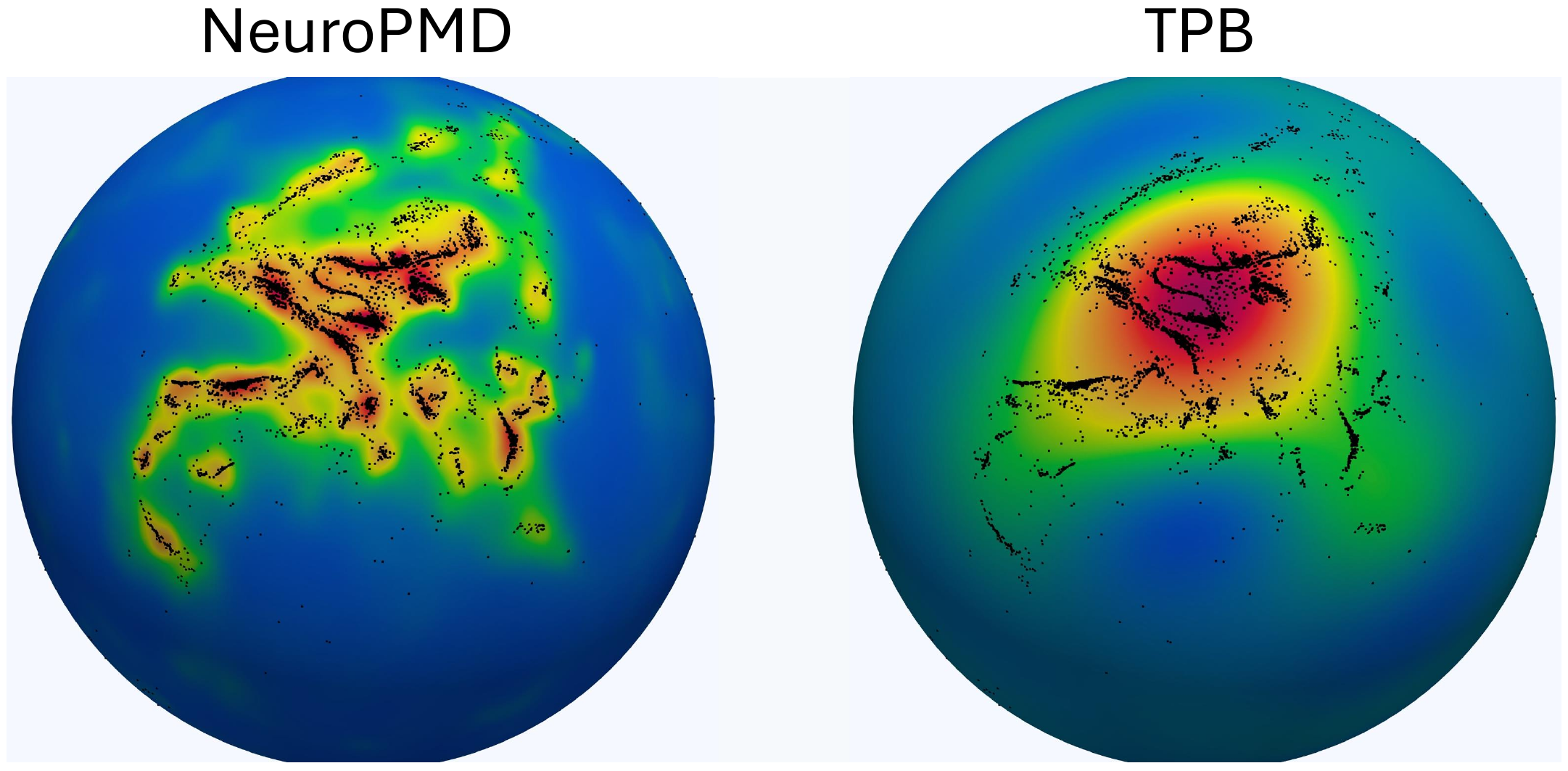}
    \caption{Marginal density function estimates for connections from the medial orbitofrontal cortex (MOFC), generated using our method (left) and tensor product basis (right). Black dots represent the endpoints connected to the MOFC. Color scales are normalized within each image to emphasize differences in the shape of the functions.}
    \label{fig:rda_marg_density}
\end{figure}
\subsubsection{Results}
\begin{figure}[!ht]
    \centering
    \includegraphics[width=\textwidth]{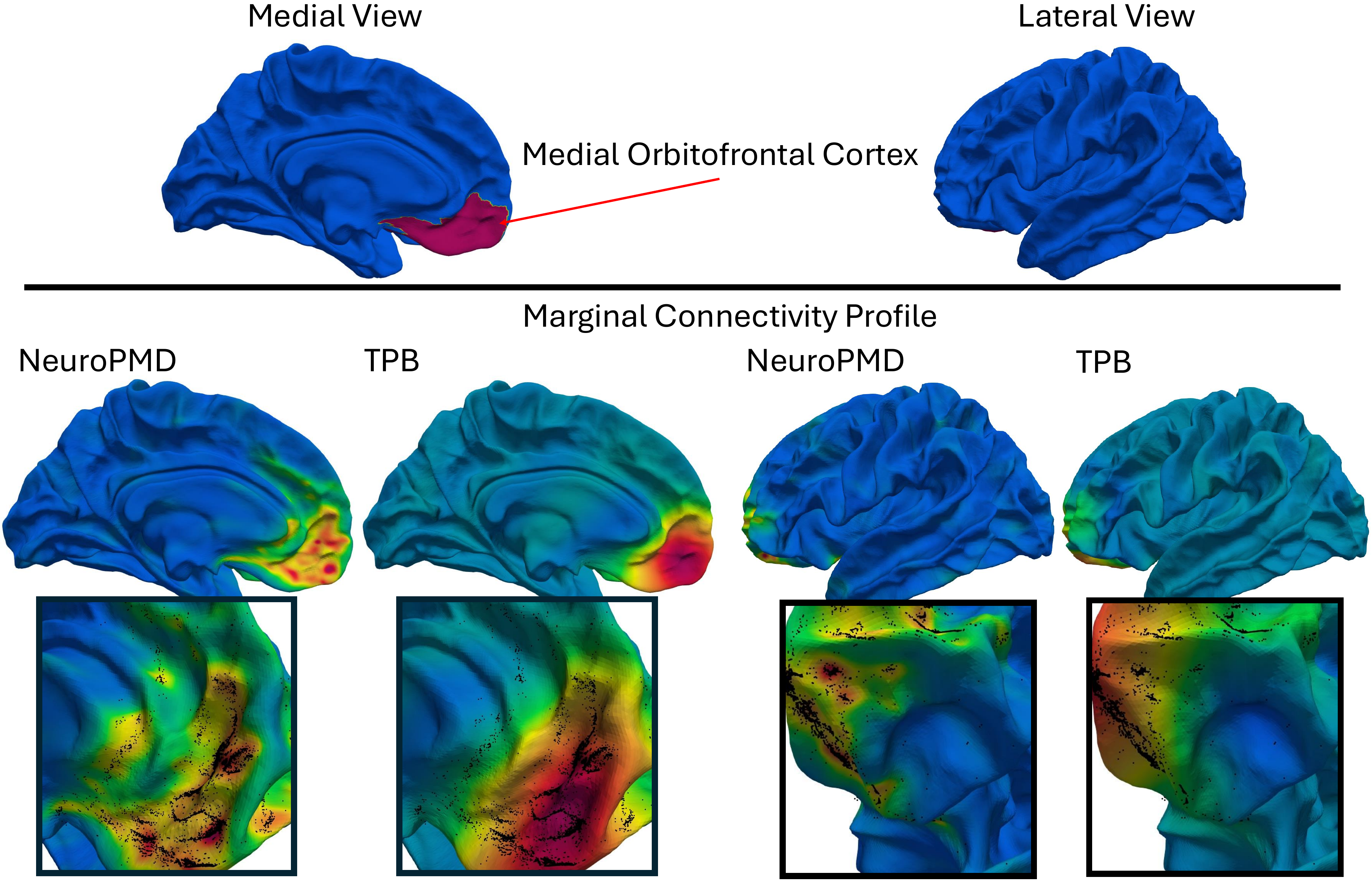}
    \caption{Same as Figure \ref{fig:rda_marg_density}, with marginal density functions mapped to the cortical surface. \textbf{Top Row}: The MOFC is highlighted in red in both medial (left) and lateral (right) views.
\textbf{Middle Row}: Comparing the estimated marginal density functions by the NeuroPMD and TPB method. \textbf{Bottom Row}: Close-up views of the marginal density functions mapped to the cortical surface. Endpoints of fiber curves connecting to the MOFC are marked in black.}
    \label{fig:rda_results}
\end{figure}

Figure~\ref{fig:rda_marg_density} shows the marginal density functions $\tilde{f}_E$ for $E =$MOFC for both NeuroPMD and TPB. The black points in Figure~\ref{fig:rda_marg_density} are endpoints of fiber curves whose other endpoints are located in the MOFC. Comparing the connection point pattern to the estimated $\tilde{f}_E$, we observe that our method successfully detects nearly all marginal modes and effectively adapts to the data's highly anisotropic shape. In contrast, the TPB method significantly over-smooths the data, failing to capture the multi-modal and high-frequency structure. 
\par 
To facilitate biological interpretation, we mapped the marginal density function and connection points back to the cortical surface in Figure~\ref{fig:rda_results}.  The middle and bottom two rows again demonstrate NeuroPMD's ability to capture the fine details of the connection patterns.  In comparison, the TPB estimates appear significantly over-smoothed. The MOFC region is generally known to have strong connections with the anterior cingulate gyrus, pregenual cingulate cortex, frontal pole, and lateral orbitofrontal cortex \citep{Beckmann2009,hsu2020}. All of these expected connection patterns are clearly reflected in the marginal density estimates of NeuroPMD. Additionally, studies suggest complex connectivity patterns within the MOFC (Chapter 3; \cite{Zald2006}). The detailed connectivity within the MOFC is evident in the zoomed-in region on the left, showing a distinct multi-modal and anisotropic spatial structure. 
\par  
Supplemental Figures S8 and S9 in  Section S6 highlight the variability in intra- and inter-MOFC connection patterns across several ABCD subjects by showing their marginal density functions estimated by NeuroPMD. It is notable that the within MOFC structure is entirely lost when using traditional \revised{atlas}-based network representations, where between-ROI connections are reduced to a single summary statistic (e.g., a count), and within-ROI connections are ignored \citep{chung2021}. We speculate that integrating the proposed deep neural field model as a data representation could enhance the power of downstream neuroscientific tasks, such as cognitive trait and neuropsychiatric disease prediction, compared to traditional \revised{atlas}-based models. \revised{Relatedly, recent work has explored predictive models operating directly in the weight space of representation networks \citep{Navon2023}, opening this as an exciting direction for future work.}

\section{Conclusion}
\label{SEC:conclusion}

This work introduces a novel deep neural network methodology for density function estimation on product Riemannian manifolds. By carefully designing the network architecture and stochastic gradient estimation, our method avoids the \revised{computational} curse of dimensionality that afflicts traditional approaches to flexible density estimation in high-dimensional settings. To promote convergence and properly regularize our estimates, a roughness penalty based on the Laplace-Beltrami operator is incorporated. To our knowledge, this is the first approach to use a deep neural network for flexible density estimation on product manifold domains. Simulation studies demonstrate improved performance over traditional approaches, particularly in high-dimensional domains, and a real-world application to a challenging neuroscience dataset shows its practical utility in revealing detailed neural connectivity patterns on the brain's surface.

\par\bigskip
\noindent{\textbf{\large Code}: Code implementing our model and algorithms has been made publicly available: 
{\small \textbf{\url{https://github.com/Will-Consagra/NeuroPMD}}}}.

\bibliographystyle{apalike}
\bibliography{refs}

\clearpage
\pagebreak

\begin{center}
{\huge\textbf{SUPPLEMENTAL MATERIAL}}
\end{center}

\setcounter{equation}{0}
\setcounter{figure}{0}
\setcounter{table}{0}
\setcounter{section}{0}
\setcounter{page}{1}
\makeatletter
\renewcommand{\theequation}{S\arabic{equation}}
\renewcommand{\thefigure}{S\arabic{figure}}
\renewcommand{\thetable}{S\arabic{table}}
\renewcommand{\thesection}{S\arabic{section}}
\renewcommand{\bibnumfmt}[1]{[S#1]}
\renewcommand{\citenumfont}[1]{S#1}
\renewcommand{\theequation}{S.\arabic{equation}}
\renewcommand{\thesection}{S\arabic{section}}
\renewcommand{\thesubsection}{S\arabic{section}.\arabic{subsection}}
\renewcommand{\thetable}{S\arabic{table}}
\renewcommand{\thefigure}{S\arabic{figure}}
\renewcommand{\thetheorem}{S\arabic{theorem}}
\renewcommand{\theproposition}{S\arabic{proposition}}
\renewcommand{\thelemma}{S\arabic{lemma}}
\renewcommand{\thecorollary}{S\arabic{corollary}}
\renewcommand{\thealgorithm}{S\arabic{algorithm}}

\section{Proofs}\label{sec:proofs}

\subsection{Proposition~\ref{prop:unbiased_grads}}
\begin{proof}
Note that 
$$
\begin{aligned}
    \mathbb{E}\left[ \boldsymbol{a}^{b}(\boldsymbol{\theta})\right] &= \frac{1}{b}\sum_{i=1}^b\mathbb{E}\left[\frac{\partial}{\partial\boldsymbol{\theta}}v_{\boldsymbol{\theta}}(x_{1i}, \ldots,x_{Di})\right] = \frac{1}{b}\sum_{i=1}^b\mathbb{E}\left[\sum_{j=1}^n\mathbb{I}\{i=j\} \frac{\partial}{\partial\boldsymbol{\theta}}v_{\boldsymbol{\theta}}(x_{1j}, \ldots,x_{Dj})\right] \\
    &=\frac{1}{b}\sum_{i=1}^b\sum_{j=1}^np(i=j) \frac{\partial}{\partial\boldsymbol{\theta}}v_{\boldsymbol{\theta}}(x_{1j}, \ldots,x_{Dj}) =\frac{1}{b}\sum_{i=1}^b\frac{1}{n}\sum_{j=1}^n \frac{\partial}{\partial\boldsymbol{\theta}}v_{\boldsymbol{\theta}}(x_{1j}, \ldots,x_{Dj}) \\
    &=\frac{1}{n}\sum_{j=1}^n \frac{\partial}{\partial\boldsymbol{\theta}}v_{\boldsymbol{\theta}}(x_{1j}, \ldots,x_{Dj}), \\
\end{aligned}
$$
where the expectation is being taken with respect to the uniform distribution over the indices $\{1, ..., n\}$. Now, considering the latter two terms of the sum, $\boldsymbol{b}^{q_{1}}(\boldsymbol{\theta})$ and $\boldsymbol{c}_{\tau}^{q_{2}}(\boldsymbol{\theta})$, and taking the expectation with respect to the uniform measure on $\Omega$, we have 
$$
\begin{aligned}
    \mathbb{E}[\boldsymbol{b}^{q_{1}}(\boldsymbol{\theta})] &=  \mathbb{E}\left[\frac{\text{Vol}(\Omega)}{q_{1}}\sum_{j=1}^{q_{1}}\frac{\partial}{\partial\boldsymbol{\theta}}\exp(v_{\boldsymbol{\theta}}(x_{1j}, \ldots, x_{Dj}))\right] \\
   &=\frac{\text{Vol}(\Omega)}{q_{1}}\sum_{j=1}^{q_{1}}\mathbb{E}\left[\frac{\partial}{\partial\boldsymbol{\theta}}\exp(v_{\boldsymbol{\theta}}(x_{1j}, \ldots, x_{Dj}))\right] \\
   &=\frac{\text{Vol}(\Omega)}{q_{1}}\sum_{j=1}^{q_{1}}\frac{1}{\text{Vol}(\Omega)}\int_{\Omega}\frac{\partial}{\partial\boldsymbol{\theta}}\exp(v_{\boldsymbol{\theta}}(x_{1}, \ldots, x_{D})d\omega\\
   &= \int_{\Omega}\frac{\partial}{\partial\boldsymbol{\theta}}\exp(v_{\boldsymbol{\theta}}(x_{1}, \ldots, x_{D}))d\omega,
\end{aligned}
$$
and thus similarly
$$
\mathbb{E}\left[\boldsymbol{c}_{\tau}^{q_{2}}(\boldsymbol{\theta})\right] = \tau\int_{\Omega}\frac{\partial}{\partial\boldsymbol{\theta}}\left[\Delta_{\Omega} v_{\boldsymbol{\theta}}\right]^2d\omega.
$$
Hence, 
$$
\begin{aligned}
    \mathbb{E}\left[\boldsymbol{a}^{b}(\boldsymbol{\theta}) - \boldsymbol{b}^{q_{1}}(\boldsymbol{\theta}) - \boldsymbol{c}_{\tau}^{q_{2}}(\boldsymbol{\theta})\right] 
    &=\frac{1}{n}\sum_{j=1}^n \frac{\partial}{\partial\boldsymbol{\theta}}v_{\boldsymbol{\theta}}(x_{1j}, \ldots,x_{Dj}) - \int_{\Omega}\frac{\partial}{\partial\boldsymbol{\theta}}\exp(v_{\boldsymbol{\theta}})d\omega - \tau\int_{\Omega}\frac{\partial}{\partial\boldsymbol{\theta}}\left[\Delta_{\Omega} v_{\boldsymbol{\theta}}\right]^2d\omega\\
    &=\frac{\partial}{\partial\boldsymbol{\theta}}\mathcal{L}(o,\boldsymbol{\theta}),
\end{aligned}
$$
as desired. 
\end{proof}
\subsection{Proposition~\ref{prop:LBO_via_ambient_projection}}
For convenience, we begin with notation and preliminaries. Let $\mathcal{M}\subset\mathbb{R}^{m}$ be a $p$-dimensional closed manifold embedded in ambient space $\mathbb{R}^{m}$. Following the formulation in \citeSupp{Harlim2023}, we approximate the point-wise application of differential operators on smooth function $v:\mathcal{M}\mapsto\mathbb{R}$ by first calculating a differential operator in the ambient space $\mathbb{R}^{m}$, followed by a projection onto the local tangent space $T_{x}(\mathcal{M})$ of the manifold. This allows us to leverage the automatic differentiation of the estimator \eqref{eqn:nf} or simple discrete differential operator approximations in the ambient space followed by a projection to approximate the operators of interest. For any $x\in\mathcal{M}$, denote the homeomorphic local parameterization $l:O\subset\mathbb{R}^{p}\mapsto V\bigcap \mathcal{M}
\subset\mathbb{R}^{m}$, for open sets $O,V$. Let $\gamma = l^{-1}(x)$, then a basis for $T_{x}(\mathcal{M})$ is defined by 
$\{\frac{\partial l}{\partial \gamma_{i}}(\gamma)\}_{i=1}^p$. Denote the matrix $\boldsymbol{A}_{l}(x)\in\mathbb{R}^{m\times p}$, whose columns are given by the vectors $\frac{\partial l}{\partial \gamma_{i}}(\gamma)$. Define the projection matrix
$$
\boldsymbol{P} := \boldsymbol{P}(x) = \boldsymbol{A}_{l}(x)(\boldsymbol{A}_{l}^{\intercal}(x)\boldsymbol{A}_{l}(x))^{-1}\boldsymbol{A}_{l}^{\intercal}(x),
$$
where from here on the dependence on $x$ is assumed and dropped for clarity. Let $\{t_{i}\}_{i=1}^p$ be a set of orthogonal vectors in $\mathbb{R}^{m}$ that span 
$T_{x}(\mathcal{M})$ and denote $\boldsymbol{T}=[t_1,\ldots,t_{p}]\in\mathbb{R}^{m\times p}$. From proposition 2.1 in \citeSupp{Harlim2023}, we have that $\boldsymbol{P} = \boldsymbol{T}\boldsymbol{T}^{\intercal}$ and the manifold gradient can be written as 
\begin{equation}\label{eqn:manifold_grad}
    \nabla_{\mathcal{M}}v(x) = \boldsymbol{P}\nabla_{\mathbb{R}^{m}}v(x).
\end{equation}
Using this definition of the gradient, the Laplace-Beltrami operator has the following formulation in ambient space coordinates 
\begin{equation}\label{eqn:manifold_laplace}
\Delta_{\mathcal{M}}v(x) = \text{div}_{\mathcal{M}}\nabla_{\mathcal{M}}v(x) =  (\boldsymbol{P}\nabla_{\mathbb{R}^{m}})\cdot (\boldsymbol{P}\nabla_{\mathbb{R}^{m}}) v(x),
\end{equation}
where $\nabla_{\mathbb{R}^{m}}$ denotes the standard Euclidean gradient in $\mathbb{R}^m$. Given these preliminaries, we provide the proof to Proposition~\ref{prop:LBO_via_ambient_projection}.
\begin{proof}
From \eqref{eqn:manifold_grad}, we can write the gradient as
$$
\nabla_{\mathcal{M}}v(x) = \boldsymbol{P} \nabla_{\mathbb{R}^{n}}v(x) = \begin{pmatrix}
                        \sum_{i=1}^m \boldsymbol{P}_{1i}\partial_i v(x) \\
                        \vdots \\
                        \sum_{i=1}^m \boldsymbol{P}_{mi}\partial_i v(x). 
                        \end{pmatrix}
$$
Clearly, this produces a vector field on the manifold. Denote $\boldsymbol{P}_{i,\cdot}\in\mathbb{R}^m$ as the row vector corresponding to the $i$'th row of $\boldsymbol{P}$, and denote $\boldsymbol{H}_{\mathbb{R}^{m}}(v)(x)\in\mathbb{R}^{m\times m}$ to be the standard Euclidean Hessian matrix of $v$ at $x$, with $l,i$ element $\partial_{li}v(x)$. Then the manifold divergence can be written as 
    $$
     \begin{aligned}
        \text{div}_{\mathcal{M}}\nabla_{\mathcal{M}}v(x) &= \boldsymbol{P} \nabla_{\mathbb{R}^{m}}\cdot \nabla_{\mathcal{M}}v(x) = \boldsymbol{P} \nabla_{\mathbb{R}^{m}}\cdot \begin{pmatrix}
                            \sum_{i=1}^m \boldsymbol{P}_{1i}\partial_i v(x) \\
                            \vdots \\
                            \sum_{i=1}^m \boldsymbol{P}_{mi}\partial_i v(x). 
                            \end{pmatrix}\\
        &= \text{sum}\Big(\begin{pmatrix}
                    \boldsymbol{P}_{11}\partial_1(\sum_{i=1}^m \boldsymbol{P}_{1i}\partial_i v(x)) + \ldots + \boldsymbol{P}_{1m} \partial_m(\sum_{i=1}^m \boldsymbol{P}_{1i}\partial_i v(x))\\
                            \vdots \\
                            \boldsymbol{P}_{m1}\partial_1(\sum_{i=1}^m \boldsymbol{P}_{mi}\partial_i v(x)) + \ldots + \boldsymbol{P}_{mm} \partial_m(\sum_{i=1}^m \boldsymbol{P}_{mi}\partial_i v(x))
                            \end{pmatrix}\Big)\\
            &= \text{sum}\Big(\begin{pmatrix}
                    (\sum_{i=1}^m \boldsymbol{P}_{11}\boldsymbol{P}_{1i}\partial_1\partial_i v(x)) + \ldots + (\sum_{i=1}^m \boldsymbol{P}_{1m}\boldsymbol{P}_{1i}\partial_m\partial_i v(x))\\
                            \vdots \\
                            (\sum_{i=1}^m \boldsymbol{P}_{m1}\boldsymbol{P}_{mi}\partial_1\partial_i v(x)) + \ldots + (\sum_{i=1}^m\boldsymbol{P}_{mm}\boldsymbol{P}_{mi}\partial_m\partial_i v(x))
                \end{pmatrix}\Big)\\
                &= \text{sum}\Big(\begin{pmatrix}
                    \sum_{l=1}^m\sum_{i=1}^m \boldsymbol{P}_{1l}\boldsymbol{P}_{1i}\partial_l\partial_i v(x))\\
                            \vdots \\
                            \sum_{l=1}^m\sum_{i=1}^m \boldsymbol{P}_{ml}\boldsymbol{P}_{mi}\partial_l\partial_i v(x)
                \end{pmatrix}\Big) = \text{sum}\Big(\begin{pmatrix}
                    \sum_{l=1}^m\sum_{i=1}^m \boldsymbol{P}_{1l}\boldsymbol{P}_{1i}\partial_{li}v(x))\\
                            \vdots \\
                            \sum_{l=1}^m \sum_{i=1}^m \boldsymbol{P}_{ml}\boldsymbol{P}_{mi}\partial_{li}v(x) 
                \end{pmatrix}\Big)\\
                &= \text{sum}\Big(\begin{pmatrix}
                    \boldsymbol{P}_{1,\cdot}\boldsymbol{H}_{\mathbb{R}^{m}}(v)(x)\boldsymbol{P}_{1,\cdot}^{\intercal}\\
                            \vdots \\
                            \boldsymbol{P}_{m,\cdot}\boldsymbol{H}_{\mathbb{R}^{m}}(v)(x)\boldsymbol{P}_{m,\cdot}^{\intercal} 
                \end{pmatrix}\Big) = \sum_{i=1}^m
                    \boldsymbol{P}_{i,\cdot}\boldsymbol{H}_{\mathbb{R}^{m}}(v)(x)\boldsymbol{P}_{i,\cdot}^{\intercal},\\
        \end{aligned} 
    $$
    where sum$()$ is the sum operator on the vector. Hence, the Laplace-Beltrami operator is given as:
    \begin{equation}\label{eqn:marginal_manifold_LBO}
    \begin{aligned}
            \Delta_{\mathcal{M}}[v](x) = \text{div}_{\mathcal{M}}\nabla_{\mathcal{M}}[v](x) = \sum_{i=1}^m
                    \boldsymbol{P}_{i,\cdot}\boldsymbol{H}_{\mathbb{R}^{m}}(v)(x)\boldsymbol{P}_{i,\cdot}^{\intercal}.\\
    \end{aligned}
    \end{equation}
    With slight abuse of notation, we now define $v:\Omega\mapsto\mathbb{R}$. For a product manifold $\Omega=\bigtimes_{d=1}^D\mathcal{M}_d$, where $\mathcal{M}_{d}\subset \mathbb{R}^{m_{d}}$ are $p_{d}$ dimension smooth manifolds, recall that the Laplace-Beltrami operator is given by $\Delta_{\Omega} = \sum_{d=1}^D\Delta_{\mathcal{M}_{d}}$ \citepSupp{Canzani2013}. Denote the block matrix 
    $$
    \boldsymbol{P} = \text{BlockDiag}\left(\boldsymbol{P}^{(1)}, ..., \boldsymbol{P}^{(D)}\right) = \begin{pmatrix}
        \boldsymbol{P}^{(1)} & 0           & \cdots & 0           \\
        0           & \boldsymbol{P}^{(2)} & \cdots & 0           \\
        \vdots      & \vdots       & \ddots & \vdots      \\
        0           & 0            & \cdots & \boldsymbol{P}^{(D)}
        \end{pmatrix},
    $$
    where $\boldsymbol{P}^{(d)}$ is the projection matrix for the marginal manifold $\mathcal{M}_{d}$. Combining this with definition \eqref{eqn:marginal_manifold_LBO}, denoting $M=\sum_{d=1}^Dm_{d}$, the final result follows directly as 
    $$
    \begin{aligned}
        \Delta_{\Omega}[v] &= \sum_{d=1}^D\Delta_{\mathcal{M}_{d}}[v] = 
        \sum_{d=1}^D\sum_{i_{d}=1}^{m_{d}}
        \boldsymbol{P}_{i_{d},\cdot}^{(d)}\boldsymbol{H}_{\mathbb{R}^{m_{d}}}(v)[\boldsymbol{P}_{i_{d},\cdot}^{(d)}]^{\intercal} = \sum_{i=1}^{M} \boldsymbol{P}_{i,\cdot}\boldsymbol{H}_{\mathbb{R}^{M}}[v]\boldsymbol{P}_{i,\cdot}^{\intercal} \\
        %&= \sum_{i=1}^M(e_i^{\intercal}\boldsymbol{P})\boldsymbol{H}_{\mathbb{R}^{M}}[v](\boldsymbol{x})(\boldsymbol{P}^{\intercal}e_i) = \sum_{i=1}^Me_i^{\intercal}\left[\boldsymbol{P}\boldsymbol{H}_{\mathbb{R}^{M}}[v](\boldsymbol{x})\boldsymbol{P}^{\intercal}\right]e_i \\
        %&= \text{trace}(\boldsymbol{P}\boldsymbol{H}_{\mathbb{R}^{M}}[v](\boldsymbol{x})\boldsymbol{P}) = \text{trace}(\boldsymbol{P}\boldsymbol{H}_{\mathbb{R}^{M}}[v](\boldsymbol{x})).
    \end{aligned}
    $$
\end{proof}

\subsection{Proof of Theorem \ref{thm:representation_space_Td}}
\label{APP:representThms}
Before we prove the theorem, we provide some necessary background. We can parameterize $\mathbb{S}^1$ by intrinsic coordinates $[0,2\pi)$, with $0$ and $2\pi$ identified. Denote $x \in \mathbb{S}^1$, then a general solution to \eqref{eqn:lbo_eigenfunctions} is given by: 
$
\begin{aligned}
\phi_{2i-z}(x) &= \frac{1}{\sqrt{\pi}}\cos\left(ix -\frac{z\pi}{2}\right), i \in \mathbb{Z}_{+}, z\in\{0,1\},
\end{aligned}
$
with eigenvalues $i^2$ for $i \neq 0$, $\phi_0(x) = \frac{1}{\sqrt{2\pi}}$ with eigenvalue zero, and $z$ controls the choice between $\sin$ and $\cos$. See Section~\ref{ssec:torus_geometry} for a derivation. Forming the tensor products, the non-trivial eigenfunctions on $\mathbb{T}^{D}$ with maximum marginal eigenvalue $i_{d,max}^2$ are given by
\begin{equation}\label{eqn:rtp_torus}
	\psi(\bs x) \in  \left\{\pi^{-D/2}\prod_{d=1}^D\cos(i_{d}x_d - \frac{z_{d}\pi}{2}):1\le i_{d}^2\le i_{d,max}^2, (z_{1},\ldots,z_{D})\in\{0,1\}^{D}\right\},
\end{equation}
where $\bs{x}:=(x_1,\ldots,x_D)\in \mathbb{T}^{D}$, 
with associated (non-zero) eigenvalues $\lambda=\sum_{d=1}^Di_{d}^2$. 
\par 
Without loss of generality, we let $\boldsymbol{A}_{\lambda} = \boldsymbol{I}$, $\forall \lambda$, so $\boldsymbol{\eta}$ contains $K$ randomly sampled elements without replacement from the set \eqref{eqn:rtp_torus}. Denote the $k$-th sampled tensor product basis function as
\begin{equation}\label{eqn:eta_torus_encoding}
\psi_{k}(\bs x) = \prod_{d=1}^D\phi_{i_{k,d}}(x_d)= \pi^{-D/2}\prod_{d=1}^D\cos(i_{k,d}x_d -z_{k,d}\pi/2), z_{k,d}\in \{0,1\}, 
\end{equation}
where  $1\leq k\leq K$ serves as the index for the $k$-th basis function, and
$(i_{k,1},\ldots, i_{k,D})\in\bigtimes_{d=1}^D\{1,\ldots,i_{d,max}\}$ is  the corresponding {\it frequency vector}.
\par 
In the following lemma, we prove that tensor product basis $\psi_{k}(\bs x)$ can be written as the sum of Fourier mappings that take the whole $\bs x$ vector as input.
\begin{lemma}
	\label{LEM:pD}
	Denote $p_D(\bs x) = \prod_{d = 1}^D \cos({w_{d}} x_{d} + a_d)$, with $D\in\mathbb Z^+$ and $w_d \in\mathbb Z^+, d=1, \ldots, D$, then
	\begin{align}
		\label{EQN:pD}
		p_D(\bs x) = \frac{1}{ |\mathcal W_D|} \sum_{\bs w_{\ell} \in \mathcal W_D} \cos (\langle \bs w_{\ell}, \bs x\rangle + b_{\ell}),  b_{\ell}\in\R,
	\end{align}
	where $\mathcal W_D= \{({w_1}, c_{\ell2} w_{2}, \ldots, c_{\ell D}{w_D}): c_{\ell d} \in \{-1 , 1\} \text{ for } d = 2,\ldots, D\}$ is a set of frequencies; 
    $1 \leq \ell \leq  2^{D-1}$ is the index for the frequency vector in $\mathcal W_D$ and $\bs w_{\ell}$ denotes the $\ell$-th frequency vector;  
    $b_\ell = a_1 + \sum_{d=2}^D c_{\ell d} a_d $ is the corresponding phase to $\bs w_\ell$; and $|\mathcal W_D| = 2^{D-1}$ is the cardinality of $\mathcal W_D$.
\end{lemma}
\begin{proof}[Proof of Lemma \ref{LEM:pD}]
The lemma can be proved by induction along with $\cos a \cos b = (\cos(a+b) + \cos(a-b))/2$. 

\textbf{Base case.} When $D = 1$, $p_1(x) = \cos(w_1 x +a_1)$. Therefore, the corresponding set of frequencies $\mathcal W_1 = \{(w_1)\}$.
When $D = 2$, $p_2(\bs x) = \{\cos(w_1 x_1 + w_2x_2 + a_1 + a_2) + \cos(w_1x_1 - w_2 x_2 + a_1-a_2)\}/2$, indicating $\mathcal W_2 = \{(w_1, w_2), (w_1, - w_2)\}$, and $b_\ell = a_1 + c_{\ell 2} a_2$ where $c_{\ell 2} = 1$ if $\bs w_\ell = (w_1,w_2)$, and $c_{\ell 2} = -1$ if $\bs w_\ell = (w_1,-w_2)$. 

\textbf{Inductive step}.  
For clarity, we write $\bs w_{\ell, D}\equiv \bs w_\ell$ to be the elements of $\mathcal W_D$ and denote $\bs x_{1:D}\in \mathbb T^D$ to be the vector consisting of the first $D$ elements of $\bs x\in\mathbb T^{D+1}$.
Suppose \eqref{EQN:pD} holds for some $D$ and $D\geq 2$, we derive the form of $p_{D+1}(\bs x)$ from $p_D(\bs x_{1:D})$ as follows.
\begin{align*}
	p_{D+1} (\bs x)
	& = p_D  (\bs x_{1:D})\cdot \cos(w_{D+1}x_{D+1} +a_{D+1})\\
	& = 2^{-D+1} \sum_{\ell = 1}^{2^{D-1}}\cos(\langle \bs w_{\ell,D}, \bs x_{1:D}\rangle + b_{\ell,D}) \cos(w_{D+1} x_{D+1} +a_{D+1})\\
	& = 2^{-D} \sum_{\ell =1}^{2^{D-1}} 
	\{
	\cos(\langle \bs w_{\ell,D}, \bs x_{1:D}\rangle + b_{\ell,D}+ w_{D+1} x_{D+1} +a_{D+1})\\
	&\hspace{2cm}+
	\cos(\langle \bs w_{\ell,D}, \bs x_{1:D}\rangle + b_{\ell,D}- w_{D+1} x_{D+1} -a_{D+1})
	\}\\
	& =2^{-D} \sum_{\ell =1}^{2^{D-1}} 
	\{
	\cos(\langle (\bs w_{\ell,D},w_{D+1}), \bs x\rangle + b_{\ell,D} +a_{D+1})
	+
	\cos(\langle (\bs w_{\ell,D},-w_{D+1}), \bs x\rangle + b_{\ell,D} -a_{D+1})
	\}\\
	& =  2^{-D} \sum_{\ell' =1}^{2^{D}} 
	\cos(\langle (\bs w_{\ell',D+1}, \bs x\rangle + b_{\ell', D+1}),
\end{align*}
where 
$\bs w_{\ell', D+1} = (\bs w_{\lceil\ell'/2 \rceil, D} , (-1)^{\ell'} w_{D+1})  \in \mathcal W_{D+1},
b_{\ell', D+1} = b_{\lceil\ell'/2 \rceil, D} + (-1)^{\ell'} a_{D+1}
$.
The proof is completed.
\end{proof}

\begin{remark}\label{rem2}
As a consequence of Lemma \ref{LEM:pD}, we can write ${\psi}_{k}(\boldsymbol x)$ in (\ref{eqn:eta_torus_encoding}) as 
$$\psi_{k}(\bs x) \propto \frac{1}{|\mathcal W_{k}|}\sum_{\bs w_{k\ell} \in \mathcal W_{k}} \cos  (\langle \bs w_{k\ell}, \bs x\rangle + b_{k\ell}), b_{k\ell}\in\R, $$ 
where, with slight abuse of notation, $\mathcal W_{k}=\{({i_{k,1}}, c_{\ell 2}{i_{k,2}}, \ldots, c_{\ell D}{i_{k,D}}): c_{\ell d} \in\{-1,1\}, d = 2,\ldots, D\}$;
$1\leq k\leq K$ is the index for the $k$-th tensor product basis;
and $b_{k\ell} = a_{k1} + \sum_{d=2}^D c_{\ell d} a_{kd} $ is the corresponding phase to $\bs w_{k\ell}$;
and $a_{kd} = -z_{k,d} \pi/2$ is the phase for the $k$-th basis function along the $d$-th dimension. \revised{Note that the cardinality of the set $\mathcal W_k$ is equal to the number of all possible combinations of $\{-1,1\}$ in $(D-1)$ dimensions, i.e. $2^{D-1}$. }
\end{remark}

Next, we introduce two lemmas that are crucial to the proof for Theorem \ref{thm:representation_space_Td}, showing the power functions $\cos(x)^m$ can be written as $\sum_{w\in\mathcal H}\cos(wx)$, where the frequency set $\mathcal H$ is related to $m$. 

\begin{lemma}\label{lem1}
Denote $\psi_{k}(\bs x) = \cos(\langle \bs w_{k},\bs x\rangle + b_{k}), \psi_{j}(\bs x) = \cos(\langle \bs w_{j},\bs x\rangle + b_{j})$, where $\{\bs w_{k} \in \mathbb{Z}^D\}_{k \in \mathcal{K}}$, and $\{\bs w_{j} \in \mathbb{Z}^D\}_{j \in \mathcal{J}}$ are two collections of frequency vectors;
and $\{b_{k} \in \mathbb{R}\}_{k \in \mathcal{K}}$ and $\{b_{j} \in \mathbb{R}\}_{j \in \mathcal{J}}$ are two collections of scalar phases, indexed by $\mathcal{K}=\{1, \ldots, K\}, \mathcal{J} = \{1,\ldots, J\}$.
Furthermore, let $\{\beta_{1k} \in \mathbb{R}\}_{k \in \mathcal{K}}$ and $\{\beta_{2j} \in \mathbb{R}\}_{j \in \mathcal{J}}$ be two sets of scalar coefficients and $\bs x \in \mathbb{T}^D$. 
Then,
\begin{align*}
\left\{\sum_{k\in\mathcal K}\beta_{1k}\psi_{k}(\bs x) \right\}
\left\{\sum_{j\in\mathcal J}\beta_{2j}\psi_{j}(\bs x) \right\}
% =\sum_{(k, j)\in\mathcal K\otimes \mathcal J} 
% \tilde\beta_{(k,j)}\psi_{(k,j)}(\bs x)
= \sum_{\tilde{\bs w}\in\mathcal D} \beta_{\tilde{\bs w}} \psi_{\tilde{\bs w}}(\bs x)
,
\end{align*}
where 
$\beta_{\tilde{\bs w}}  = \frac{1}{2}\beta_{1k}\beta_{2j}$,
$\psi_{\tilde{\bs w}}(\bs x) =  \cos(\langle \tilde{\bs w},\bs x\rangle + b_{\tilde{\bs w}})$,
$\mathcal D \equiv \mathcal D(\{\bs w_{k}\}_{k \in \mathcal K}, \{\bs w_{j}\}_{j \in \mathcal J})= 
\{\tilde{\bs w} = \bs w_{k} \pm \bs w_{j}, k\in\mathcal K, j\in \mathcal J\}$,
and 
$\tilde b_{\tilde{\bs w}} = b_{k} + c_{\tilde{\bs w}} b_{j}$,
with 
$k, j, c_{\tilde {\bs w}}$ induced by the form of $\tilde {\bs w}$, i.e. $c_{\tilde{\bs w}} = 1$ if $\tilde {\bs w} = \bs w_{k} + \bs w_{j}$, and $c_{\tilde{\bs w}} = -1$ if $\tilde {\bs w} = \bs 
 w_{k} - \bs w_{j}$.
Note that the cardinality of $\mathcal D$ is proportional to $|\mathcal K|\times |\mathcal J|$.
\end{lemma}

\begin{proof}[Proof of Lemma \ref{lem1}]
\begin{align*}
& \left\{\sum_{k\in\mathcal K}\beta_{1k}\psi_{k}(\bs x) \right\}
\left\{\sum_{j\in\mathcal J}\beta_{2j}\psi_{j}(\bs x) \right\}\\
= & \left\{\sum_{k\in\mathcal K} \beta_{1k}\cos(\langle \bs w_{k}, \bs x\rangle + b_{k})\right\}
\left\{\sum_{j\in\mathcal J} \beta_{2j}\cos(\langle \bs w_{j}, \bs x\rangle + b_{j})\right\}\\
= &\sum_{k\in\mathcal K} \sum_{j\in\mathcal J} 
\beta_{1k}\beta_{2j}
\cos(\langle \bs w_{k}, \bs x\rangle + b_{k})
\cos(\langle \bs w_{j}, \bs x\rangle + b_{j})\\
= &\sum_{k\in\mathcal K} \sum_{j\in\mathcal J} 
\frac{1}{2}\beta_{1k}\beta_{2j}
\cos(\langle \bs w_{k}+\bs w_{j}, \bs x\rangle + b_{k}+ b_{j}) +
\cos(\langle \bs w_{k}-\bs w_{j}, \bs x\rangle + b_{k}-b_{j})\\
= &\sum_{\tilde {\bs w}\in\mathcal D} \beta_{\tilde {\bs w}}\psi_{\tilde{\bs w}}(\bs x).
\end{align*}
The proof is completed.
\end{proof}

\begin{corollary} \label{cor1}
Denote $\psi_{k}(\bs x) = \sum_{\ell=1}^{2^{D-1}}\cos(\langle \bs w_{k\ell},\bs x\rangle + b_{k\ell}), \psi_{j}(\bs x) = \sum_{\ell'=1}^{2^{D-1}}\cos(\langle \bs w_{j\ell'},\bs x\rangle + b_{j\ell'})$, where
$\ell,\ell'$ index frequencies in the corresponding $\psi_k$ and $\psi_j$, respectively;
$\{\bs w_{k\ell} \in \mathbb{Z}^D\}_{k \in \mathcal{K}, 1\leq\ell\leq 2^{D-1}}$ and $\{\bs w_{j\ell'} \in \mathbb{Z}^D\}_{j \in \mathcal{J}, 1\leq\ell'\leq 2^{D-1}}$ are two collections of frequency vectors;
and $\{b_{k\ell} \in \mathbb{R}\}_{k \in \mathcal{K}}$ and $\{b_{j\ell'} \in \mathbb{R}\}_{j \in \mathcal{J}}$ are two collections of scalar phases, indexed by $\mathcal{K}=\{1, \ldots, K\}, \mathcal{J} = \{1,\ldots, J\}$.
Furthermore, let $\{\beta_{1k} \in \mathbb{R}\}_{k \in \mathcal{K}}$ and $\{\beta_{2j} \in \mathbb{R}\}_{j \in \mathcal{J}}$ be two sets of scalar coefficients and $\bs x \in \mathbb{T}^D$. 
Then,
\begin{align*}
\left\{\sum_{k\in\mathcal K}\beta_{1k}\psi_{k}(\bs x) \right\}
\left\{\sum_{j\in\mathcal J}\beta_{2j}\psi_{j}(\bs x) \right\}
= \sum_{\tilde{\bs w}\in\mathcal D} \beta_{\tilde{\bs w}} \psi_{\tilde{\bs w}}(\bs x)
,
\end{align*}
where 
$\beta_{\tilde{\bs w}}  = \frac{1}{2}\beta_{1k}\beta_{2j}$,
$\psi_{\tilde{\bs w}}(\bs x) =  \cos(\langle \tilde{\bs w},\bs x\rangle + b_{\tilde{\bs w}})$,
$\mathcal D 
\equiv \mathcal D(\{\bs w_{k\ell}\}_{k \in \mathcal K}, \{\bs w_{j\ell'}\}_{j \in \mathcal J})
= 
\{\tilde{\bs w} = \bs w_{k\ell} \pm \bs w_{j\ell'}, k\in\mathcal K, j\in \mathcal J, 1\leq\ell,\ell'\leq 2^{D-1}\}$,
and 
$b_{\tilde{\bs w}} = b_{k\ell} + c_{\tilde{\bs w}} b_{j\ell'}$,
with 
$k, j, \ell, \ell', c_{\tilde {\bs w}}$ induced by the form of $\tilde {\bs w}$. In particular,  $c_{\tilde{\bs w}} = 1$ if $\tilde {\bs w} = \bs w_{k\ell} + \bs w_{j\ell'}$, and $c_{\tilde{\bs w}} = -1$ if $\tilde {\bs w} = \bs 
 w_{k\ell} - \bs w_{j\ell'}$.
Note that the cardinality of $\mathcal D$ is proportional to $2^{D} |\mathcal K||\mathcal J|$.
\end{corollary}

\begin{proof}[Proof of Corollary \ref{cor1}]
\begin{align*}
& \left\{\sum_{k\in\mathcal K}\beta_{1k}\psi_{k}(\bs x) \right\}
\left\{\sum_{j\in\mathcal J}\beta_{2j}\psi_{j}(\bs x) \right\}\\
= & \left\{\sum_{k\in\mathcal K} \beta_{1k}\sum_{\ell}\cos(\langle \bs w_{k\ell}, \bs x\rangle + b_{k\ell})\right\}
\left\{\sum_{j\in\mathcal J} \beta_{2j}\sum_{\ell'}\cos(\langle \bs w_{j\ell'}, \bs x\rangle + b_{j\ell'})\right\}\\
= &\sum_{k\in\mathcal K} \sum_{j\in\mathcal J} 
\beta_{1k}\beta_{2j}
\sum_{\ell,\ell'}\cos(\langle \bs w_{k\ell}, \bs x\rangle + b_{k\ell})
\cos(\langle \bs w_{j\ell'}, \bs x\rangle + b_{j\ell'})\\
= &\sum_{k\in\mathcal K} \sum_{j\in\mathcal J} \sum_{\ell,\ell'}
\frac{1}{2}\beta_{1k}\beta_{2j}
\cos(\langle \bs w_{k\ell}+\bs w_{j\ell'}, \bs x\rangle + b_{k\ell}+ b_{j\ell'}) +
\cos(\langle \bs w_{k\ell}-\bs w_{j\ell'}, \bs x\rangle + b_{k\ell}-b_{j\ell'})\\
= &\sum_{\tilde{\bs w}\in\mathcal D} \beta_{\tilde{\bs w}} \psi_{\tilde{\bs w}}(\bs x)
\end{align*}
The proof is completed.
\end{proof}

%%%%%%%%%%%%%%%% Lemma 2 %%%%%%%%%%%%%%%% 
\begin{lemma}\label{lem2}
Let $\left\{\bs w_{k\ell}\in \mathbb{R}^D\right\}_{k \in \mathcal{K}, 1\leq \ell\leq 2^{D-1}}$ and $\left\{b_{k\ell} \in \mathbb{R}\right\}_{k \in \mathcal{K}}$ be a collection of frequency vectors and scalar phases, respectively, indexed by the set $\mathcal{K} =\{1,\ldots, K\}\subseteq \mathbb{N}$. Furthermore, $\left\{\beta_k \in \mathbb{R}\right\}_{k\in \mathcal{K}}$ be a set of scalar coefficients, and let $m \in \mathbb{N}$ Then,
\begin{align*}
&\left(\sum_{k \in \mathcal{K}} \beta_k \sum_{\ell}\cos \left(\left\langle\bs w_{k\ell}, \bs x\right\rangle+b_{k\ell}\right)
\right)^m
=
\sum_{\tilde{\bs w} \in \mathcal{H}_m}
\tilde\beta_{\tilde{\bs w}}\cos \left(\left\langle \tilde{\bs w}, \bs x\right\rangle+\tilde b_{\tilde{\bs w}}\right),
\end{align*}
where
\begin{align}\label{DEF:Hm}
\mathcal H_m 
& = \left\{\tilde{\bs w}=\sum_{k =1}^K  c_{k} \bs w_{k\ell}\given c_{k} \in \{n_k, n_k-2, \ldots , -(n_k-2)\}, \sum_{k =1}^K  n_k=m, n_k\in \mathbb N, 1\leq \ell\leq 2^{D-1}\right\} \nonumber\\
& \equiv \left\{\tilde{\bs w}=\sum_{k =1}^K  c_{k} \bs w_{k}\given c_{k} \in \{n_k, n_k-2, \ldots , -(n_k-2)\}, \sum_{k =1}^K  n_k=m, n_k\in \mathbb N, \bs w_{k} \in\mathcal W_{k}\right\},
\end{align}
with $n_j\in\mathbb N$ being the polynomial degree contributed by the $j$th frequency. 
In addition,
\begin{align}
\mathcal{H}_m
\subseteq \tilde{\mathcal{H}}_m
% &=\left\{\tilde{\bs w}=\sum_{k \in \mathcal{K}} \sum_{\mu = 1}^m c_{k,\mu} \bs w_{k\ell}, c_{k,\mu} \in \mathbb{Z} \wedge \sum_{k \in \mathcal{K}}\sum_{\mu = 1}^m| c_{k,\mu} | \leq m, 1\leq \ell\leq 2^{D-1}\right\} \nonumber\\
&=\left\{\tilde{\bs w}=\sum_{k\in\mathcal K} \sum_{\mu = 1}^{m} c_{k,\mu} \bs w_{k}, c_{k,\mu} \in \mathbb{Z} \wedge \sum_{k\in\mathcal K}\sum_{\mu = 1}^{m}| c_{k,\mu} | \leq m, \bs w_{k}\in \mathcal W_{k}\right\}.
\label{DEF:Hm_tilde}
\end{align}

Note that we use the notation $\mathcal{H}_m$ and $\tilde{\mathcal{H}}_m$ instead of explicitly writing the dependence on the set $\left(\left\{\bs{w}_{k\ell}\right\}_{k \in \mathcal{K}, 1\leq \ell\leq 2^{D-1}}\right)$ for simplicity.
\end{lemma}

\begin{proof}[Proof of Lemma \ref{lem2}]
The proof of (\ref{DEF:Hm}) is completed by induction. 

\textbf{Base case.}
When $m = 1$, it is only possible for one $k\in\mathcal K$ such that $m = 1$, $c_k = 1$, and $c_{k'} = 0, \forall k'\neq k$.
Exhausting all $|\mathcal K|$ possibilities would give us  all elements in $\mathcal H_1 = \{\bs{w}_k, k\in \mathcal K\}$.
% It is straightforward to prove $\mathcal H_1 \subseteq\tilde{\mathcal H}_1 = \{\tilde{\bs w} =\sum_{k\in\mathcal K} c_k\bs w_{k\ell}, \sum_{k\in\mathcal K}|c_k|\leq 1\}$. Therefore, the equality holds. 

\textbf{Inductive Step.}
Suppose the equality holds for $m$, then we have

\begin{align*}
&\left\{\sum_{k \in \mathcal{K}} \beta_k \psi_{k}(\bs x)
\right\}^{m+1}\\
=&\left\{\sum_{k \in \mathcal{K}} \beta_k \sum_{\ell}\cos \left(\left\langle\bs w_{k\ell}, \bs x\right\rangle+b_{k\ell}\right)
\right\}^{m+1}\\
= & \left\{\sum_{k \in \mathcal{K}} \beta_k \sum_{\ell}\cos \left(\left\langle\bs w_{k\ell}, \bs x\right\rangle+b_{k\ell}\right)
\right\}^m
\left\{\sum_{k \in \mathcal{K}} \beta_k \sum_{\ell}\cos \left(\left\langle\bs w_{k\ell}, \bs x\right\rangle+b_{k\ell}\right)
\right\}\\
= & 
\left\{ \sum_{\tilde{\bs w}_{k\ell} \in \mathcal{H}_m} \tilde\beta_k\cos \left(\left\langle \tilde{\bs w}_{k\ell}, \bs x\right\rangle+\tilde b_{\tilde{\bs w}_{k\ell}}\right)\right\}
\left\{\sum_{k' \in \mathcal{K}} \beta_{k'} \sum_{\ell'}\cos \left(\left\langle\bs w_{k'\ell'}, \bs x\right\rangle+b_{k'\ell'}\right)
\right\}\\
= & 
\sum_{\tilde{\tilde{\bs w}}_{k\ell} \in \mathcal D\left\{\mathcal{H}_m,\left\{\bs{w}_{k'}\right\}_{k' \in \mathcal{K}}\right\} }
\tilde{\tilde\beta}_k 
\cos \left(\left\langle \tilde{\tilde{\bs w}}_{k\ell}, \bs x\right\rangle+\tilde{\tilde b}_{\tilde{\tilde{\bs w}}_{k\ell}}\right),
\end{align*}
where we used $\cos a \cos b = (\cos(a+b)+\cos(a-b))/2$, and
$\tilde\beta_{\cdot}, \tilde{\tilde{\beta}}_{\cdot}, \tilde b_{\cdot}, \tilde{\tilde{b}}_{\cdot}$ are constants.
By Lemma \ref{lem1}, 
$$
\begin{aligned}
&\mathcal D\left\{\mathcal{H}_m,\left\{\bs{w}_{k'}\right\}_{k' \in \mathcal{K}}\right\} \\
& = \left\{\bs w'| {\bs w}'=\sum_{k \in \mathcal{K}}  c_{k} \bs w_{k\ell} \pm \bs{w}_{k'\ell'}, c_{k} \in \{n_k,n_k-2,\ldots, -(n_k-2)\}, \sum_{k\in\mathcal K} n_k = m, n_k\in \mathbb N\right\}\\
& = \left\{\bs w'| {\bs w}'=\sum_{k \in \mathcal{K}}  c_{k} \bs w_{k\ell} \pm \bs{w}_{k'\ell'}, c_{k} \in \{n_k+1,n_k-1,\ldots, -(n_k-1)\}, \sum_{k\in\mathcal K} n_k = m, n_k\in \mathbb N\right\}\\
& = \left\{\bs w'| {\bs w}'=\sum_{k \in \mathcal{K}}  c_{k} \bs w_{k\ell} \pm \bs{w}_{k'\ell'}, c_{k} \in \{n_k,n_k-2,\ldots, -(n_k-2)\}, \sum_{k\in\mathcal K} n_k = m+1, n_k\in \mathbb N\right\}\\
&\equiv \mathcal H_{m+1}
\end{aligned}
$$
Equation (\ref{DEF:Hm_tilde}) holds by definition, since $c_{k,\mu}$ in (\ref{DEF:Hm_tilde}) belongs to a larger set than $c_k$ in (\ref{DEF:Hm}).
\end{proof}

In order to prove Theorem \ref{thm:representation_space_Td}, we first prove an analogous Theorem \ref{thm:representation_space_Td_poly}, where activation function is a polynomial of degree $M$, $\alpha^{(l)}(v) = \sum_{\mu=1}^M \beta_mv^m$, then the space of frequencies of $v_{\bs\theta}(\bs x)$ is $\mathcal H^{(L)}$. 
The form of $\mathcal H^{(L)}$ reveals the power of proposed NF in approximating functions of a rich set of frequencies.

\begin{theorem}
\label{thm:representation_space_Td_poly}
Let $\mathcal M_d = \mathbb S^1$, $v_{\bs{\theta}}: \mathbb{T}^D  = \bigtimes_{d=1}^D\mathbb{S}^1\rightarrow \mathbb{R}$ be an NF of the form \eqref{eqn:nf}, with $\boldsymbol{\eta}:\mathbb{T}^{D}\mapsto\mathbb{R}^{K}$ whose $k^{t h}$ element is defined as \eqref{eqn:eta_torus_encoding}. Denote the set of frequencies $\mathcal W_{k}=\{({i_{k,1}}, \pm{i_{k,2}}, \ldots, \pm{i_{k,D}})\}$, for $k=1,\ldots,K$. 
We consider a polynomial activation function $\alpha^{(l)}(v)=\sum_{m=0}^{M} \beta_m v^m$ for $l>1$. 
Let $\mathbf W^{(l)} \in\mathbb R^{H_l\times H_{l-1}}$ be the matrix of frequencies, and $\mathbf b^{(l)} \in \R^{H_l}$ are the vector of phases. 
Then
$$
v_{\bs{\theta}}(\bs x) = \sum_{\bs{w}^{\prime} \in \mathcal{H}^{(L)}} \beta_{\bs{w}^{\prime}} \cos \left(\left\langle\bs{w}^{\prime}, \bs x\right\rangle+{b}_{\bs{w}^{\prime}}\right),$$
where 
\begin{align*}
\mathcal{H}^{(l)}
% & =\left\{\tilde{\bs w}=\sum_{k =1}^K   (\Pi_{\iota=1}^L c_k^{(\iota)}) \bs w_{k}\given c_k^{(\iota)} \in \mathcal A_{n_k}, \sum_{k =1}^K  n_k \in \{1,\ldots, M\}, n_k\in \mathbb N, \bs w_{k}\in\mathcal W_{k}\right\}\\
&= \left\{\tilde{\bs w}=\sum_{k =1}^K   c_k \bs w_{k}\given c_k \in \mathcal A^{(l)}, \bs w_{k}\in\mathcal W_{k}\right\}
\subseteq \left\{\tilde{\bs w}=\sum_{k =1}^K   c_k \bs w_{k}\given \sum_{k =1}^K  |c_k| \leq M^l, \bs w_{k}\in\mathcal W_{k}\right\},
\end{align*}
where $\mathcal A^{(l)} = \{c_k = \Pi_{\iota=1}^l c_k^{(\iota)} \mid c_k^{(\iota)}\in \mathcal A\}$ is a set of integers that contains possible candidates of coefficient $c_k$, with $\mathcal A = \{c_k\in \{n_k, n_k-2, \ldots , -(n_k-2)\} \mid \sum_{k =1}^K  n_k \in \{1,\ldots, M\}, n_k\in \mathbb N\}$ and $\beta_{\bs{w}^{\prime}}$ are complicated functions of $\boldsymbol{\theta}$.
% In addition,
% \begin{align*}
% \mathcal{H}^{(L)}\subseteq \tilde{\mathcal{H}}_{M^L}:=\left\{\tilde{\bs w}=\sum_{k=1}^K \sum_{\mu = 1}^{M^L} c_{k,\mu} \bs w_{k}, c_{k,\mu} \in \mathbb{Z} \wedge \sum_{k=1}^K\sum_{\mu = 1}^{M^L}| c_{k,\mu} | \leq M^L, \bs w_{k}\in \mathcal W_{k}\right\}.
% \end{align*}
 %where $\mathcal W_{k}=\{({i_{k,1}}, \pm{i_{k,2}}, \ldots, \pm{i_{k,D}})\}$; $1\leq k\leq K=h_0$ is the index for the $k$-th tensor product basis.
\end{theorem}

\begin{proof}[Proof of Theorem \ref{thm:representation_space_Td_poly}]
First we note that $\mathcal H^{l_1} \subseteq H^{l_2}$ for any $l_1\leq l_2$.
We prove the theorem by induction.
To focus on the frequency representation, we let bias terms $\mathbf b^{(l)}$ be zero.
%\gu{[We used $b$ for both bias in MLP and in basis function $\eta$. need to change one of them]}
% Denote
% \begin{align}\label{DEF:Hml}
% \mathcal H_m^{(l)}=\left\{\tilde{\bs w}=\sum_{k =1}^K c_k\bs w_{k}\given c_{k}\in\mathcal B_{m}^{(l)}, \bs w_{k} \in\mathcal W_{k}\right\},
% \end{align}
% where  $\mathcal B^{(l)}_m = \{c_k = \Pi_{\iota=1}^{l}c_k^{(\iota)}\in\mathcal B_m\}$ is a set of integers that contains possible candidates of coefficient $c_k$, with $\mathcal B_m = \{c_k\in \{n_k, n_k-2, \ldots , -(n_k-2)\} \given  \sum_{k =1}^K  n_k =m, n_k\in \mathbb N\}$.
% By definition, $\mathcal A =\cup_{m=1}^M \mathcal B_m$ and $ \mathcal B^{(l)} :=\cup_{m=1}^M \mathcal A_m^{(l)} \subseteq \mathcal A^{(l)}$.

\textbf{Base case.}
When $l = 1$, consider the pre-activation of a node at the first layer for any INR in the form of (\ref{eqn:nf}). 
Recall that $\bf W^{(1)}$ is of dimension $H_1\times H_0$, where $H_0 = K$.
We denote $\bs{v}^{(1)}  = \bf{W}^{(1)}\eta(\bs{x})$ as the pre-activation vector of first layer. 
Specifically, the $j$-th element of $\bs v^{(1)}$ is $v^{(1)}_{j} = \bf W^{(1)}_{j}\eta(\bs x) = \sum_{k=1}^{K} W_{jk}^{(1)}\psi_k(\bs x) = 
\sum_{k=1}^{K} W_{jk}^{(1)} \sum_{\ell = 1}^{2^D-1} \frac{1}{2^D-1}\cos (\langle \bs w_{k\ell} , \bs x\rangle + b_{k\ell})$. 

Applying Lemma \ref{lem2}, the output of this node after activation is
\begin{align*}
z_{j}^{(1)} & = \alpha^{(1)}(v_{j}^{(1)}) = 
\sum_{m = 0}^M \beta_m(v_{j}^{(1)})^m\\
& = \sum_{m= 0}^M  \beta_m \left(\sum_{k=1}^{K} W_{jk}^{(1)} \sum_{\ell = 1}^{2^D-1} \frac{1}{2^D-1}\cos (\langle \bs w_{k\ell} , \bs x\rangle + b_{k\ell}) \right)^m \\
& {=} \sum_{m= 0}^M  \beta_m
\sum_{\tilde{\bs w} 
\in \mathcal{H}_m}
\tilde\beta_{\tilde{\bs w}, j} \cos \left(\left\langle \tilde{\bs w}, \bs x\right\rangle+\tilde b_{\tilde{\bs w}, j}\right)\\
& =
\sum_{\tilde{\bs w} 
\in \mathcal{H}^{(1)}}
\tilde\beta_{\tilde{\bs w}, j} \cos \left(\left\langle \tilde{\bs w}, \bs x\right\rangle+\tilde b_{\tilde{\bs w}, j}\right).
\end{align*}
% where $\mathcal H^{(1)}= \cup_{m= 1}^M\mathcal H_m^{(1)}\equiv\cup_{m= 1}^M\mathcal H_m$. 

\textbf{Induction Step.}
Assume the output of the nodes at layer $l$ satisfy the following expression:
$$
{h}_j^{(l)}=\sum_{\tilde{\bs w} 
\in \mathcal{H}^{(l)}}
\tilde\beta_{\tilde{\bs w}, j} \cos \left(\left\langle \tilde{\bs w}, \bs x\right\rangle+\tilde b_{\tilde{\bs w}, j}\right).
$$
% where $\mathcal H^{(l)} \subseteq \tilde{\mathcal H}_{M^l}$.
Then, the pre-activation of the $j^{t h}$ node at the $(l+1)^{t h}$ layer can be expressed as
$
v_j^{(l+1)}=\sum_{\bs{w} \in \mathcal{H}^{(l)}} \beta_{\bs{w}, j} \cos \left(\left\langle\bs{w}, \bs x\right\rangle+\tilde b_{\bs{w}, j}\right),
$
with $\beta_{\bs w, j}$ being different from $\tilde \beta_{\tilde{\bs w}, j}$.
Applying the activation function to the output of the $j^{t h}$ node at the $(l+1)^{t h}$ layer, we have
$$
\begin{aligned}
{h}_j^{(l+1)}
&=\alpha^{(l+1)}\left({v}_j^{(l+1)}\right)=\sum_{m=0}^M \beta_m\left({v}_j^{(l+1)}\right)^m 
=\sum_{m=0}^M \beta_m\left(\sum_{\bs{w} \in \mathcal{H}^{(l)}} \beta_{\bs{w}, j} \cos \left(\left\langle\bs{w}, \bs x\right\rangle+\tilde b_{\bs{w}, j}\right)\right)^m\\
& =\sum_{m=0}^M \beta_m\sum_{\bs{w}^{\prime} \in \mathcal{H}_m^{(l)}} {\tilde{\beta}}_{\bs{w}^{\prime}, j} 
\cos \left(\left\langle\bs{w}^{\prime}, \bs x\right\rangle+{\tilde{\tilde{b}}}_{\bs{w}^{\prime}, j}\right),
\end{aligned}
$$
where 
$$
\begin{aligned}
\mathcal{H}^{(l)}_m
=&\left\{\bs{w} \given
\bs{w}= \sum_{k=1}^K c_k \bs w_k, \bs w_k\in \mathcal H^{(l)}, c_k\in \{n_k, n_k-2, \ldots , -(n_k-2)\}, \sum_{k =1}^K  n_k =m, n_k\in \mathbb N\right\}\\
= & \left\{\bs{w} \given
\bs{w}= \sum_{k=1}^K c_k \sum_{k=1}^K \{\Pi_{\iota = 1}^l\tilde c_{k}^{(\iota)}\}{\bs w}_{k}, \tilde c_k^{(\iota)}\in\mathcal A,
c_k\in \mathcal B_m, n_k\in \mathbb N, {\bs w}_{k}\in\mathcal W_{k}\right\},
\end{aligned}
$$
% \textcolor{red}{Note here the induction of the last line above needs a double check.
% $$
% \begin{aligned}
% \mathcal{H}^{(l)}_m
% = & \left\{\bs{w} \given
% \bs{w}= \sum_{k_{l+1}=1}^{H^{l}} c_{k_{l+1}} \sum_{k_l=1}^{H^{l-1}}c_{k_l}{\bs w}_{k_l}, \tilde c_{k_l}^{(\iota)}\in\mathcal A,
% c_{k_l} \in \mathcal B_m, n_k\in \mathbb N, {\bs w}_{k_l}\in\mathcal W_{k}\right\},
% \end{aligned}
% $$
% }
where $\mathcal B_m =\{c_k\in \{n_k, n_k-2, \ldots , -(n_k-2)\}, \sum_{k =1}^K  n_k =m\}$.
By construction, we know that $\cup_{m=1}^M \mathcal B_m = \mathcal A$, and $\cup_{m=1}^M \mathcal H_m^{(l)} = \mathcal H^{(l+1)}$, therefore, 
\[
{h}_j^{(l+1)} =\sum_{\bs{w}^{\prime} \in \mathcal{H}^{(l+1)}} \tilde{\tilde{\beta}}_{\bs{w}^{\prime}, j} 
\cos \left(\left\langle\bs{w}^{\prime}, \bs x\right\rangle+\tilde {\tilde{\tilde{b}}}_{\bs{w}^{\prime}, j}\right).
\]
The proof is complete by setting $l = L$.
\end{proof}

\begin{proof}
Proof of Theorem \ref{thm:representation_space_Td}. We prove Theorem \ref{thm:representation_space_Td} by approximating the activation function $\alpha^{(l)} (v) = \sin(v)$ using polynomials and Theorem \ref{thm:representation_space_Td_poly}. 

First we note for any $v$ in the neighborhood of zero, $\sin(v) = \sum_{m=0}^M(-1)^m v^{2m+1} /{(2m+1)!}+ \varepsilon$, where $\varepsilon=(-1)^{M+1} \xi^{2M+3}/{(2M+3)!}$, and $\xi$ is a constant between $0$ and $v$. We have $|\varepsilon|\leq  |\xi|^{2M+3} /{(2M+3)!}\leq|v|^{2M+3}/{(2M+3)!}$.

Due to the activation function $\sin(\cdot)$, we have $|\mathbf h^{(l-1)}|\leq 1, 2\leq l\leq N$.
At layer $l$, since $|\mathbf W^{(l)}| <C_w$ and $|\mathbf b^{(l)}|<C_b$ for some finite constant $0<C_w, C_b<\infty$, we have pre-activation value $\mathbf v^{(l)} = \mathbf W^{(l)} \mathbf h^{(l-1)} + \mathbf b^{(l)}$ bounded, i.e. $|\mathbf v^{(l)}| <C_w + C_b$. The $j^{t h}$ entry of $\mathbf v^{(l)}$ is denoted by $v_j^{(l)}$, and the $j^{t h}$ node after activation takes the form $h^{(l)}_j= \sin(v_j^{(l)}) = \sum_{m=0}^M(-1)^m (v_j^{(l)})^{2m+1} /{(2m+1)!}+ \varepsilon_{j}^{(l)}$, where $\varepsilon_{j}^{(l)}=(-1)^{M+1} ( \xi_j^{(l)})^{2M+3}/{(2M+3)!}$, and $|\xi_j^{(l)}|<(C_w + C_b)$. \revised{Therefore, when $l=2$, $|\varepsilon_{j}^{(2)}|<(C_w + C_b)^{2M+3}/(2M+3)!$.}

At the $(l+1)^{t h}$ layer, pre-activation node is $\mathbf v^{(l+1)} = \mathbf W^{(l+1)} {\mathbf h}^{(l)} + \mathbf b^{(l+1)}$ and its approximation $\widetilde{\mathbf v}^{(l+1)} = \mathbf W^{(l+1)} \widetilde{\mathbf h}^{(l)} + \mathbf b^{(l+1)} = \mathbf W^{(l+1)} (\mathbf h^{(l)} + \bs\varepsilon^{(l)}) + \mathbf b^{(l+1)}$, where $|\bs\varepsilon^{(l)}| < \revised{\nu}$. The node after activation is $\mathbf h^{(l+1)} = \sin({\mathbf v}^{(l+1)})$ and its approximation $\widetilde{h}_j^{(l+1)} = \sum_{m=0}^M (-1)^m \revised{(\widetilde v_j^{(l+1)})^{2m+1}}/(2m+1)!+ (-1)^{M+1} ( \xi_j^{(l+1)})^{2M+3}/{(2M+3)!}$. The approximation error can be decomposed into two parts
\begin{align*}
\varepsilon_j^{(l+1)} = & h_j^{(l+1)}-\widetilde{h}_j^{(l+1)} \\
= & \sin(v_j^{(l+1)})- \widetilde{\sin}(\widetilde v_j^{(l+1)})\\
= & \{\sin(v_j^{(l+1)}) - \sin(\widetilde v_j^{(l+1)})\} + \{\sin(\widetilde v_j^{(l+1)}) - \widetilde{\sin}(\widetilde v_j^{(l+1)})\}\\
= & I + II, 
\end{align*}
where  $\widetilde \sin$ indicates the polynomial approximation to $\sin$. It can be shown that $|I| \leq C_w|\varepsilon_j^{(l)}|$, and $|II| \leq (C_w + C_b)^{2M+3}/(2M+3)!$.

\revised{We can work out bound on the propagated error at the $(L-1)^{t h}$ layer by induction. For the base case $l=3$, we have
$
|\varepsilon_j^{(3)}|\le C_{w}^{3-2}|\varepsilon_j^{(2)}| + (C_w + C_b)^{2M+3}/(2M+3)!\sum_{t=0}^{3-3}C_{w}^{t}.
$
Assuming this pattern holds for $l$:
$
|\varepsilon_j^{(l)}|\le C_{w}^{l-2}|\varepsilon_j^{(2)}| + (C_w + C_b)^{2M+3}/(2M+3)!\sum_{t=0}^{l-3}C_{w}^{t},
$
then for $l+1$, we have 
$$
\begin{aligned}
    |\varepsilon_j^{(l+1)}|&\le
    C_{w}|\varepsilon_j^{(l)}| +(C_w + C_b)^{2M+3}/(2M+3)!\\
    &\le C_{w}(C_{w}^{l-2}|\varepsilon_j^{(2)}| + (C_w + C_b)^{2M+3}/(2M+3)!\sum_{t=0}^{l-3}C_{w}^{t})
    +(C_w + C_b)^{2M+3}/(2M+3)!
\\
&=C_{w}^{l-1}|\varepsilon_j^{(2)}| + (C_w + C_b)^{2M+3}/(2M+3)!\sum_{t=1}^{l-2}C_{w}^{t} + (C_w + C_b)^{2M+3}/(2M+3)! \\
&=C_{w}^{l-1}|\varepsilon_j^{(2)}| + (C_w + C_b)^{2M+3}/(2M+3)!\sum_{t=0}^{l-2}C_{w}^{t}. \\
\end{aligned}
$$
Therefore, we have that for $L\ge 3$
$$
|h_j^{(L-1)}-\widetilde{h}_j^{(L-1)}| \le (C_w + C_b)^{2M+3}/(2M+3)!\left(C_{w}^{L-3}+ \mathbb{I}\{L\ge 4\}\sum_{t=0}^{L-4}C_{w}^{t}\right).
    $$
}
Suppose the maximal tolerance for the approximation error of $\widetilde h_j^{(L-1)}$ for $h_j^{(L-1)}$  is $\revised{\nu}$, then the minimal degree of polynomial activation function for all interior layers is 
\begin{equation}\label{eqn:Morder}
    \revised{M_{\nu,L} = \min\{M: (C_w + C_b)^{2M+3}/(2M+3)!\left(C_{w}^{L-3}+ \mathbb{I}\{L\ge 4\}\sum_{t=0}^{L-4}C_{w}^{t}\right)<\nu\}}
\end{equation}
Note that for any integer $m$, $\lim_{m\to\infty}C^m/m! =0$ holds for any finite constant $C$. Therefore $M_{\revised{\nu},L}$ is finite and exists, such that approximation error of polynomial to sine function can be controlled.
\end{proof}

\revised{We conclude this Section with a couple remarks on Theorem~\ref{thm:representation_space_Td}. The cardinality of the integer set $\mathcal{A}=\{c_{k}^{(l)}\in\{n_k,n_k-2,\cdots,-(n_k-2)\}|\sum_{k=1}^Kn_k\in\{1,\cdots,M_{\nu,L}\},k=1,\cdots,K\}$ is always at least $|\mathcal{A}|\ge M_{\nu,L}$. Hence, the cardinality of $\mathcal{A}^{(L)}=\{c_k=\prod_{l=1}^Lc_k^{(l)}|c_k^{(l)}\in\mathcal{A}\}$, the set of integer scalings of the initial frequency sets, is roughly up to the order $\binom{L+M_{\nu,L}-1}{L}$, the number of distinct size $L$ multi-sets formed from the integer elements of $\mathcal{A}$, though this is only an upper bound due to the occurrence of ``collisions'', i.e., products of distinct integers can be the same value if all are not prime. From~\ref{eqn:Morder}, we see that, for fixed small $\nu>0$,  $M_{\nu,L}$ grows with increasing $L$. So in summary, increasing the depth $L$ increases the number of possible integer scaling of the initial frequency set $\{\mathcal{W}_k\}$ rapidly.}

\section{Geometry}
\subsection{$\mathbb{S}^1$}\label{ssec:torus_geometry}

We can parameterize $\mathbb{S}^1\subset\mathbb{R}^2$ using $\gamma\in[0,2\pi)$ via the map $l(\gamma)\mapsto (\cos(\gamma), \sin(\gamma))$. The tangent vector is then given by $\partial_{\gamma}l(\gamma) = (-\sin(\gamma), \cos(\gamma))$, hence the induced Riemannian metric is given by $G_{\gamma\gamma} = \langle \partial_{\gamma}l(\gamma), \partial_{\gamma}l(\gamma)\rangle =  \sin^{2}(\gamma) + \cos^{2}(\gamma) = 1$. Denoting $v:\mathbb{S}^1\mapsto\mathbb{R}^1$, then by definition the Laplace-Beltrami operator under the induced Riemannian metric $G_{\gamma\gamma}$ simplifies to $\Delta_{\mathbb{S}^{1}}[v]=\frac{1}{\sqrt{\det G_{\gamma\gamma}}}\frac{\partial}{\partial\gamma}(\sqrt{\det G_{\gamma\gamma}}G_{\gamma\gamma}^{-1}\frac{\partial v}{\partial\gamma})=\frac{\partial^2 v}{\partial\gamma^2}$. The operator equation~\eqref{eqn:lbo_eigenfunctions} then takes the form
$$
\frac{\partial^2\phi_k }{\partial\gamma^2} = -\lambda_k\phi_k.
$$
We assume the solution takes the form $\phi_k(\gamma) = \exp(\mu_k\gamma)$, for constant $\mu_k$ to be found, resulting in
$\mu_k^2\exp(\mu_k\gamma) = -\lambda_k\exp(\mu_k\gamma)$,
which results in the characteristic equation $\mu_k^2 = -\lambda_k$. We know that $\lambda_k$ are non-negative, hence we have two cases. 
\par 
1) $\lambda_k > 0$: which gives roots $\mu_k = \pm i\sqrt{\lambda_k}$, implying solutions $\{\exp(i\sqrt{\lambda_k}\gamma), \exp(-i\sqrt{\lambda_k}\gamma)\}$. It is well known that by Euler's formula, we have the real part of these solutions given as $\{\cos(\sqrt{\lambda_k}\gamma), \sin(\sqrt{\lambda_k}\gamma)\}$. Hence, all solutions are of the form $\phi_k(\gamma) =  A\cos(\sqrt{\lambda_k}\gamma) + B\sin(\sqrt{\lambda_k}\gamma)$, for some constants $A$ and $B$. Since $\phi_k$ is a function on $\mathbb{S}^1$, we know that $\phi_k(0)=\phi_k(2\pi)$ and $\frac{\partial \phi_k}{\partial \gamma}(0)=\frac{\partial \phi_k}{\partial \gamma}(2\pi)$. These can be used to form sets of equations for determining $\lambda_k$ via
$$
\begin{aligned}
A &= A\cos(2\pi\sqrt{\lambda_k}) + B\sin(2\pi\sqrt{\lambda_k})\\
B\sqrt{\lambda_k} &= -A\sqrt{\lambda_k}\sin(2\pi\sqrt{\lambda_k}) + B\sqrt{\lambda_k}\cos(2\pi\sqrt{\lambda_k}).
\end{aligned}
$$
Refactoring the system of equations and writing them in matrix form gives:
$$
\begin{pmatrix}
    1 - \cos(2\pi\sqrt{\lambda_k}) & -\sin(2\pi\sqrt{\lambda_k})\\
    \sin(2\pi\sqrt{\lambda_k}) & 1 - \cos(2\pi\sqrt{\lambda_k})
\end{pmatrix} \begin{pmatrix} A \\
B
\end{pmatrix} = \begin{pmatrix} 0 \\
0
\end{pmatrix}
$$
For non-trivial solutions, the matrix must be singular and hence have zero determinant, which implies 
$$
\begin{aligned}
    0 &= (1 - \cos(2\pi\sqrt{\lambda_k}))^2 + \sin^2(2\pi\sqrt{\lambda_k}) \\
    &= 1 - 2\cos(2\pi\sqrt{\lambda_k}) + \cos^2(2\pi\sqrt{\lambda_k}) +\sin^2(2\pi\sqrt{\lambda_k}) \\
    &= 2 - 2\cos(2\pi\sqrt{\lambda_k})
\end{aligned}
$$
Hence, we have that $\sqrt{\lambda_k}$ must satisfy
$$
\cos(2\pi\sqrt{\lambda_k}) = 1,
$$
which implies $\sqrt{\lambda_k} \in \mathbb{Z}$, or equivalently that $\lambda_k = k^2$ for $k\in\mathbb{Z}$. Since we want our eigenfunctions to be orthonormal, i.e.,
$
\|\phi_k\|_{L^2(\mathbb{S}^1)} = 1$,
using the fact that $\int_{\mathbb{S}^{1}}\sin^2(k\gamma)d\gamma = \int_{\mathbb{S}^{1}}\cos^2(k\gamma)d\gamma = \pi,\forall k\in\mathbb R$, this implies $A = B = \frac{1}{\sqrt{\pi}}$. Hence, we have 2 linearly independent eigenfunctions $\{\frac{\cos(k \gamma)}{\sqrt{\pi}}, \frac{\sin(k \gamma)}{\sqrt{\pi}}\}$ corresponding to eigenvalue $k^2$, for $k\in\mathbb{Z}$.
\par 
2) For $\lambda_k=0$, we have the constant eigenfunction. The orthonormality condition enforces this eigenfunction, denoted $\phi_0=\frac{1}{\sqrt{2\pi}}$.

\subsection{$\mathbb{S}^2$}
\label{ssec:diff_geom_s2}

A local parameterization of $\mathbb{S}^2\subset \mathbb{R}^3$ that is smooth and bijective (outside the poles) is 
$$
\begin{aligned}
    l(\gamma_1, \gamma_2) &= (\sin(\gamma_1)\cos(\gamma_2), \sin(\gamma_1)\sin(\gamma_2), \cos(\gamma_1))^{\intercal}, \\
\end{aligned}
$$
which is valid (invertible) for all points on $\mathbb{S}^2$ except $(0,0,\pm 1)$. In this parameterization, $\gamma_1$ is the inclination and $\gamma_2$ is the azimuth angle. By definition, the vectors 
\begin{equation}\label{eqn:tangent_space_basis_vectors_computataion}
    \begin{aligned}
        & t_1 := \frac{\partial l}{\partial\gamma_1}(\gamma_1,\gamma_2) = (\cos(\gamma_1)\cos(\gamma_2), \cos(\gamma_1)\sin(\gamma_2), -\sin(\gamma_1))^{\intercal} \\
        & t_2 := \frac{\partial l}{\partial\gamma_2}(\gamma_1,\gamma_2) = (-\sin(\gamma_1)\sin(\gamma_2), \sin(\gamma_1)\cos(\gamma_2), 0)^{\intercal}
    \end{aligned}
\end{equation}
form an orthogonal basis for the tangent space $T_{x}(\mathbb{S}^2)$. The induced Riemannian metric is given by $G\in\mathbb{R}^{2\times 2}$ with element-wise definition 
$G_{i,j}(\gamma_1,\gamma_2) = \langle t_{i}, t_{j}\rangle$.
Denoting $v:\mathbb{S}^2\mapsto\mathbb{R}$, the Laplace-Beltrami operator under the induced Riemannian metric can be calculated as 
$$
\Delta_{\mathbb{S}^{2}}[v]=\frac{1}{\sqrt{|\det G(\gamma_1,\gamma_2)|}}\sum_{i,j}\frac{\partial}{\partial\gamma_{i}}G^{-1}_{ij}(\gamma_1,\gamma_2)\sqrt{|\det G(\gamma_1,\gamma_2)|}\frac{\partial v}{\partial\gamma_{i}}(\gamma_1,\gamma_2).
$$
The eigenfunctions of $\Delta_{\mathbb{S}^{2}}$ are the spherical harmonics. In extrinsic (Euclidean) coordinates, they take the form of harmonic homogenous polynomials restricted to $\mathbb{S}^2$. %\citep{ArfaouiRezguiBenMabrou2017}
Under their common parameterization using spherical coordinates, they have the analytic form:
$$
    Y_l^m(\gamma_1, \gamma_2) = \sqrt{\frac{(2l+1)(l-m)!}{4\pi (l+m)!}}P_l^m(cos(\gamma_1))e^{im\gamma_2}; \quad \gamma_1\in [0, \pi],\gamma_2\in [0, 2\pi],
$$
where $P_m^l$ are the Legendre polynomials with degree $l = 0, 1, \ldots, $ and order $m = -l, \ldots, 0, \ldots, l$. A real-valued set of spherical harmonics that is complete for $L^2(\mathbb{S}^2)$ can be constructed according to 
$$
    \phi_j = 
      \begin{cases}
      \sqrt{2}\text{Re}(Y_k^m) & -k\le m < 0\\
      Y_k^0 & m=0\\
      \sqrt{2}\text{Img}(Y_k^m) & 0 < m \le k
    \end{cases}  
$$
for $k = 0,1,\ldots,l$, $m=-k,\ldots,0,\ldots,k$ and $j = k^2+k+m+1$. For any finite maximum degree $l$, the total number of basis functions is $m = (l+1)^2$. 

\section{Supporting Methodological Details}

%\subsection{Architecture}
%Figure~\ref{fig:mlp} provides an illustration of the proposed architecture for deep product manifold density modeling.
%\begin{figure}
%    \centering
    % \includegraphics[width=0.85\linewidth]{Figures/MLP.png}
    %\includegraphics[width=0.85\linewidth]{Figures/MLP1.png}
    %\caption{Illustration of proposed MLP architecture. 
    % \wc{Change node $z$ to $h$, rank to $H_{l}$.}
    %}
    %\label{fig:mlp}
%\end{figure} 

\subsection{Approximation Properties for Shallow Network}\label{ssec:approx_prop_shallowNeuroPMD}
\revised{In this Section, we provide the approximation properties of an appropriately rescaled shallow ($L=1$) NeuroPMD in the Hilbert-Sobolev space of order $s$ on $\Omega$.  Theorem~\ref{thm:shallowNeuroPMDMinMax} provides a lower bound on $K$ needed for achieving a target approximation error, and is an adaptation of Proposition 1 in \citep{bach2017equivalence} to our setting.} 
\begin{theorem}\label{thm:shallowNeuroPMDMinMax}
\revised{Let $\mathcal{H}^{s}(\Omega)$ be a Sobolev space on $\Omega$, with $2s > \sum_{d=1}^Dp_{d}$ and take $\epsilon>0$. Let $\boldsymbol{i}_{1},...,\boldsymbol{i}_{K}\sim q(\boldsymbol{i})\propto \frac{(1+\lambda_{\boldsymbol{i}})^{-s}}{(1+\lambda_{\boldsymbol{i}})^{-s}+\epsilon}$. Then for any $\delta \in (0,1)$, if $K \ge 5r_{max}(\epsilon)\log\frac{16r_{max}(\epsilon)}{\delta}$, with probability greater than $1-\delta$, we have that \begin{equation}\label{eqn:shallowNeuroPMDrate}
        \sup_{\|f\|_{\mathcal{H}^{s}(\Omega)}\le 1}\inf_{\|\boldsymbol{\theta}\|_{2}^2\le \frac{4}{K}}\|f -  v_{\boldsymbol{\theta}}(\boldsymbol{x})\|_{L^2(\Omega)}^2 \le 4\epsilon.
    \end{equation}
where $f\in\mathcal{H}^{s}(\Omega)$, $r_{\max}(\epsilon) \asymp {\epsilon}^{-\frac{\sum_{d=1}^Dp_d}{2s}}$, and $v_{\boldsymbol{\theta}}$ is an $L=1$ layer shallow NF: $v_{\boldsymbol{\theta}}(\boldsymbol{x}) = \boldsymbol{\theta}^{\intercal}\boldsymbol{\eta}(\boldsymbol{x})$, with (scaled) encoder $\boldsymbol{\eta}(\boldsymbol{x}) = (\sqrt{\frac{(1+\lambda_{\boldsymbol{i}_1})^{-s}}{q(\boldsymbol{i}_{1})}}\psi_{\boldsymbol{i}_1}(\boldsymbol{x}), ..., \sqrt{\frac{(1+\lambda_{\boldsymbol{i}_K})^{-s}}{q(\boldsymbol{i}_{K})}}\psi_{\boldsymbol{i}_K}(\boldsymbol{x}))$.
}
\end{theorem}
\revised{\begin{proof}
We first define a positive definite kernel function $k:\Omega\times\Omega\mapsto \mathbb{R}$ as follows
$$
\begin{aligned}
\kappa(\boldsymbol{x},\boldsymbol{y}) &=\sum_{\boldsymbol{i}\in\mathbb{Z}^{D}}S(\lambda_{\boldsymbol{i}})\psi_{\boldsymbol{i}}(\boldsymbol{x})\psi_{\boldsymbol{i}}(\boldsymbol{y}) = \sum_{\boldsymbol{i}\in\mathbb{Z}^{D}}\sqrt{S(\lambda_{\boldsymbol{i}})}\psi_{\boldsymbol{i}}(\boldsymbol{x})\sqrt{S(\lambda_{\boldsymbol{i}})}\psi_{\boldsymbol{i}}(\boldsymbol{y}) \\
&:=\sum_{\boldsymbol{i}\in\mathbb{Z}^{D}}\varphi(\boldsymbol{i},\boldsymbol{x})\varphi(\boldsymbol{i},\boldsymbol{y}) = \int_{\mathbb{Z}^{D}}\varphi(\boldsymbol{i},\boldsymbol{x})\varphi(\boldsymbol{i},\boldsymbol{y})d\tau(\boldsymbol{i}),
\end{aligned}
$$
where 
%\textcolor{red}{$\Omega\in \Pi_{d=1}^D \mathcal M_d$ is a product manifold embedded in the ambient Euclidean space of dimension $\sum_{d=1}^D p_d$;}
$\psi_{\boldsymbol{i}}=\prod_{d=1}^D\phi_{i_{d}}$ are the tensor product eigenfunctions with eigenvalue $\lambda_{\boldsymbol{i}}=\sum_{d=1}^D\lambda_{i_{d}}$ of $-\Delta_{\Omega}$;  $S(\lambda_{\bs i})$ is a weight function such that  $\sum_{\boldsymbol{i}\in\mathbb Z^D}S(\lambda_{\boldsymbol{i}})<\infty$, 
%and the function $x\mapsto\kappa(\bs x, \bs x)$ is integrable with respect to measure $d\omega$;
$\varphi(\boldsymbol{i},\boldsymbol{x})=\sqrt{S(\lambda_{\boldsymbol{i}})}\psi_{\boldsymbol{i}}(\boldsymbol{x})$, and $\tau(\boldsymbol{i})$ is the counting measure.
Since the LBO eigenfunctions $\psi_{\bs i}$ form a complete orthonormal basis, thus for any $f\in L^2(\Omega)$, we have $f(\boldsymbol{x}) = \sum_{\boldsymbol{i}}\langle f, \psi_{\boldsymbol{i}}\rangle_{L^2(\Omega)}\psi_{\boldsymbol{i}}(\boldsymbol{x})$. 
}

\revised{
    Let $\mathcal{F}_{\kappa}\subset L^2(\Omega)$ be the reproducing kernel Hilbert space corresponding to $\kappa$. That is,
    $$
    \mathcal{F}_{\kappa} := \{f:f(\boldsymbol{x}) = \sum_{\boldsymbol{i}\in\mathbb{Z}^{D}}c_{\boldsymbol{i}}\kappa(\boldsymbol{x},\boldsymbol{y}_{\bs i}),\boldsymbol{y}_{\bs i}\in\Omega, c_{\boldsymbol{i}}\in\mathbb{R}\}.
    $$  
    Importantly, for smoothness level $s \ge 0$, we cam define 
    $$
    S(\lambda_{\boldsymbol{i}}) = \left(1+\lambda_{\boldsymbol{i}}\right)^{-s},
    $$
    hence the kernel is given by 
    $\kappa(\boldsymbol{x},\boldsymbol{y}) = \sum_{\boldsymbol{i}}\left(1+\lambda_{\boldsymbol{i}}\right)^{-s}\psi_{\boldsymbol{i}}(\boldsymbol{x})\psi_{\boldsymbol{i}}(\boldsymbol{y})
    $
    and then $\mathcal{F}_{\kappa}:=\mathcal{H}^{s}(\Omega)$, the Sobolev space on product manifold $\Omega$ (see proposition 2. of \citeSupp{de2021reproducing}).
    }

\revised{Next define the integral operator $\Sigma: L_2(\Omega) \mapsto L_2(\Omega)$ as $[\Sigma]f(\boldsymbol{x}) = \int_{\Omega}\kappa(\boldsymbol{x},\boldsymbol{y})f(\boldsymbol{y})d\omega(\boldsymbol{y)}$. 
	It follows that
 $$
    \begin{aligned}
        [\Sigma]f(\boldsymbol{x}) &= \int_{\Omega}\kappa(\boldsymbol{x},\boldsymbol{y})f(\boldsymbol{y})d\omega(\boldsymbol{y)} \\
        &= \int_{\Omega}\sum_{\boldsymbol{i}\in \mathbb{Z}^D}S(\lambda_{\boldsymbol{i}})\psi_{\boldsymbol{i}}(\boldsymbol{x})\psi_{\boldsymbol{i}}(\boldsymbol{y})\sum_{\boldsymbol{i}^{\prime}\in \mathbb{Z}^D}\langle f, \psi_{\boldsymbol{i}^{\prime}}\rangle_{L^2(\Omega)}\psi_{\boldsymbol{i}^{\prime}}(\boldsymbol{y})d\omega(\boldsymbol{y)} \\
        &= \sum_{\boldsymbol{i}\in \mathbb{Z}^D}S(\lambda_{\boldsymbol{i}})\psi_{\boldsymbol{i}}(\boldsymbol{x})\sum_{\boldsymbol{i}^{\prime}\in \mathbb{Z}^D}\langle f, \psi_{\boldsymbol{i}^{\prime}}\rangle_{L^2(\Omega)}\underbrace{\int_{\Omega}\psi_{\boldsymbol{i}}(\boldsymbol{y})\psi_{\boldsymbol{i}^{\prime}}(\boldsymbol{y})d\omega(\boldsymbol{y})}_{\delta_{\boldsymbol{i},\boldsymbol{i}^{\prime}}} \\
        &= \sum_{\boldsymbol{i}\in \mathbb{Z}^D}S(\lambda_{\boldsymbol{i}})\langle f, \psi_{\boldsymbol{i}}\rangle_{L^2(\Omega)}\psi_{\boldsymbol{i}}(\boldsymbol{x}),
    \end{aligned}
    $$
    where $\delta_{\bs i, \bs i'} = I(\bs i = \bs i')$ is an indicator function.
    }
    \revised{
    Let $\epsilon >0$ be any positive constant, then 
    $
    [\Sigma + \epsilon I]f (\bs x)= \sum_{\boldsymbol{i}}\left(S(\lambda_{\boldsymbol{i}}) + \epsilon\right)\langle f, \psi_{\boldsymbol{i}}\rangle_{L^2(\Omega)}\psi_{\boldsymbol{i}}(\boldsymbol{x}).
    $
    Due to the symmetry and compactness $[\Sigma + \epsilon I]$, 
    the inverse of the operator is of the form
    $$
    [\Sigma + \epsilon I]^{-1}f(\bs x) = \sum_{\boldsymbol{i}\in\mathbb Z^D}\left(S(\lambda_{\boldsymbol{i}}) + \epsilon\right)^{-1}\langle f, \psi_{\boldsymbol{i}}\rangle_{L^2(\Omega)}\psi_{\boldsymbol{i}}(\boldsymbol{x}).
    $$
    This can be verified by noting that $[\Sigma + \epsilon I]\psi_{\bs i}(\bs x) =\sum_{\bs i^\prime\in\mathbb Z^D} (S(\lambda_{\bs i^\prime}) + \epsilon) \delta_{\bs i, \bs i'}\psi_{\bs i'}(\bs x) = (S(\lambda_{\bs i}) + \epsilon)\psi_{\bs i}(\bs x)$, and therefore
     $[\Sigma + \epsilon I]  \{\sum_{\boldsymbol{i}\in\mathbb Z^D}\left(S(\lambda_{\boldsymbol{i}}) + \epsilon\right)^{-1}\langle f, \psi_{\boldsymbol{i}}\rangle_{L^2(\Omega)}\psi_{\boldsymbol{i}}(\boldsymbol{x})\}= \sum_{\bs i\in\mathbb Z^D} \langle f, \psi_{\bs i}\rangle_{L^2(\Omega)} \psi_{\bs i}(\bs x) = f(\bs x)$.
    }

\revised{
Now we are ready to apply Proposition 1 in  \citeSupp{bach2017equivalence}. Let $q: \mathbb{Z}^{D}\to \mathbb R$ be the probability density function w.r.t. the counting measure $\tau$. Let $\bs i$ be sampled from the corresponding distribution, i.e.  $\boldsymbol{i}_{1},...,\boldsymbol{i}_{K}\overset{iid}{\sim} q(\boldsymbol{i})$. Consider
$$
\begin{aligned}
r_{\max}(q,\epsilon) &= \sup_{\boldsymbol{i}\in\mathbb{Z}^{D}}\frac{1}{q(\boldsymbol{i)}}\int_{\Omega}\varphi({\boldsymbol{i}},\boldsymbol{x})\left[\Sigma + \epsilon I\right]^{-1}
\varphi(\boldsymbol{i},\boldsymbol{x})
d\omega(\boldsymbol{x}) \\ 
&= \sup_{\boldsymbol{i}\in\mathbb{Z}^{D}}\frac{1}{q(\boldsymbol{i)}}\int_{\Omega}\varphi(\boldsymbol{i},\boldsymbol{x})\left(\sum_{\boldsymbol{i}^{\prime}\in \mathbb Z^D}\left(S(\lambda_{\boldsymbol{i}^{\prime}}) + \epsilon\right)^{-1}\langle \varphi(\boldsymbol{i},\cdot), \psi_{\boldsymbol{i}^{\prime}}\rangle_{L^2(\Omega)}\psi_{\boldsymbol{i}^{\prime}}(\boldsymbol{x})\right)d\omega(\boldsymbol{x}) \\
&= \sup_{\boldsymbol{i}\in\mathbb{Z}^{D}}\frac{1}{q(\boldsymbol{i)}}\int_{\Omega}\sqrt{S(\lambda_{\boldsymbol{i}})}\psi_{\boldsymbol{i}}(\boldsymbol{x})\left(\sum_{\boldsymbol{i}^{\prime}\in \mathbb Z^D}\left(S(\lambda_{\boldsymbol{i}^{\prime}}) + \epsilon\right)^{-1}\sqrt{S(\lambda_{\boldsymbol{i}})}\langle \psi_{\boldsymbol{i}}, \psi_{\boldsymbol{i}^{\prime}}\rangle_{L^2(\Omega)}\psi_{\boldsymbol{i}^{\prime}}(\boldsymbol{x})\right)d\omega(\boldsymbol{x}) \\
&= \sup_{\boldsymbol{i}\in\mathbb{Z}^{D}}\frac{1}{q(\boldsymbol{i)}}\frac{S(\lambda_{\boldsymbol{i}})}{S(\lambda_{\boldsymbol{i}})+\epsilon}\int_{\Omega}\psi_{\boldsymbol{i}}^2(\boldsymbol{x})d\omega(\boldsymbol{x}) \\
&=\sup_{\boldsymbol{i}\in\mathbb{Z}^{D}}\frac{1}{q(\boldsymbol{i)}}\frac{S(\lambda_{\boldsymbol{i}})}{S(\lambda_{\boldsymbol{i}})+\epsilon}.
\end{aligned}
$$
}

\revised{
Then by Proposition 1 of \cite{bach2017equivalence}, for any $\delta \in (0,1)$, if 
$$
K \ge 5r_{\max}(q,\epsilon)\log\frac{16r_{\max}(q,\epsilon)}{\delta},
$$
with probability greater than $1-\delta$, we have that 
$$
\sup_{\|f\|_{\mathcal{F}_{\kappa}}\le 1}\inf_{\|\boldsymbol{\theta}\|_{2}^2\le \frac{4}{K}}\|f -  v_{\boldsymbol{\theta}}(\boldsymbol{x})\|_{L^2(\Omega)}^2 \le 4\epsilon.
$$
}

\revised{Furthermore, we could define an optimized distribution for $q$ following Equation (10) of \citeSupp{bach2017equivalence},
    $$
    \begin{aligned}
            q_{\epsilon}^{*}(\boldsymbol{i})&\propto \frac{\int_{\Omega}\varphi({\boldsymbol{i}},\boldsymbol{x})\left[\Sigma + \epsilon I\right]^{-1}[\varphi(\boldsymbol{i},\boldsymbol{x})]d\omega(\boldsymbol{x})}{\sum_{\boldsymbol{i}\in\mathbb Z^D}\int_{\Omega}\varphi({\boldsymbol{i}},\boldsymbol{x})\left[\Sigma + \epsilon I\right]^{-1}[\varphi(\boldsymbol{i},\boldsymbol{x})]d\omega(\boldsymbol{x}) } \\
            &\propto \frac{S(\lambda_{\boldsymbol{i}})}{S(\lambda_{\boldsymbol{i}})+\epsilon}\Big/\sum_{\boldsymbol{i}^{\prime}\in\mathbb Z^D}\frac{S(\lambda_{\boldsymbol{i}^{\prime}})}{S(\lambda_{\boldsymbol{i}^{\prime}})+\epsilon},
    \end{aligned}
    $$
    which is a proper distribution since $\sum_{\boldsymbol{i}^{\prime}}\frac{S(\lambda_{\boldsymbol{i}^{\prime}})}{S(\lambda_{\boldsymbol{i}^{\prime}})+\epsilon}<\infty$ by assumption.}
    \revised{Plugging $q^{*}$ into the expression for $r_{\max}$ above, for an $s$-Sobolev space with $S(\lambda_{\boldsymbol{i}})= (1+\lambda_{\boldsymbol{i}})^{-s}$, we have:
    $$
    \begin{aligned}
            r_{\max}(q^{*},\epsilon) &=  \sum_{\boldsymbol{i}^{\prime}\in\mathbb{Z}^{D}}\frac{S(\lambda_{\boldsymbol{i}^{\prime}})}{S(\lambda_{\boldsymbol{i}^{\prime}})+\epsilon} =\sum_{\boldsymbol{i}\in\mathbb{Z}^{D}}\frac{(1+\lambda_{\boldsymbol{i}})^{-s}}{(1+\lambda_{\boldsymbol{i}})^{-s}+\epsilon}.
    \end{aligned}
    $$}
\revised{With slight abuse of notation, let $j:=j(\boldsymbol{i})\in\mathbb{N}$ be a flattened indexing of multi-index $\boldsymbol{i}$ that preserves order, i.e. 
$\lambda_1\le \lambda_2\le \ldots$. Recall that each marginal domain $\mathcal{M}_{d}$ is a closed Riemannian manifold of dimension $p_{d}\ge 1$, and thus $\Omega$ is of dimension $p^* := \sum_{d=1}^D p_d$. By Weyl’s law (Theorem 72 \citeSupp{Canzani2013}), we have that }
\revised{  
$$
\lambda_{j} \asymp \frac{\sqrt{2\pi}}{\{\omega^* \text{Vol}(\Omega)\}^{2/p^*}}j^{2/p^*}, \quad j\rightarrow\infty,
$$
where $\omega^* = \frac{2\pi^{p^* }}{p^* \Gamma(p^*/2)}$,
therefore, we have that 
$$
S(\lambda_j) = (1+\lambda_j)^{-s}\asymp j^{-2s/p^*}, \quad j \to \infty.
$$ 
Now, let $j^{*}(\epsilon) = \max\{j\ge 1: S(\lambda_j)\ge \epsilon\}$. By asymptotic formula for S($\lambda_j)$, we have
$$
\begin{aligned}
    &j^{-2s/p^*}\ge \epsilon 
    \iff 
    \epsilon^{-p^*/2s} \ge j \implies j^{*}(\epsilon)\asymp\epsilon^{-p^*/2s}
\end{aligned}
$$
for sufficiently small $\epsilon>0$.
}
\revised{By \citeSupp{bach2017equivalence} Section 4.2,
$
r_{\max}(q^{*},\epsilon) \asymp j^{*}(\epsilon).
$ 
Therefore, we have 
$
r_{\max}(q^{*},\epsilon) \asymp j^{*}(\epsilon)
= {\epsilon}^{-\frac{\sum_{d=1}^Dp_d}{2s}}.
$
The proof is completed following Proposition 1 of \cite{bach2017equivalence}.
\end{proof}
}
\revised{In practice, we sample the encoding basis uniformly
over multi-index set $\underset{\lambda \le \lambda_{\text{max}}}{\bigcup}\mathcal{I}_{\lambda}$, but since this uniform sampling scheme is not well-defined asymptotically, we use $q(\boldsymbol{i})$ to recover standard results. 
%Notice that since $r_{\max}(\epsilon)\asymp{\epsilon}^{-\frac{\sum_{d=1}^Dp_d}{2s}}$, $K$ must grow exponentially in $D$ to achieve the desired accuracy, a well know manifestation of the curse of dimensionality. 
A consequence of Theorem~\ref{thm:shallowNeuroPMDMinMax} is that a shallow NeuroPMD is a universal approximator for a sufficiently large $K$. This is formalized below as a corollary.} 
\revised{
\begin{corollary}\label{thm:shallowNeuroPMDUAP}
For any $f\in\mathcal{H}^{s}(\Omega)$ and $\epsilon>0$, there exists $K\in\mathbb{N}$, multi-indices $\boldsymbol{i}_{1},...,\boldsymbol{i}_{K}\in\mathbb{Z}^{D}$ (defining encoder $\boldsymbol{\eta}$), and parameters $\boldsymbol{\theta}\in\Theta$ such that the $L=1$ layer $v_{\boldsymbol{\theta}}$ defined in \eqref{eqn:nf} satisfies 
\begin{equation}\label{eqn:shallowNeuroPMDrate}
\|f -  v_{\boldsymbol{\theta}}(\boldsymbol{x})\|_{L^2(\Omega)}^2\le \epsilon
\end{equation}
\end{corollary}
}
\revised{\begin{proof} This is a direct consequence of Theorem~\ref{thm:shallowNeuroPMDMinMax}. To see this, denote $f_0 =f\Big/\|f\|_{\mathcal{H}^{s}(\Omega)}$, let
$\epsilon^{*} = \epsilon\Big/4\|f\|_{\mathcal{H}^{s}(\Omega)}^2$
and $\delta\in (0,1)$. We assume $K\ge 5r_{max}(\epsilon^{*} )\log\frac{16r_{max}(\epsilon^{*})}{\delta}$, 
where $r_{max}(\epsilon^{*})\asymp{\epsilon^{*}}^{-\frac{\sum_{d=1}^Dp_d}{2s}}$. Define 
$\widehat{\boldsymbol{\theta}} \in \text{arginf}_{\|\boldsymbol{\theta}\|_{2}^2\le \frac{4}{K}}\|f_0 - v_{\boldsymbol{\theta}}\|_{L^2(\Omega)}^2. 
$
By Theorem~\ref{thm:shallowNeuroPMDMinMax}, since $\|f_0\|_{\mathcal{H}^{s}(\Omega)}\le 1$, with probability greater than $1-\delta$,
$\|f_0 - v_{\widehat{\boldsymbol{\theta}}}\|_{L^2(\Omega)}^2\le 4\epsilon^{*}$.
Denote the re-scaled parameters $\tilde{\boldsymbol{\theta}}=\|f\|_{\mathcal{H}^{s}(\Omega)}\boldsymbol{\theta}$, and notice that since $L=1$, $v_{\tilde{\boldsymbol{\theta}}}:= \|f\|_{\mathcal{H}^{s}(\Omega)}v_{\boldsymbol{\theta}}$, due to linearity, and denote $\widehat{\tilde{\boldsymbol{\theta}}} = \|f\|_{\mathcal{H}^{s}(\Omega)}\widehat{\boldsymbol{\theta}}$. Now, notice that with probability greater than $1-\delta$, we have 
\begin{equation}\label{eqn:prob_bounds}
\begin{aligned}
    \|f-v_{\widehat{\tilde{\boldsymbol{\theta}}}}\|_{L^2(\Omega)}^2 &= \|\|f\|_{\mathcal{H}^{s}(\Omega)}f_0-\|f\|_{\mathcal{H}^{s}(\Omega)}v_{\widehat{\boldsymbol{\theta}}}\|_{L^2(\Omega)}^2 = \|f\|_{\mathcal{H}^{s}(\Omega)}^2\|f_0-v_{\widehat{\boldsymbol{\theta}}}\|_{L^2(\Omega)}^2 \\
    & \le \|f\|_{\mathcal{H}^{s}(\Omega)}^2 4\epsilon^{*} = \epsilon.
\end{aligned}
\end{equation}
Since $1-\delta \in (0,1)$, the bound in \eqref{eqn:prob_bounds} always has positive probability, hence there exists some set of multi-indices $(\boldsymbol{i}_{1}^{*},...,\boldsymbol{i}_{K}^{*})$ defining an encoder $\boldsymbol{\eta}$ such that the correspond  $v_{\widehat{\tilde{\boldsymbol{\theta}}}}$ achieves the desired accuracy.
\end{proof}}

\subsection{Algorithm Implementation Details}\label{ssec:algo_implementation_detials}
%Algorithm~\ref{alg:sgd} provides pseudocode for our stochastic gradient ascent based estimation procedure.
We provide several remarks below to address key implementation details and practical considerations for the Algorithm~\ref{alg:sgd}’s successful deployment.
%\par 
%\noindent{\textbf{Convergence}: In general, proving the convergence of the iterates in Algorithm~\ref{alg:sgd} to a stationary point of \eqref{eqn:MAP_pp_reparam} requires some additional assumptions on $\mathcal{L}(o,\boldsymbol{\theta})$ that are difficult to verify, e.g. globally Lipschitz gradients \citep{ghadimi2013}. Empirically, we found convergence to be robust for the simple learning rate sequences discussed below, provided that $\tau$ was large enough to avoid the case of unbounded or nearly unbounded likelihoods. 
%}
\par 
\noindent{\textbf{Learning Rate Schedule}: In our experiments, for relatively low dimensional cases ($p_{d}=1$, $D\le 3$), a small fixed learning rate was typically sufficient for fast and stable convergence. For the higher dimensional cases $p_{d}\ge 2$ and/or $D > 3$, we found a fixed learning rate often leads to relatively slow algorithmic convergence. To accelerate convergence in these cases, we employ a cyclic learning rate schedule under the triangular policy from \citeSupp{Smith2017}, which adaptively cycles the learning rate between upper and lower bounds on a fixed schedule and has been observed empirically to improve training speeds dramatically \citepSupp{smith2019}.}
\par 
\noindent{\textbf{Monte-Carlo Integration Sampling}: Line~\ref{alg:manifold_sampling} in Algorithm~\ref{alg:sgd} requires sampling the uniform distribution over each $\mathcal{M}_{d}$. For many manifolds of interest, e.g. $\mathbb{S}^{p}$, fast exact procedures are available. If $\mathcal{M}_{d}$ is represented via triangulation, fast approximate methods exist that only require uniform sampling within the triangles and then weighting by their volumes \citepSupp{Osada2002}. 
Higher dimensional and non-standard marginal manifolds may require more computationally intensive methods for sampling, e.g. MCMC \citepSupp{Zappa2018}. However, since the gradients in Equation \eqref{eqn:unbiased_gradient_estimates} are unbiased for any $q_1$ and $q_2$, smaller values can potentially be used in such cases to reduce computational demands, albeit at the expense of increased variance.} \par 
\noindent{\revised{\textbf{Roughness Penalty Computation:}} While automatic differentiation can be used in the calculation of the roughness penalty gradient as outlined in Section~\ref{ssec:algo}, this approach requires including the spatial locations of the evaluation points in the computational graph, which may become inefficient for large (deep) networks and large $q_2$. In such cases, it may be preferable to approximate the operators using standard Euclidean finite difference schemes, as this only requires forward passes through the network. \revised{As an additional note, extrinsic formulations of the full manifold Hessian \citep{absil2013extrinsic} may provide a route towards incorporating more general roughness penalties into our framework, (e.g., involving mixed partial derivatives), which we leave as an interesting direction for future work}.}
\par 
\noindent{\revised{\textbf{Non-Analytic Spectrum}: In some applications, the marginal manifold $\mathcal{M}_{d}$ is represented by a triangulation. This is common in imaging application, where typically $p_d=2/3$ and $m_d=p_d+1$. For example, the raw structural connectome data analyzed in Section~\ref{ssec:rda_cc} lives on a product of two cortical $2$-d surfaces embedded in $\mathbb{R}^3$ and discretized using a dense triangular mesh. In such cases, Laplace–Beltrami eigenfunctions are not available in closed form. However, algorithms exist for approximating these eigenfunctions on the mesh via finite element basis functions \citep{reuter2009}, which can then be sampled from to form the encoding function $\boldsymbol{\eta}$. In order to calculate the roughness penalty in Proposition~\ref{prop:LBO_via_ambient_projection}, we need to form the projection matrices $\boldsymbol{P}_{d}$ onto the tangent space $T_{x_{d}}(\mathcal{M}_d)$. For the setting $m_{d}=p_{d}+1$, the projection matrix can be defined by $\boldsymbol{P}_{d} = (\boldsymbol{I}-\boldsymbol{n}_{x_{d}}\boldsymbol{n}_{x_{d}}^{\intercal})$, where $\boldsymbol{n}_{x_{d}}\in\mathbb{R}^{m_{d}}$ is the normal vector to $\mathcal{M}_{d}$ at $x_{d}$ \citep{fuselier2013high}. On the mesh, the normal $\boldsymbol{n}_{x_{d}}$ can be easily approximated by identifying normal of the triangle containing $x_{d}$.}}

\subsection{Non-Separable First-Layer Encoding for $\mathbb{T}^{D}$}\label{ssec:sep_vs_nonsep_encoding}

Let $\mathbb{T}^D = \bigtimes_{d=1}^D\mathbb{S}^1$.  For a fixed eigenvalue $\lambda$ of $\Delta_{\mathbb{T}^{D}}$, taking the tensor products of the marginal eigenfunctions, we have that the eigenfunctions are given by:
\begin{equation}\label{eqn:sep_eigenfuncs_TD}
\boldsymbol{\psi}_{\lambda} := \left(\pi^{-D/2}\prod_{d=1}^D\cos(i_{d}x_d - \frac{z_{d}\pi}{2}), (z_{1},\ldots,z_{D})\in\{0,1\}^{D}, \boldsymbol{i}\in \mathcal{I}_{\lambda}\right),
\end{equation}
with index set $\mathcal{I}_{\lambda}$ defined as in Section~\ref{sssec:RLBO_encoding}. Define the function
\begin{equation}\label{eqn:non_sep_eigenfuncs_TD}
        \tilde{\boldsymbol{\psi}}_{\lambda}:=\left(\pi^{-D/2}h(\sum_{d=1}^Di_{d}x_d):\boldsymbol{i}\in\mathcal{I}, h\in\{\sin,\cos\}\right).
\end{equation}
We want to show that there exists a rotation matrix $\boldsymbol{A}_{\lambda}\neq \boldsymbol{I}$ such that $\tilde{\boldsymbol{\psi}}_{\lambda} = \boldsymbol{A}_{\lambda}\boldsymbol{\psi}_{\lambda}$. That is, $\tilde{\boldsymbol{\psi}}_{\lambda}$ is a non-separable basis for eigenspace $\lambda$ with $\text{span}(\tilde{\boldsymbol{\psi}}_{\lambda}) =\text{span}(\boldsymbol{\psi}_{\lambda})$. This can be established by proving the following four facts: 1) All functions in \eqref{eqn:sep_eigenfuncs_TD} and \eqref{eqn:non_sep_eigenfuncs_TD} are eigenfunctions of $\Delta_{\mathbb{T}^{D}}$ for fixed eigenvalue $\lambda$. 2) The rank of basis \eqref{eqn:sep_eigenfuncs_TD} and \eqref{eqn:non_sep_eigenfuncs_TD} are equal. 3) The functions in \eqref{eqn:non_sep_eigenfuncs_TD} are pairwise orthonormal. 4) Any element in \eqref{eqn:non_sep_eigenfuncs_TD} can be written as a linear combination of elements in \eqref{eqn:sep_eigenfuncs_TD}. The proofs for 1-4 proceed as follows:
\begin{enumerate}
    \item This follows directly from the facts $\frac{\partial^2}{\partial x_{d}^2}\sin(\sum_{d=1}^Di_{d}x_{d}) = - i_{d}^2\sin(\sum_{d=1}^Di_{d}x_{d})$ and $\frac{\partial^2}{\partial x_{d}^2}\cos(\sum_{d=1}^Di_{d}x_{d}) = - i_{d}^2\cos(\sum_{d=1}^Di_{d}x_{d})$.
    Hence, all coordinate functions in \eqref{eqn:sep_eigenfuncs_TD} and \eqref{eqn:non_sep_eigenfuncs_TD} are eigenfunctions of $\Delta_{\mathbb{T}^{D}}$ with the same eigenvalue, namely $\lambda=\sum_{d=1}^Di_{d}^2$.
    \item  \eqref{eqn:non_sep_eigenfuncs_TD} and \eqref{eqn:sep_eigenfuncs_TD} have the same number of functions. This is due to the fact they are indexed by the same condition $i_{1}^2,\ldots,i_{d}^2=\lambda$, for $(i_{1}, \ldots, i_{D})\in\mathbb{Z}^{D}$, with a multiplicative factor of 2 coming from the choice between sin and cosine.
    \item The elements of \eqref{eqn:non_sep_eigenfuncs_TD} are pairwise orthonormal. This can be established from the following standard sum to product identities: 
            $$
            \sin(\sum_{d=1}^Di_{d}x_d)\sin(\sum_{d=1}^Di_{d}^{\prime}x_d) = \frac{1}{2}\left[\cos(\sum_{d=1}^D(i_{d}-i_{d}^\prime)x_d) - \cos(\sum_{d=1}^D(i_{d}+i_{d}^\prime)x_d)\right]
            $$
            $$
            \cos(\sum_{d=1}^Di_{d}x_d)\cos(\sum_{d=1}^Di_{d}^{\prime}x_d) = \frac{1}{2}\left[\cos(\sum_{d=1}^D(i_{d}+i_{d}^\prime)x_d) + \cos(\sum_{d=1}^D(i_{d}-i_{d}^\prime)x_d)\right]
            $$
            $$
            \sin(\sum_{d=1}^Di_{d}x_d)\cos(\sum_{d=1}^Di_{d}^{\prime}x_d) = \frac{1}{2}\left[\sin(\sum_{d=1}^D(i_{d}+i_{d}^\prime)x_d) + \sin(\sum_{d=1}^D(i_{d}-i_{d}^\prime)x_d)\right],
            $$
    and hence the integrals are clearly zero since $i_{d}-i_{d}^\prime\in\mathbb{Z}$.
    \item That any element in \eqref{eqn:non_sep_eigenfuncs_TD} can be written as a linear combination of elements in \eqref{eqn:sep_eigenfuncs_TD} can be established by induction. 
    \par 
    For $D=2$, this holds using standard  sum and difference formulas
    $$
    \begin{aligned}
    \cos(i_{1}x_{1}+i_{2}x_{2})&=\cos(i_{1}x_{1})\cos(i_{2}x_{2}) - \sin(i_{1}x_{1})\sin(i_{2}x_{2}) \\
    \sin(i_{1}x_{1}+i_{2}x_{2})&=\sin(i_{1}x_{1})\cos(i_{2}x_{2}) + \cos(i_{1}x_{1})\sin(i_{2}x_{2})
    \end{aligned}
    $$ 
    Assuming this holds for $D$. Then for $D+1$ we again invoke the sum and difference formulas
    \begin{equation}\label{eqn:sumAndDiff_D1}
    \begin{aligned}
    \cos(\sum_{d=1}^Di_{d}x_{d}+i_{D+1}x_{D+1})&=\cos(\sum_{d=1}^Di_{d}x_{d})\cos(i_{D+1}x_{D+1}) - \sin(\sum_{d=1}^Di_{d}x_{d})\sin(i_{D+1}x_{D+1}) \\
    \sin(\sum_{d=1}^Di_{d}x_{d}+i_{D+1}x_{D+1})&=\sin(\sum_{d=1}^Di_{d}x_{d})\cos(i_{D+1}x_{D+1}) + \cos(\sum_{d=1}^Di_{d}x_{d})\sin(i_{D+1}x_{D+1}).
    \end{aligned}
    \end{equation}
    From the induction hypothesis, $\sin(\sum_{d=1}^Di_{d}x_{d})$ and $\cos(\sum_{d=1}^Di_{d}x_{d})$ can be written as a linear combination of separable marginal functions, hence the RHS of \eqref{eqn:sumAndDiff_D1} is also a linear combination of separable marginal functions
\end{enumerate}
This establishes that there exists a rotation matrix $\boldsymbol{A}_{\lambda}$ such that the non-separable eigenbasis functions $\tilde{\boldsymbol{\psi}}_{\lambda} = \boldsymbol{A}_{\lambda}\boldsymbol{\psi}_{\lambda}$, hence forming a complete and orthonormal basis for the eigenspace corresponding to $\lambda$. 
\begin{figure}[!ht]
    \centering
    \includegraphics[width=0.8\linewidth]{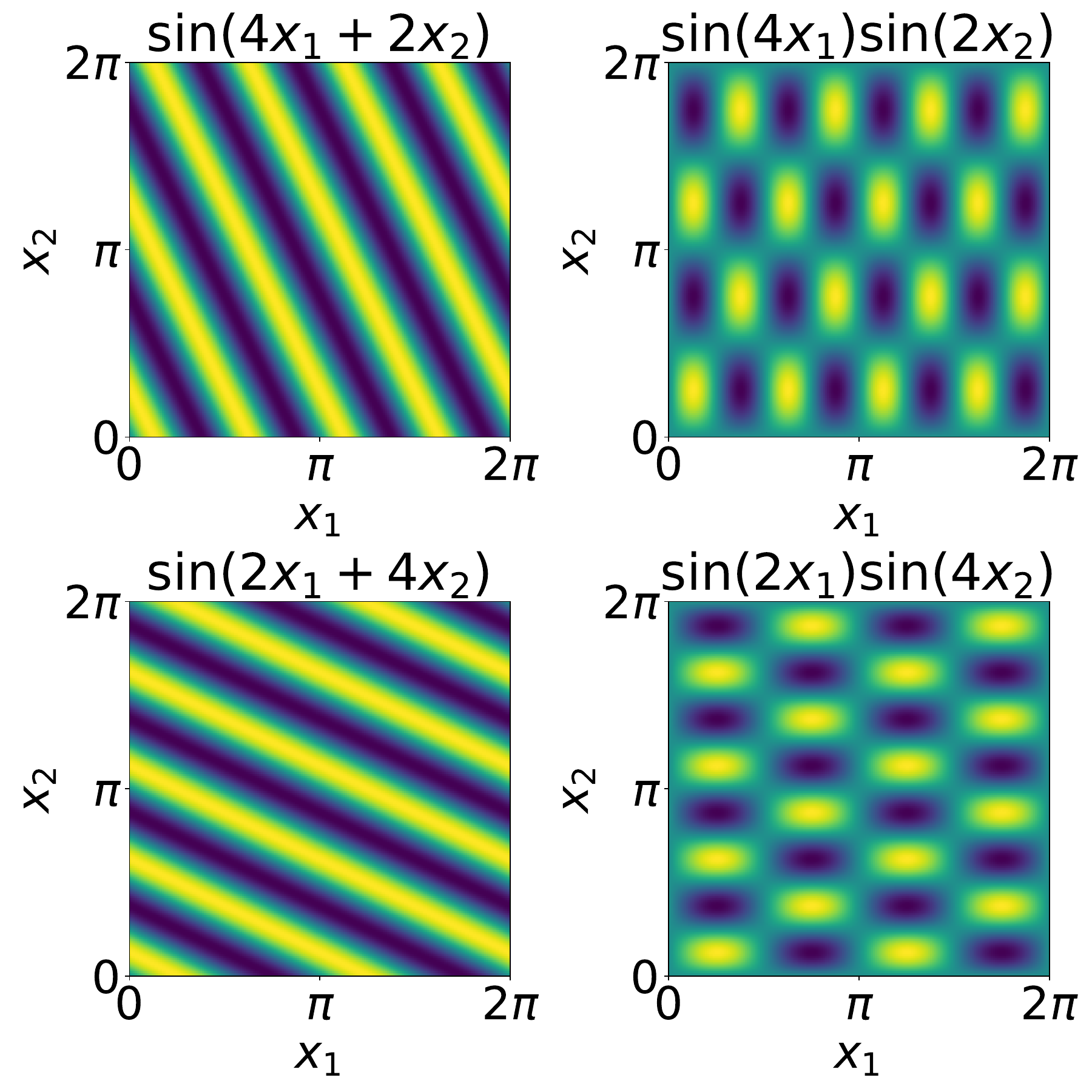}
    \caption{Two non-separable eigenfunctions (left column) and separable eigenfunctions (right column) of $\Delta_{\mathbb{T}^2}$  for the same eigenspace $\lambda = 20$.}
    \label{fig:separable_vs_non_seperable_encodings}
\end{figure}
\par 
Figure~\ref{fig:separable_vs_non_seperable_encodings} plots two randomly selected basis functions for both separable and non-separable formulations for $\mathbb{T}^2$. Notice that, due to the separable structure, all basis functions in \eqref{eqn:sep_eigenfuncs_TD} have oscillations only in the marginal directions, i.e. along the coorindate axis for each copy of $\mathbb{S}^1$, while the functions from \eqref{eqn:non_sep_eigenfuncs_TD} have oscillations over varying ``diagonal'' directions in $\mathbb{T}^{D}$. %Said differently, the functions \eqref{eqn:sep_eigenfuncs_TD} differ only in their frequencies along the coordinate axis, but not in the direction of oscillations, while for \eqref{eqn:non_sep_eigenfuncs_TD}, the functions differ both in frequency and direction. 
See Section~\ref{ssec:spec_bias} for an empirical comparison of the impact of separable and non-separable encodings on network convergence. 

\section{Competing Methods}

\subsection{Tensor Product Basis Density Estimator}\label{apx:tensor_product_learning}
With a slight change of notation for clarity, denote the truncation of the marginal eigenfunctions of $\Delta_{\mathcal{M}_{d}}$ as $\phi_{d} := (\phi_{d,1}, \ldots \phi_{d,r_{d}})$, where $r_{d}$ is the rank. 
The tensor product basis estimator is a special case of \eqref{eqn:nf} for the following specification:
\begin{enumerate}
    \item $\boldsymbol{\eta} := \bigotimes_{d=1}^D\phi_{d}$, i.e. the full set of tensor product functions with $K=\prod_{d=1}^Dr_{d}$
    \item $L=1$, so $\boldsymbol{\theta} := \boldsymbol{W}^{(1)}\in\mathbb{R}^{\prod_{d=1}^Dr_{d}\times 1}$ and  
    $$
    v_{\boldsymbol{\theta}}(x_1,\ldots,x_{D}) = \boldsymbol{\theta}^{\intercal}\boldsymbol{\eta}(x_1,\ldots,x_{D})$$
\end{enumerate}
and hence the task is to find the coefficients $\boldsymbol{\theta}$. WLOG, assume $r_1=r_2=\ldots=r_{D}=r$, so $\boldsymbol{\theta}\in\mathbb{R}^{r^{D}}$.  Plugging this representation into \eqref{eqn:MAP_pp_reparam}, we get the optimization problem:
\begin{equation}\label{eqn:lin_model_optimization_problem}
\begin{aligned}
    \widehat{\boldsymbol{\theta}} &= \max_{\boldsymbol{\theta}\in\mathbb{R}^{r^{D}}}  \Big(\frac{1}{n}\sum_{i=1}^n \boldsymbol{\theta}^\intercal \boldsymbol{\eta}(x_{1i},\ldots,x_{Di}) - \int_{\Omega}\exp\left(\boldsymbol{\theta}^\intercal \boldsymbol{\eta}\right)d\omega  - R_{\tau}(\boldsymbol{\theta}^\intercal \boldsymbol{\eta})\Big).\\
    \end{aligned}
\end{equation}
Under the linear basis expansion, the penalty term in \eqref{eqn:lin_model_optimization_problem} takes on a special form given by
\begin{equation}\label{eqn:lin_model_roughpenalty}
    \begin{aligned}
    R_{\tau}(\boldsymbol{\theta}^\intercal \boldsymbol{\eta}) &= \tau \int_{\Omega}[\Delta_{\Omega}\boldsymbol{\theta}^\intercal \boldsymbol{\eta}]^2d\omega = \tau \boldsymbol{\theta}^\intercal[\int_{\Omega}\Delta_{\Omega}\boldsymbol{\eta}[\Delta_{\Omega}\boldsymbol{\eta}]^{\intercal}d\omega]\boldsymbol{\theta} := \tau \boldsymbol{\theta}^{\intercal}\boldsymbol{F}\boldsymbol{\theta} \\
\end{aligned}
\end{equation}
where $\boldsymbol{F}\in\mathbb{R}^{r^{D}\times r^{D}}$ is the matrix of inner products of the LBO of the encoding function. Due to the fact that $\boldsymbol{\eta}$ are eigenfunctions, $\boldsymbol{F}$ is a diagonal matrix with elements given by the sums of eigenvalues $\sum_{d=1}^D\lambda_{i_{d}}$. Notice that $\exp(\boldsymbol{\theta}^{\intercal}\boldsymbol{\eta}(x_1,\ldots,x_{D}))$ is convex in $\boldsymbol{\theta}$, hence $\int_{\Omega}\exp(\boldsymbol{\theta}^{\intercal}\boldsymbol{\eta})d\omega$ is also convex in $\boldsymbol{\theta}$ \citepSupp{boyd2004convex}.
%(see page 79 of \citepSupp{boyd2004convex}). 
Using this fact and the convexity of the quadratic form of \eqref{eqn:lin_model_roughpenalty}, the optimization problem \eqref{eqn:lin_model_optimization_problem} is a sum of convex functions, and hence is convex (in $\boldsymbol{\theta}$). Then under some conditions on the decay of the learning rate \citepSupp{Bottou2018}, Algorithm \eqref{alg:sgd} is guaranteed to converge. 

\subsection{Product Squared Neural Families}\label{ssec:pSNF_explication}
\revised{For modeling density/intensity functions on product manifolds, \citep{tsuchida2024exact} propose to extend the squared neural families (SNF) model from \citep{tsuchida2023}, which models intensity/density using the following representation:
\begin{equation}\label{eqn:snef_model}
    \lambda(x) = \alpha\| V\psi(x)\|_{2}^2, \quad \psi(x) = \sigma(Wt(x) + b),
\end{equation}
where $x\in\mathcal{M}$, $\sigma$ is an activation function, $t:\mathcal{M}\mapsto\mathbb{R}^{H_{0}}$ is an encoding function, $W\in\mathbb{R}^{H_{1}\times H_{0}}$, $b\in\mathbb{R}^{n}$ is the weight and bias matrix, and $V\in\mathbb{R}^{H_{2}\times H_{1}}$ is a learnable matrix. The primary advantage of the representation \eqref{eqn:snef_model} is that through careful choices of encoding function $t$, activation function $\sigma$ and base measure $\mu:\mathcal{B}(\mathcal{M})\mapsto\mathbb{R}^{+}_0$, the normalization integral 
$$
\Lambda = \int_{\Omega}\lambda(x)d\mu(x),
$$
has an analytic solution. Specifically, the functions $t, \sigma, \mu$ must be chosen such that the inner product matrix 
$$
G = \int_{\mathcal{M}}[\sigma(Wt(x) + b)]^{\intercal}\sigma(Wt(x) + b)d\mu(x) \in \mathbb{R}^{H_1\times H_1}
$$
is an analytic function of parameters $W,b$. Note, this is where the depth restriction comes in, as it is very difficult to find combinations of hyperparameters that lead to closed form integrals under composition. The authors work out combinations of hyperparameters for exact integration for several standard manifolds \citep{tsuchida2023,tsuchida2024exact}.  Specifically, for $\mathbb{S}^{d-1}$, it is shown that for the uniform surface measure, identity warping ($t(x) := x$) and exponential activation ($\sigma=\exp$), the inner product matrix $G$ has closed form expression with element-wise definition 
$$
G(i,j) = \exp(b_{i} + b_{j})\frac{\Gamma(d/2)2^{d/2-1}J_{d/2-1}(\|w_i + w_j\|_2)}{\|w_i + w_j\|_2},
$$
where $J_{p}$ is the modified Bessell function of the first kind of order $p$, $w_i$ is the i'th row of $W$,  and $b_i$ is the $i$'th element of $b$. Note that for $\mathbb{S}^1$, this reduces to 
$$
G(i,j) := \exp(b_{i} + b_{j})J_0(\|w_i + w_j\|).
$$
}
\par 
\revised{In the case of product domains $\Omega = \bigtimes_{d=1}^D\mathcal{M}_{d}$, \citep{tsuchida2024exact} suggest the following element-wise product representation 
\begin{equation}\label{eqn:prod_snef_model}
\lambda(x_1,...,x_{D}) = \alpha\|V\psi(x_1,...,x_{D})\|_{2}^2, \quad \psi(x_1,...,x_{D}) = \bigodot_{d=1}^D\sigma_{d}(W_{d}t_d(x_d)+b_d),
\end{equation}
where $\odot$ is the Hadamard product. This product representation was chosen primarily to maintain the exact integration property, which is guaranteed so long as each of the marginal normalization integrals can be constructed analytically, since $G = \bigodot_{d=1}^D G_{d}$, where 
$$
G_{d} = \int_{\mathcal{M_{d}}}[\sigma(W_{d}t_{d}(x) + b_{d})]^{\intercal}\sigma(W_{d}t_{d}(x) + b_{d})d\mu_{d}(x) \in \mathbb{R}^{H_{1,d}\times H_{1,d}}.
$$  
Note that while the authors speculate on several other ways of ``mixing'' across the product domains, including replacing the elementwise product in \eqref{eqn:prod_snef_model} with an outer product, their suggestion is to use the elementwise product since it is their only proposal that both inherits the closed form integration property from the marginal models and avoids the computational curse of dimensionality. For example, the outer product approach would require differentiating through the matrix $G=\bigotimes_{d=1}^D G_{d}$ during training, resulting in exponential growth in the parameter $V$, since $H_{2} = \prod_{d=1}^DH_{1,d}$.}
\par 
\revised{For the experiments in Section~\ref{ssec:simulation_studies}, we set $H_{1}=H_{2}$ $\forall d=1,...,D$. To approximately match the parameter count in the NeuroPMD, for the $\mathbb{T}^2$ case we set $H_{1}=H_{2}=181$, resulting in a total of $33,486$. For the $\mathbb{T}^4$ case, to approximately match NeuroPMD we initially set $H_{1}=H_{2}=222$ for a total of $50,173$ trainable parameters, but found that the training was unstable and the network did not learn a reasonable estimate. Therefore, we increased the number of trainable parameters and ultimately found that $H_{1}=H_{2}=257$ (total parameters = 68,106) resulted in reliable convergence and reasonable estimates for this more challenging setting. Furthermore, we found that the default initialization from the \texttt{snefy} GitHub package  package did not lead to reliable convergence. After experimentation, we identified a stable initializing strategy: initializing the weight matrices $W_d$ with mean-zero normals with variance $\frac{4}{D}$ and the $V$ matrix with standard normal entries. The network was trained with the ADAM optimizer.}

\section{Additional Synthetic Data Experiments}\label{sec:additional_experiments}

\subsection{Synthetic Data Generation}

The true density functions were defined as equally--weighted mixtures of anisotropic wrapped normal distributions \citepSupp{mardia2009directional}. The covaraince matrices for each mixture component were defined using a parameter \textit{anisotropy-factor} $\in[0,\infty)$, which defines the ratio between the maximum and minimum eigenvalue, hence controlling the degree of local anisotropy. Specifying different   anisotropy-factors for different mixture components allows us to simulate density functions with spatially varying anisotropy. Specifically, the covariance matrices were  created for each mixture component as follows:
\begin{enumerate}
    \item A random $D\times D$ orthogonal matrix $\boldsymbol{Q}$ was sampled. 
    \item For a pre-defined anisotropy-factor, eigenvalues were calculated according to a logarithmically spaced grid from $\{1, ..., \text{anisotropy-factor}\}$ and collected to form the diagonal matrix $\boldsymbol{\Lambda}$.
    \item The covaraince matrix was calculated as $\boldsymbol{C} = \frac{\boldsymbol{Q}\boldsymbol{\Lambda}\boldsymbol{Q}^{\intercal}}{\text{max}_{ij}[\boldsymbol{Q}\boldsymbol{\Lambda}\boldsymbol{Q}^{\intercal}]_{ij}}$, where $\text{max}_{ij}[\boldsymbol{M}]_{ij}$ returns the maximum over all elements in the matrix $\boldsymbol{M}$.
\end{enumerate}

For $\mathbb{T}^2$, we set the number of mixture components to three with means $(\pi, \pi/2)$, $(\pi, 5\pi/3)$, $(\pi/4, \pi)$ and anisotropy-factors $100,100,20$, respectively. For $\mathbb{T}^4$, we used five components with mean vectors: $(0.5, 0.5, 0.5, 0.5)$, $(\pi, \pi, \pi, \pi)$, $(\pi/2, 3\pi/2, 3\pi/2, \pi/2)$, \\ $(3\pi/2, \pi/2, \pi/2, 3\pi/2)$, $(\pi/4, \pi/4, 7\pi/4, 7\pi/4)$, with anisotropy-factors $100,100,20,50,75$, respectively.

\subsection{Comparison to ReLU MLP Model}\label{sssec:baseline_MLP}
MLP-based neural fields with standard ReLU activations are known to suffer from spectral bias, often facing significant difficulties in training and exhibiting poor convergence in the high-frequency components of the underlying function \citepSupp{rahaman2019}. However, to our knowledge, most empirical evaluations of this effect have been conducted within the context of $L^2$-based learning, which is not our setting. To investigate this behavior in our case, we substituted the network architecture \eqref{eqn:nf} with a ``vanilla'' ReLU MLP, where all hidden activations are ReLU functions and the initial encoding layer is the identity:
$$
\begin{aligned}
\mathbf{h}^{(0)} &= \mathbf{x}, \\
\mathbf{h}^{(l)} &= \sigma \left( \mathbf{W}^{(l)} \mathbf{h}^{(l-1)} + \mathbf{b}^{(l)} \right), \quad \text{for } l = 1, \ldots, L-1, \\
v_{\theta}(\mathbf{x}) &= \mathbf{W}^{(L)} \mathbf{h}^{(L-1)},
\end{aligned}
$$
where $\sigma$ is the ReLU function. For a fair comparison, we used the same network depth, width, and training algorithm configuration outlined in Section~\ref{ssec:simulation_studies} and evaluated the model on the same synthetic data. The one difference was that, instead of applying a roughness penalty to enforce regularity, we utilized early stopping, as the ReLU network is no longer second-order differentiable. The early stopping iteration $T$ was selected using the same ISE criteria defined in equation \eqref{eqn:l2_selection_criteria} on a small held out validation set. Table~\ref{tab:sim_vanilla_MLP} provides the results for both $\mathbb{T}^2$ and $\mathbb{T}^4$ cases. We see that the ReLU MLP performs poorly in both cases compared to all estimators considered in the main text (see Table~\ref{tab:sim_1_2_MC_results}). Given that such ReLU networks are universal approximators in the wide-width limit, these results highlight the importance of carefully designing architectures to ensure fast and stable convergence to a quality solution in practice, effectively mitigating spectral bias.

\begin{table}
\centering
\begin{tabular}{llr}
\toprule
 & & ReLU MLP \\
\midrule
\multirow{2}{*}{$\mathbb{T}^2$} & FR & $1.665 \pm 4.37 \times 10^{-3}$ \\
 & nISE & $0.956 \pm 2.96 \times 10^{-3}$ \\
\midrule
\multirow{2}{*}{$\mathbb{T}^4$} & FR & $1.540 \pm 1.558 \times 10^{-2}$ \\
 & nISE & $0.980 \pm 4.70 \times 10^{-4}$ \\
\bottomrule
\end{tabular}
\caption{Monte Carlo average simulation results for ReLU MLP with encoding $\boldsymbol{\eta}(\boldsymbol{x})=\boldsymbol{x}$.}
\label{tab:sim_vanilla_MLP}
\end{table}

\subsection{Spectral Bias}\label{ssec:spectralBiasEval}
\revised{
%Neural networks are known to exhibit a ``spectral bias'', i.e., a tendency to initially learn smoother, lower-frequency representations during training, before gradually learning higher-frequency components \citep{rahaman2019}. This phenomena has been linked to the strong generalization properties of neural networks in certain tasks \citep{Cao2019TowardsUT}, but is generally undesirable in NF models, as it can make learning important high-frequency functional features impractical due to the prohibitively large number of training iterations required. 
In this section, we assess the spectral bias of NeuroPMD through two analyses: an ablation study comparing separable and non-separable first-layer encodings (Section~\ref{ssec:spec_bias}), and a comparison with the competing neural density estimator pSNF (Section~\ref{sssec:spec_bias_NeuroPMD_vs_pSNEF}).}
\subsubsection{Separable vs. Non-Separable First-Layer Encodings}\label{ssec:spec_bias}
%Deep neural networks are known to exhibit a ``spectral bias'', i.e., a tendency to initially learn smoother, lower-frequency representations during training, before gradually incorporating higher-frequency components \citep{rahaman2019}. This phenomena has been linked to the strong generalization properties of neural networks in certain tasks \citep{Cao2019TowardsUT}, but is generally undesirable in NF models, as it can make learning important high-frequency functional features arduous. \cite{tancik2020} show that proper \textit{first-layer encodings} ($\boldsymbol{\eta}$ in \eqref{eqn:nf}) can help mitigate this bias, allowing NFs to converge to high-frequency function representations much faster. \par 
In this section, we examine the spectral bias of the random separable LBO encoding \eqref{eqn:sep_eigenfuncs_TD} and the rotated random non-separable LBO encoding \eqref{eqn:non_sep_eigenfuncs_TD}, both integrated into identical network architectures 
and trained under the same parameterization of optimization Algorithm \ref{alg:sgd} for $T=5,000$ iterations, using an identical fixed learning rate schedule to ensure a fair comparison. We evaluate the convergence of the networks on randomly selected replications for both the $\mathbb{T}^{2}$ and $\mathbb{T}^{4}$ synthetic data discussed in Section~\ref{ssec:simulation_studies}. Due to the randonmenss in both Algorithm~\ref{alg:sgd} and network model \eqref{eqn:nf}, the optimization was repeated 5 times with different randomization seeds for both synthetic examples. Every $10$ iterations during each training run, the current parameter estimate $\boldsymbol{\theta}^c$ was used to calculate the FR error with the ground truth density. 
\begin{figure}
    \centering
    \includegraphics[width=\textwidth]{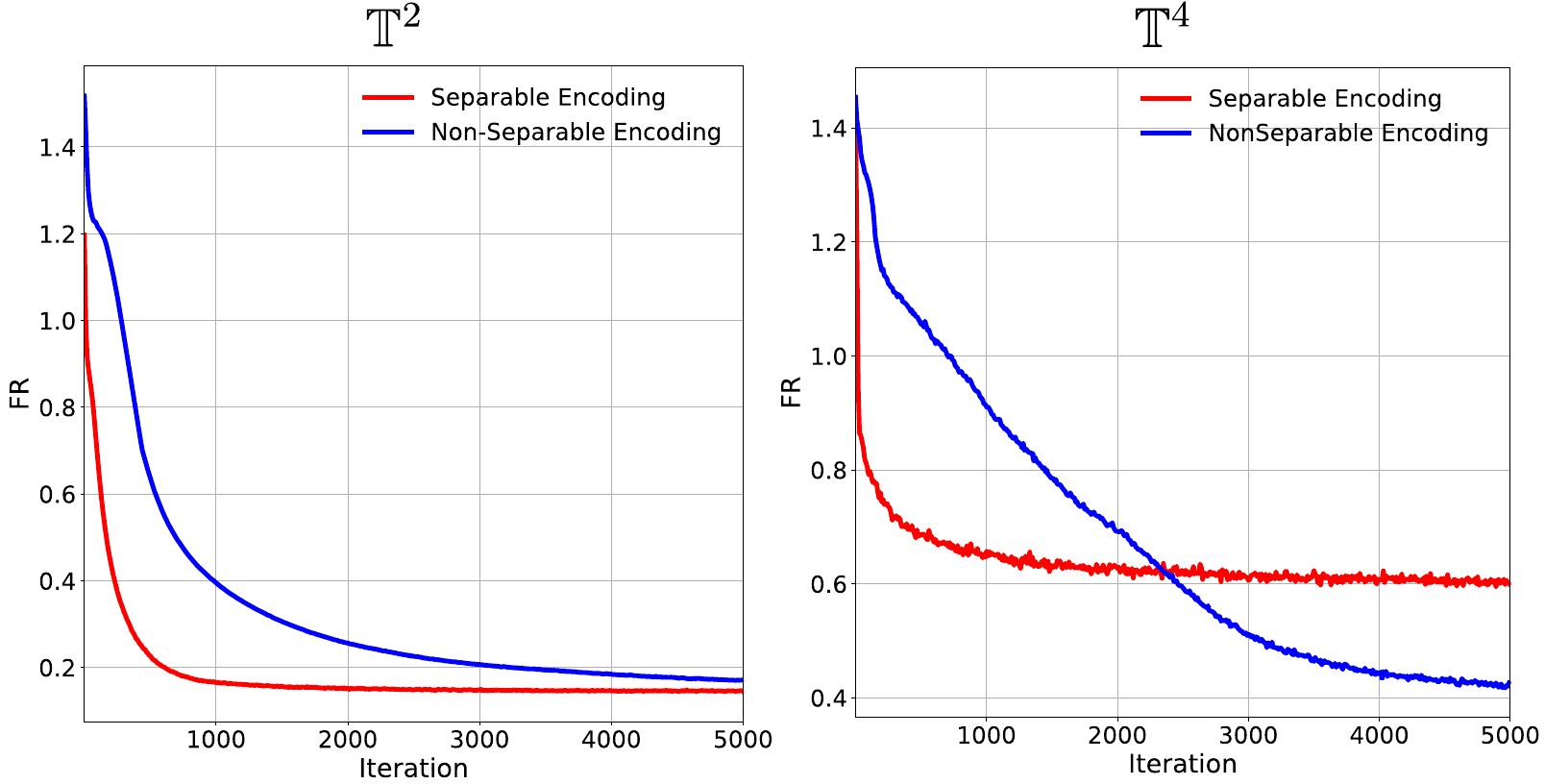}
    \caption{Median convergence profiles across runs of Algorithm~\ref{alg:sgd} in terms of the FR error for the $\mathbb{T}^{2}$ example (left) and $\mathbb{T}^{4}$ example (right).}
    \label{fig:spec_bias_comparison}
\end{figure}
\par 
Figure~\ref{fig:spec_bias_comparison} shows the resulting median convergence profiles across runs of Algorithm~\ref{alg:sgd} in terms of the FR error for the $\mathbb{T}^{2}$ example (left) and $\mathbb{T}^{4}$ example (right). For $\mathbb{T}^{2}$, we observe the separable encoding achieves faster convergence speed. In contrast, the convergence behavior in the high-dimensional
$\mathbb{T}^{4}$ case shows a more complex pattern: while separable encodings initially converge rapidly, their convergence rate slows considerably later in the training. The non-separable encodings are a bit slower to begin with, but ultimately exhibit superior convergence speed when considering the entirety of the training duration, making them more effective at controlling spectral bias in this regime.
\par 
As illustrated in Figure~\ref{fig:separable_vs_non_seperable_encodings}, the encodings appear to introduce different implicit biases. Specifically, the separable encodings concentrate all frequency oscillations along the same marginal axes, while the non-separable encodings distribute oscillations across multiple directions. Whether the implicit bias of the non-separable encodings consistently offer superior performance in high-dimensional domains remains an open question and an important avenue for future research.
\par 
Providing further context are the results of the ``vanilla'' ReLU MLP trained directly on spatial coordinates, presented in  Table~\ref{tab:sim_vanilla_MLP}. This baseline approach effectively forms a low-frequency encoding (the identity encoding) that does not account for manifold structure, and exhibits significantly worse performance. These results further highlight the critical role of first-layer encodings $\boldsymbol{\eta}$ in the performance of the network architecture \eqref{eqn:nf} estimated via Algorithm~\ref{alg:sgd}. 
\par 
In practice, the ISE criterion in \eqref{eqn:l2_selection_criteria} can be used to select between encoding types. Table~\ref{tab:selecting_encoding} reports the median ISE criterion at iteration $T=5,000$, calculated using a small held-out test set. The results show a lower value for the separable encoding in the $\mathbb{T}^2$ case and a lower value for the non-separable encoding for the $\mathbb{T}^4$ case, consistent with the ground truth performance observed in Figure~\ref{fig:spec_bias_comparison}. Since this criterion can be calculated without access to ground truth, it provides an effective method for selecting the $\boldsymbol{\eta}$ hyperparameter in practice.

\begin{table}
\centering
\begin{tabular}{|c|c|c|}
\hline
 & $\mathbb{T}^2$ & $\mathbb{T}^4$ \\
\hline
Separable Encoding & -6.876 & -32.393 \\
\hline
Non-Separable Encoding & -6.799 & -38.473 \\
\hline
\end{tabular}
\caption{Median ISE criteria (Equation~\ref{eqn:l2_selection_criteria}) at iteration $T=5,000$ for different encoding types, calculated on a held-out test set. }
\label{tab:selecting_encoding}
\end{table}

\subsubsection{NeuroPMD vs. pSNEF}\label{sssec:spec_bias_NeuroPMD_vs_pSNEF}
\begin{figure}[ht!]
    \centering
    \includegraphics[width=\textwidth]{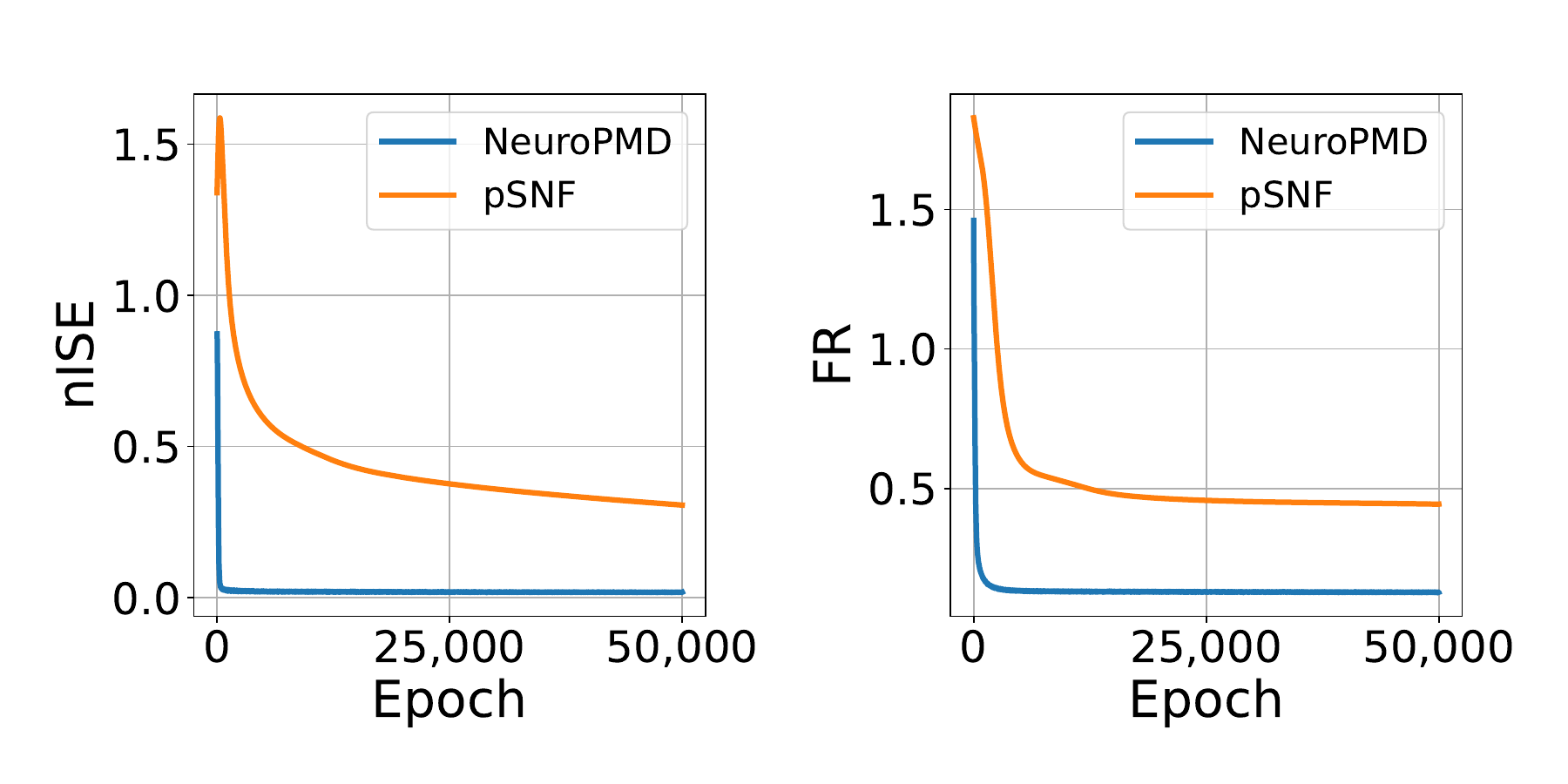}
    \caption{\revised{nISE (left) and FR (right) of NeuroPMD and pSNF estimates as a function of training epoch.}}
    \label{fig:spec_bias_learning_curves}
\end{figure}
\revised{In this section, we perform an empirical comparison of the spectral bias of the two neural network density estimators evaluated in Section~\ref{ssec:simulation_studies}, NeuroPMD and pSNF. We consider the $\mathbb{T}^{2}$ synthetic data and use the same architectures, batch sizes and initializations as outlined in the main text. For NeuroPMD, we again select the penalty hyperparameter $\tau$ using the scheme from Section~\ref{ssec:hyper_param_selection}. We now run the training for $50,000$ epochs of stochastic gradient descent for a fixed common learning rate of $10^{-5}$, for all $20$ replications.}
\par 
\revised{Figure~\ref{fig:spec_bias_learning_curves} shows both error metrics as a function of training epochs. We observe that NeuroPMD achieves significantly faster (algorithmic) convergence compared to pSNF. Figure~\ref{fig:ensity_estims_50k} shows the (log) density  estimates for both methods after $50,000$ epochs. Compared to the corresponding results at $10,000$ epochs in Figure~\ref{fig:sim_T2}, we see that while the pSNF estimates do improve, they still exhibit substantial oversmoothing, leading to severe bias in low density regions. In contrast, the NeuroPMD estimates at $50,000$ are nearly visually identical to those at 10,000 epochs, which is consistent with the convergence behavior shown in Figure~\ref{fig:spec_bias_learning_curves}. %This highlights the importance of proper roughness penalty selection, as we are able to identify a value of $\tau$ that yields stable convergence over a large number of training iterations, without overfitting.
} 
\begin{figure}[ht!]
    \centering
    \includegraphics[scale=0.75]{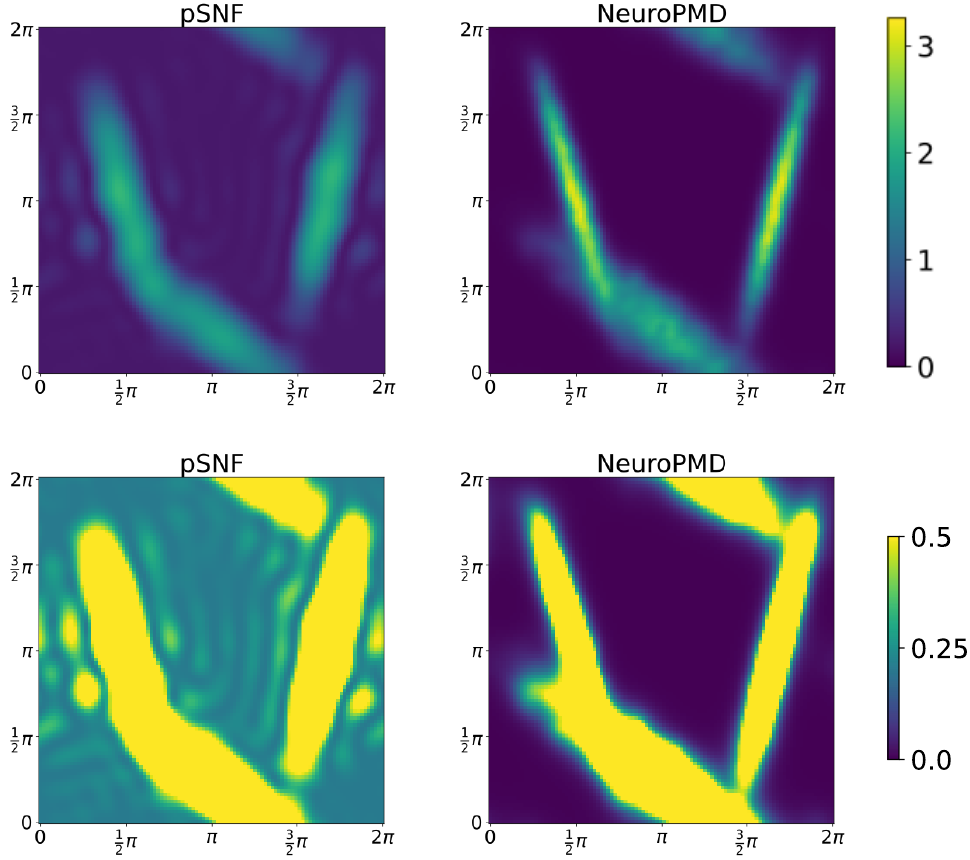}
    \caption{\revised{Log density estimates of NeuroPMD and pSNF after $50,000$ training epochs. Both rows present the same function visualized with different colorbars for enhanced comparison.}}
    \label{fig:ensity_estims_50k}
\end{figure}
\par 
\revised{Figure~\ref{fig:spectral_convergence} compares the convergence behavior of NeuroPMD and pSNF in the spectral domain. The top left panel shows the $2$-d Fourier transform of the true density function over the frequency domain, while the top right hand panel shows a zoomed-in view to highlight the dominant frequency content. The subsequent rows display the corresponding spectral representations for NeuroPMD and pSNF estimates, shown both over the full frequency domain and the zoomed region, for several training epochs. When comparing the zoomed ground truth spectral content to that of NeuroPMD, we see that our method has already captured most of the relevant frequency structure after only $500$ training epochs. In contrast, the spectral content of the pSNF estimates remains  quite different from ground truth even after $50,000$ epochs. Specifically, we notice that after an initial transient phase, the pSNF spectral content is highly concentrated near the origin and expands outward to higher frequencies very slowly (as measured by iteration number), a clear manifestation of the spectral bias. This behavior is not present in the NeuroPMD estimates, which rapidly recover the high frequency content of the target density.}
\par
\revised{Finally, we repeated the experiments with a larger pSNF architecture, increasing the width of both layers to $H_1=H_2=222$ ($50,173$ trainable parameters). We found virtually no difference in the results, indicating that the spectral bias of pSNF is not alleviated by increasing network size.} 
\begin{figure}
    \centering
    \includegraphics[scale=0.43]{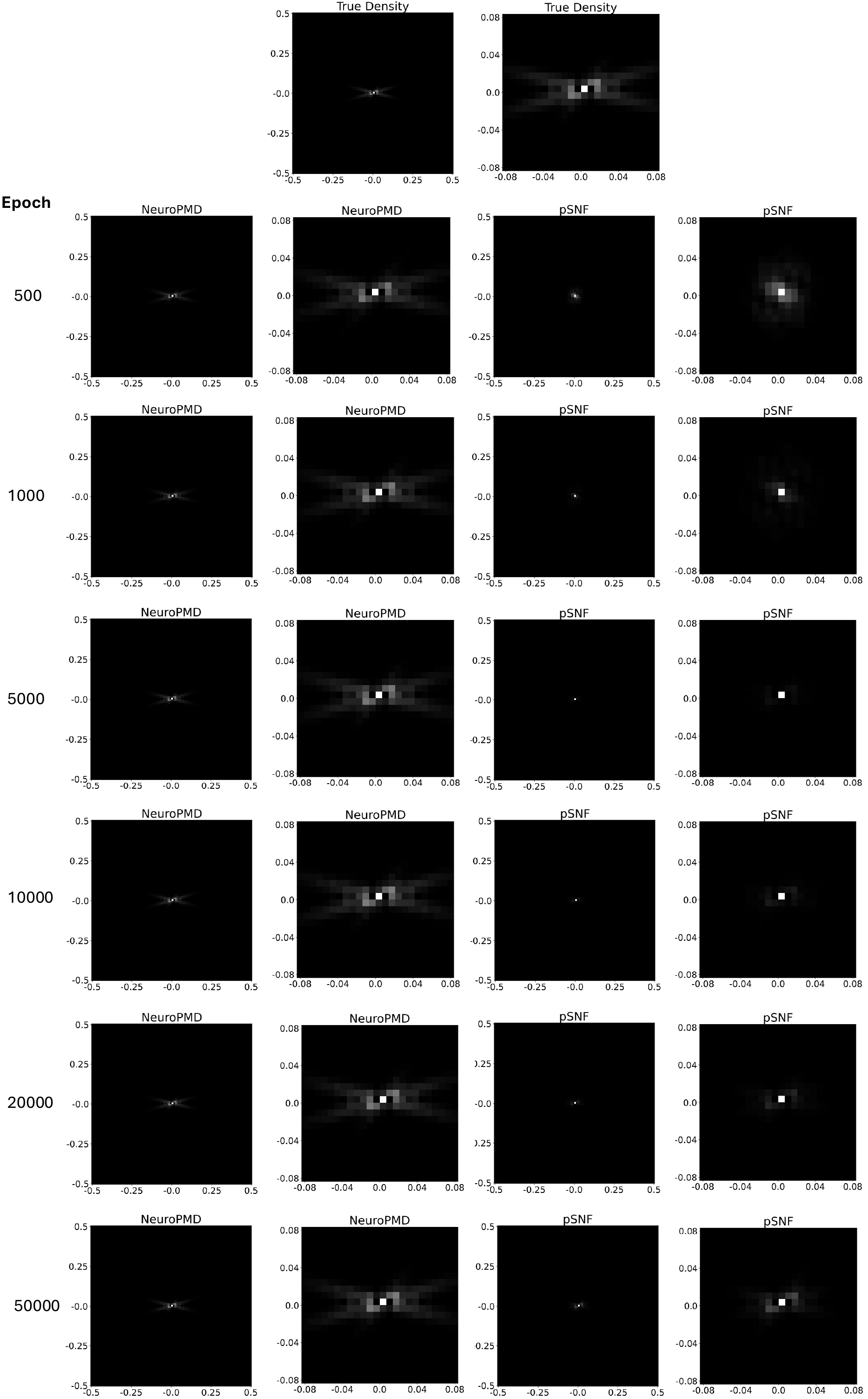}
    \caption{\revised{Spectral representations of the true density (top row) and the corresponding estimates from NeuroPMD and pSNF at various training epochs. Each representation is shown over the frequency domain (cycles/radian), and a zoomed-in view that highlights the dominant frequency content.}}
    \label{fig:spectral_convergence}
\end{figure}

\subsection{Selecting $\tau$ Parameter}\label{ssec:hyperparam}
\begin{figure}
    \centering
    \includegraphics[width=\textwidth]{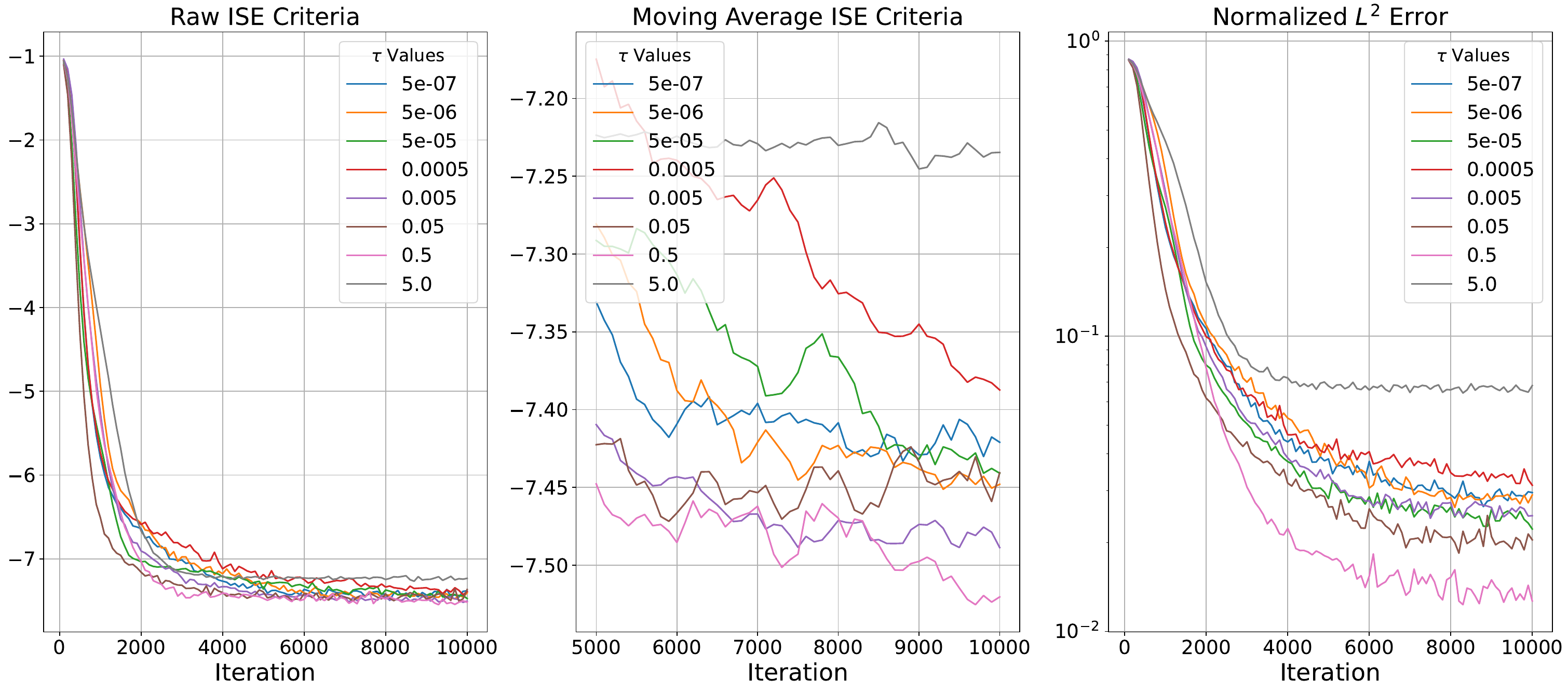}
    \caption{(Left) ISE criteria evaluated every 100 iterations of Algorithm~\ref{alg:sgd} on the validation set. (Middle) Smoothed moving average of the ISE criteria over the final 5,000 iterations. (Right) Normalized $L^2$ error (nISE) sampled along the optimization trajectory.}
    \label{fig:tau_select}
\end{figure}
Figure~\ref{fig:tau_select} displays the optimization trajectory for a randomly selected replication of the $\mathbb{T}^2$ synthetic data experiments described in  Section~\ref{ssec:simulation_studies}. The left plot shows the ISE criteria from Equation~\ref{eqn:l2_selection_criteria}, evaluated on the validation set using density estimates obtained every 100 iterations of Algorithm~\ref{alg:sgd}. The middle plot shows a moving average of the same ISE criteria, calculated with a sliding window of 50 iterations to smooth out stochastic variations, over the final $5,000$ iterations of the optimization trajectory. The right plot displays the normalized ISE criteria (nISE) sampled at the same points along the optimization trajectory. 
\par 
We observe a strong correspondence in the ordering of the $\tau$ curves with respect to the ISE criteria calculated on the validation set (left and middle plots) and the true normalized $L^2$-error with the ground truth density function (right plot). Combined with the strong aggregate simulation results reported in Section~\ref{ssec:simulation_studies}, this indicates strong performance for our hyperparameter selection scheme.

\subsection{Convergence Diagnostics}\label{ssec:convergence_diagnostics}
 
\begin{figure}[ht!]
    \centering
    \includegraphics[width=\textwidth]{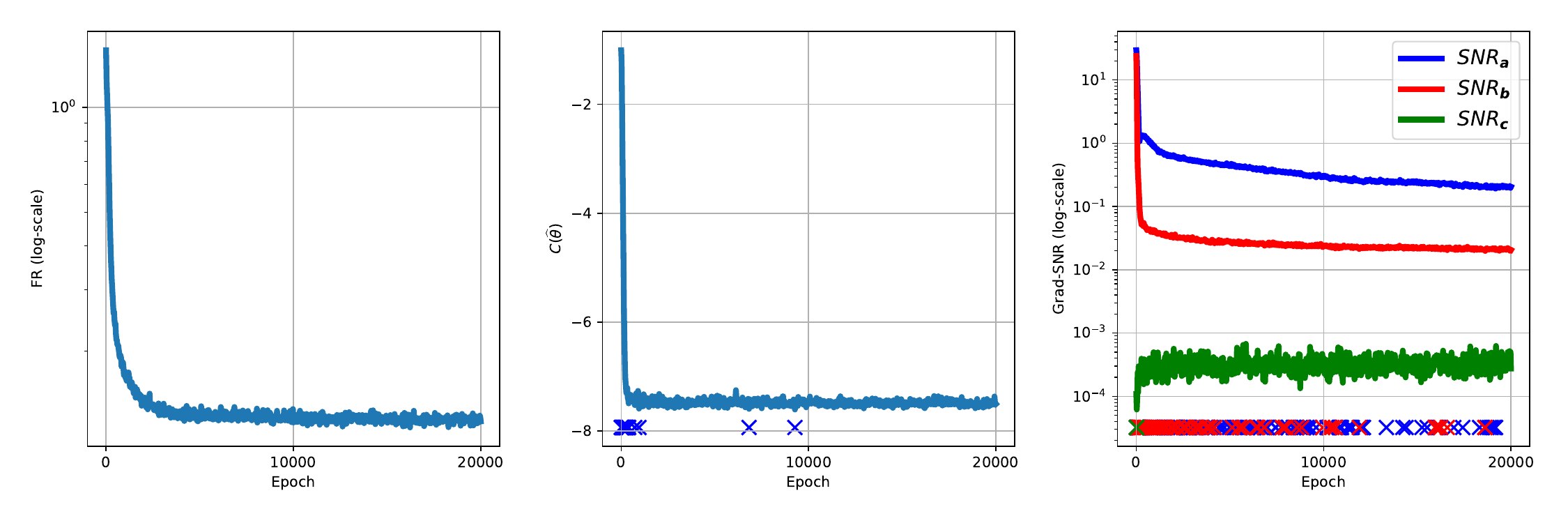}
    \caption{\revised{Convergence diagnostics for a representative NeuroPMD fit on a $\mathbb{T}^2$ replication. (left) Ground truth FR metric (log-scale) versus training epoch. (middle/right) Candidate convergence criteria as a function of training epochs. The colored markers on the x-axis indicate the location of running minima, e.g., for use in early stopping.}}
    \label{fig:convergence_diagnostics}
\end{figure}

\revised{In the experiments in Section~\ref{ssec:simulation_studies}, NeuroPMD was trained using a fixed number of epochs. This was done both because it is common practice in the neural field literature and to avoid any variability in performance between methods that may arise from early stopping.
%, i.e., using a consistent number of iterations promotes fair comparison. 
%It is worth noting that NeuroPMD converges significantly faster than TPB or pSNF. As a result, reducing the number of training epochs would only amplify the relative performance advantage of NeuroPMD.
That said, certain situations may require some from of early stopping, that is, monitoring a convergence criterion during training and halting if the measure has not improved after a prescribed number of epochs (called the patience). To explore this in our setting, we propose two convergence measures as candidate early-stopping rules, and conclude the section with a brief evaluation of their performance. The proposed measures are described below:}
\begin{enumerate}
    \item \revised{The first approach tracks the ISE criterion \eqref{eqn:l2_selection_criteria} on a small held-out validation set.}
    \item \revised{The second approach avoids the need for a validation set and instead leverages the so called gradient signal-to-noise ratio (SNR), defined as
    $$
    \text{SNR}_p := \left[\frac{\partial}{\partial\boldsymbol{\theta}}\mathcal{L}(o,\boldsymbol{\theta})\right]_{p}^{2}/\left[\text{Var}\left(\frac{\partial}{\partial\boldsymbol{\theta}}\mathcal{L}(o,\boldsymbol{\theta}\right)\right]_{pp}
    $$
    for parameter $\boldsymbol{\theta}_p$. This statistic has been used for both early stopping \citeSupp{mahsereci2017early} and network architecture selection \citeSupp{siems2021dynamic}. We adapt this metric to our case as follows: Recall from Prop.~\ref{prop:unbiased_grads}, we have the unbiased gradient estimator 
    $$
    \boldsymbol{g}^{b,q_{1},q_{2},\tau}(\boldsymbol{\theta}) = \boldsymbol{a}^{b}(\boldsymbol{\theta}) - \boldsymbol{b}^{q_{1}}(\boldsymbol{\theta}) - \boldsymbol{c}_{\tau}^{q_{2}}(\boldsymbol{\theta}),
    $$
    where 
    $$
    \begin{aligned}
        \boldsymbol{a}^{b}(\boldsymbol{\theta}) &= \frac{1}{b}\sum_{i=1}^b\frac{\partial}{\partial\boldsymbol{\theta}}v_{\boldsymbol{\theta}}(x_{1i}, \ldots,x_{Di}) := \frac{1}{b}\sum_{i=1}^b \boldsymbol{a}_i(\boldsymbol{\theta})\\
        \boldsymbol{b}^{q_{1}}(\boldsymbol{\theta}) &=  \frac{\text{Vol}(\Omega)}{q_{1}}\sum_{j=1}^{q_{1}}\frac{\partial}{\partial\boldsymbol{\theta}}\exp(v_{\boldsymbol{\theta}}(x_{1j}, \ldots, x_{Dj})) := \frac{\text{Vol}(\Omega)}{q_{1}}\sum_{j=1}^{q_{1}} \boldsymbol{b}_j(\boldsymbol{\theta)}\\
        \boldsymbol{c}^{q_{2}}_{\tau}(\boldsymbol{\theta}) &=  \tau\frac{\text{Vol}(\Omega)}{q_{2}}\sum_{l=1}^{q_{2}}\frac{\partial}{\partial\boldsymbol{\theta}}\left[\Delta_{\Omega}v_{\boldsymbol{\theta}}(x_{1l}, \ldots, x_{Dl})\right]^2 := \tau\frac{\text{Vol}(\Omega)}{q_{2}}\sum_{l=1}^{q_{2}}\boldsymbol{c}_l(\boldsymbol{\theta}).
    \end{aligned}    
    $$
    We estimate the variance 
    $$
    \begin{aligned}
        \text{Var}(\boldsymbol{g}^{b,q_{1},q_{2},\tau}(\boldsymbol{\theta})) &=  \text{Var}(\boldsymbol{a}^{b}(\boldsymbol{\theta}) - \boldsymbol{b}^{q_{1}}(\boldsymbol{\theta}) - \boldsymbol{c}_{\tau}^{q_{2}}(\boldsymbol{\theta}))\\
        &=\text{Var}(\boldsymbol{a}^{b}(\boldsymbol{\theta})) + \text{Var}(\boldsymbol{b}^{q_{1}}(\boldsymbol{\theta})) + \text{Var}(\boldsymbol{c}_{\tau}^{q_{2}}(\boldsymbol{\theta})), \\
    \end{aligned}
    $$
    %taken w.rt. to the joint distribution $\text{Unif}\{1, ..., n\}\times\text{Unif}(\Omega)\times\text{Unif}(\Omega)$
    by using a diagonal approximation formed from the empirical covariances of each component variance in the above sum, i.e.
    $$
    \begin{aligned}
        \widehat{\text{Var}}(\boldsymbol{a}^{b}(\boldsymbol{\theta}))  &= \frac{1}{b-1}\sum_{b=1}^{b}(\boldsymbol{a}_i(\boldsymbol{\theta}) - \bar{\boldsymbol{a}}(\boldsymbol{\theta}))^{\odot 2}\\ 
        \widehat{\text{Var}}(\boldsymbol{b}^{q_{1}}(\boldsymbol{\theta}))  &= \frac{\left[\text{Vol}(\Omega)\right]^2}{q_{1}-1}\sum_{j=1}^{q_{1}}(\boldsymbol{b}_j(\boldsymbol{\theta}) - \bar{\boldsymbol{b}}(\boldsymbol{\theta}))^{\odot 2}  \\
       \widehat{\text{Var}}(\boldsymbol{c}_{\tau}^{q_{2}}(\boldsymbol{\theta})) &=\tau^{2}\frac{\left[\text{Vol}(\Omega)\right]^2}{q_{2}-1}\sum_{l=1}^{q_{2}}(\boldsymbol{c}_l(\boldsymbol{\theta}) - \bar{\boldsymbol{c}}(\boldsymbol{\theta}))^{\odot 2}. 
    \end{aligned}
    $$
    The diagonal approximation is used to avoid storage and computation of the full $\text{dim}(\Theta)^2$-dimensional empirical covariance, which is intractable for deep networks. Finally, we compute the average gradient signal to noise ratio for each component of $\boldsymbol{g}^{b,q_{1},q_{2},\tau}(\boldsymbol{\theta})$ via 
\begin{equation}\label{eqn:grad_snr_per_compnent}
    \begin{aligned}
        \text{SNR}_{\boldsymbol{a}} &= \frac{1}{\dim(\Theta)}\sum_{p=1}^{\dim(\Theta)}\left[\boldsymbol{a}^{b}(\boldsymbol{\theta})\right]_{p}^2/\left[\widehat{\text{Var}}(\boldsymbol{a}^{b}(\boldsymbol{\theta}))\right]_{p}\\
        \text{SNR}_{\boldsymbol{b}} &= \frac{1}{\dim(\Theta)}\sum_{p=1}^{\dim(\Theta)}\left[\boldsymbol{b}^{q_{1}}(\boldsymbol{\theta})\right]_{p}^2/\left[\widehat{\text{Var}}(\boldsymbol{b}^{q_{1}}(\boldsymbol{\theta}))\right]_{p}\\
        \text{SNR}_{\boldsymbol{c}} &= \frac{1}{\dim(\Theta)}\sum_{p=1}^{\dim(\Theta)}\left[\boldsymbol{c}_{\tau}^{q_{2}}(\boldsymbol{\theta})\right]_{p}^2/\left[\widehat{\text{Var}}(\boldsymbol{c}^{q_{2}}_{\tau}(\boldsymbol{\theta}))\right]_{p}\\
    \end{aligned}
    \end{equation}}
\end{enumerate}
\revised{Figure~\ref{fig:convergence_diagnostics} (middle/right panels) displays both proposed criteria as a function of training epoch (evaluated every 25 epochs) for a representative NeuroPMD fit on a randomly selected $\mathbb{T}^2$ experimental replication. The colored markers on the x-axis indicate running minima of each criteria. The left hand plot shows the FR metric (log-scale) vs. epoch to assess ground truth convergence. We see that Criteria 1 attains almost all running minimia within the first $1,000$ epochs and then plateaus. In contrast, both  $\text{SNR}_{\boldsymbol{a}}$ and $\text{SNR}_{\boldsymbol{b}}$ continue to decline much later into the training, resulting in a dense sequence of running mimima until $\approx 10,000$ epochs. Given that the left hand plot shows that NeuroPMD does not fully reach its ground truth FR error floor until at least $7,500-10,000$ epochs, an early stopping rule based on Criteria 1 (with any reasonable patience) would almost certainly halt training prematurely. On the other hand, Criterion 2 is better able to measure the marginal improvements achieved later in training, and thus less likely to terminate too early. Notably, we see that $\text{SNR}_{\boldsymbol{c}}$ remains smaller than either $\text{SNR}_{\boldsymbol{a}}$ or $\text{SNR}_{\boldsymbol{b}}$ and stays relatively flat during training. Given that this measures the gradient SNR of the roughness penalty, this may indicate that the smoothness of the function stabilizes very early during training and does not contribute much to the learning after. This is supported by the frequency analysis in Figure~\ref{fig:spectral_convergence}, which shows the spectral content of the NeuroPMD is relatively stable over most of the optimization trajectory.}

\section{Additional Neural Connectivity Data Analysis}\label{ssec:additional_rda}

Figure~\ref{fig:abcd_3subj_s2} shows the marginal density function estimates, $\tilde{f}_{E}$, for connections originating from the medial orbitofrontal cortex (MOFC) for three randomly selected ABCD subjects. Figure~\ref{fig:abcd_3subj_surf} maps these marginal density functions onto the cortical surface for biological interpretability. While similar overall structures are apparent in all subjects, there is significant between-subject variability in the high-resolution details.

\begin{figure}[!ht]
    \centering
    \includegraphics[width=\textwidth]{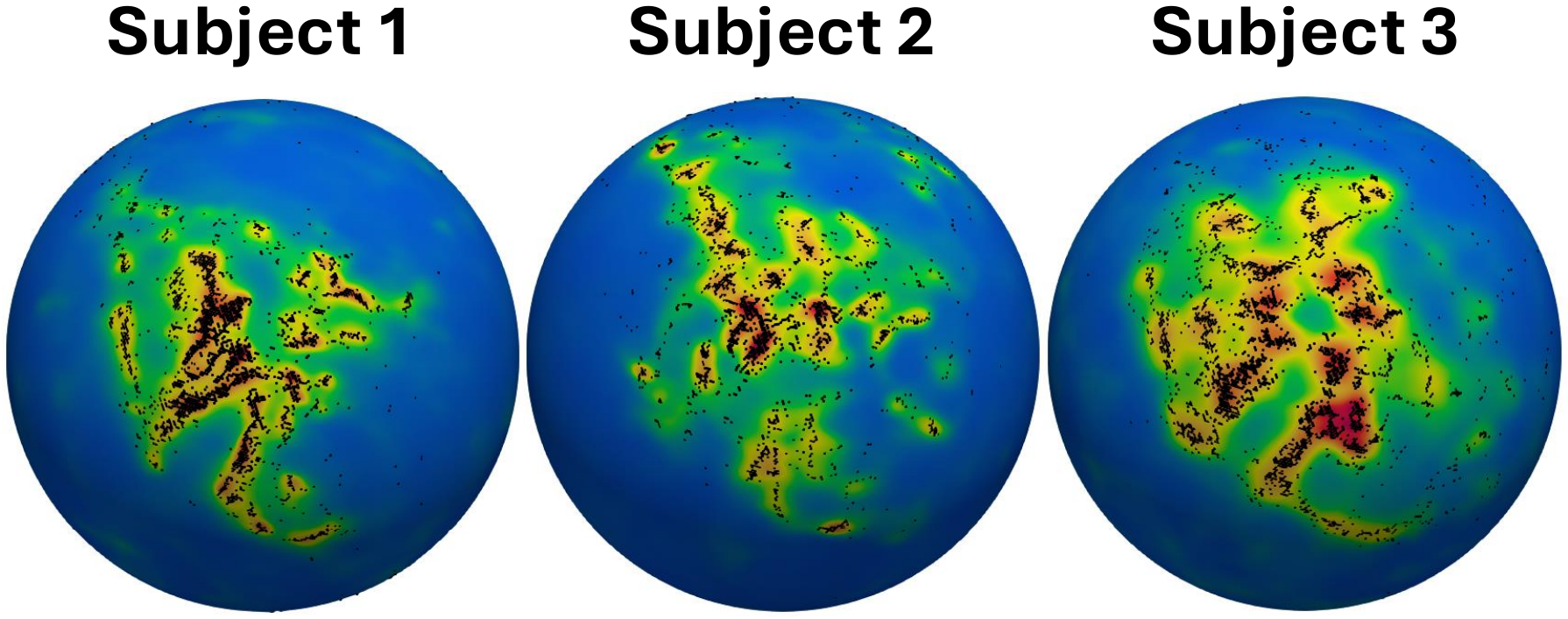}
    \caption{Marginal density function estimates for connections from the medial orbitofrontal cortex (MOFC), generated using our method (NeuroPMD) for three randomly selected ABCD subjects (left, middle and right). Black dots represent the endpoints connected to the MOFC. Color scales are normalized within each image to emphasize differences in the shape of the functions.}
    \label{fig:abcd_3subj_s2}
\end{figure}

\begin{figure}[!ht]
    \centering
    \includegraphics[width=\textwidth]{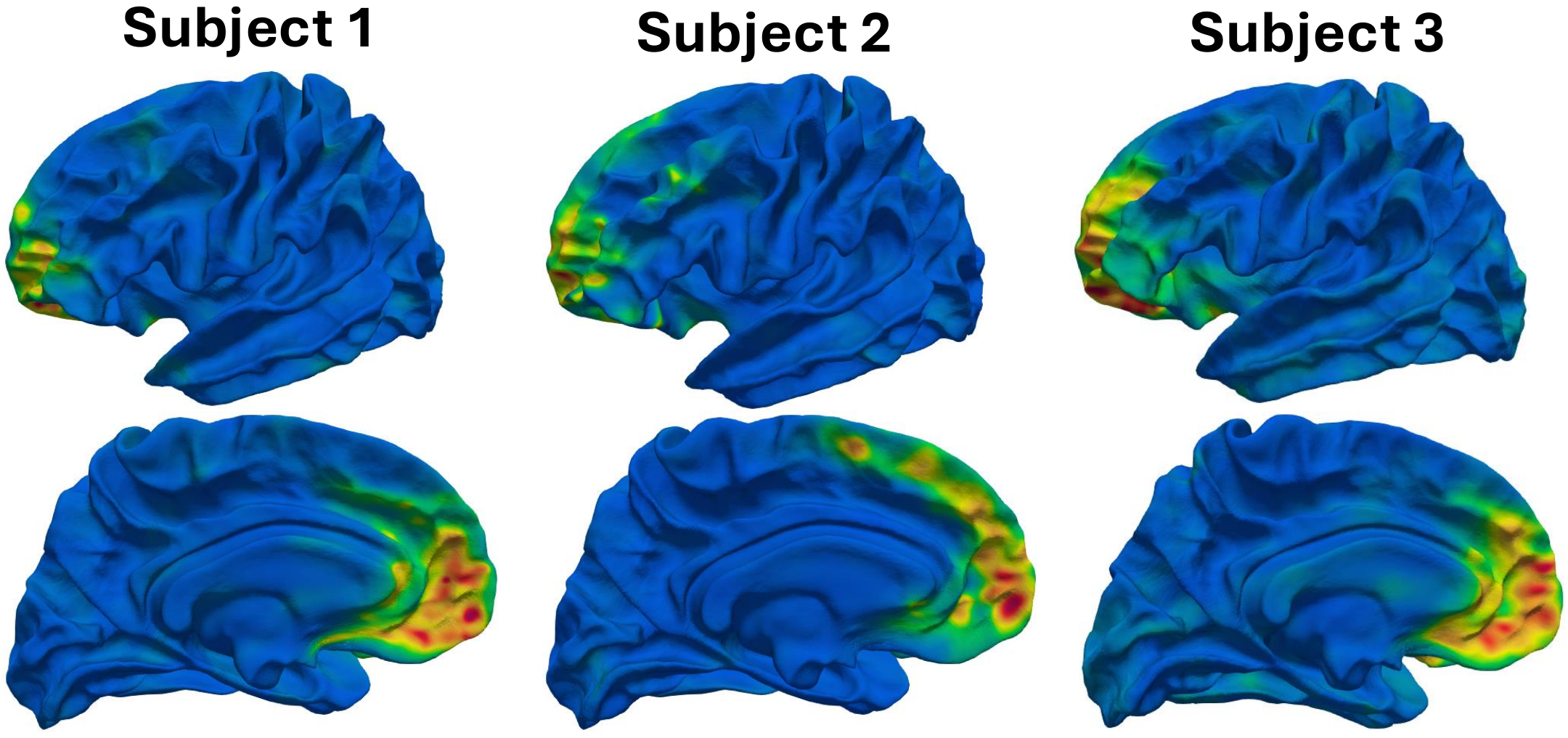}
    \caption{Same as Figure~\ref{fig:abcd_3subj_s2}, with marginal density functions mapped to the cortical surface.}
    \label{fig:abcd_3subj_surf}
\end{figure}

\newpage
\bibliographystyleSupp{apalike}
\bibliographySupp{refs}

\end{document}